\PassOptionsToPackage{table}{xcolor}
\documentclass{article}

\usepackage[T1]{fontenc}
\usepackage{iclr2027_conference,times}

\usepackage{amsmath,amsfonts,bm}

\def\eqref#1{(\ref{#1})}

\def\1{\bm{1}}

\DeclareMathAlphabet{\mathsfit}{\encodingdefault}{\sfdefault}{m}{sl}
\SetMathAlphabet{\mathsfit}{bold}{\encodingdefault}{\sfdefault}{bx}{n}

\newcommand{\R}{\mathbb{R}}

\usepackage{microtype}
\usepackage{graphicx}
\usepackage{xcolor}

\usepackage{tikz}

\definecolor{contribred}{RGB}{190,105,100}

\newcommand{\contribnum}[1]{%
  \tikz[baseline=(char.base)]{
    \node[
      circle,
      draw=contribred,
      text=contribred,
      line width=0.6pt,
      inner sep=0.7pt,
      font=\scriptsize\bfseries
    ] (char) {#1};
  }%
}

\newif\ifshowrevisions
\showrevisionsfalse

\usepackage{booktabs}
\usepackage{tabularx}
\usepackage{array}
\usepackage{longtable}
\usepackage{float}
\usepackage{placeins}
\usepackage{amsmath}
\usepackage{amssymb}
\usepackage{mathrsfs}
\usepackage{mathtools}
\usepackage{amsthm}
\usepackage{enumitem}
\setlist[itemize]{label=$\bullet$}
\usepackage{url}

\theoremstyle{plain}
\newtheorem{theorem}{Theorem}[section]
\newtheorem{proposition}[theorem]{Proposition}
\newtheorem{lemma}[theorem]{Lemma}
\newtheorem{corollary}[theorem]{Corollary}
\theoremstyle{definition}
\newtheorem{definition}[theorem]{Definition}

\theoremstyle{remark}
\newtheorem{remark}[theorem]{Remark}

\usepackage[bookmarks=false]{hyperref}
\hypersetup{hidelinks}
\usepackage[capitalize,noabbrev]{cleveref}
\crefname{theorem}{theorem}{theorems}
\Crefname{theorem}{Theorem}{Theorems}
\crefname{proposition}{proposition}{propositions}
\Crefname{proposition}{Proposition}{Propositions}
\crefname{lemma}{lemma}{lemmas}
\Crefname{lemma}{Lemma}{Lemmas}
\crefname{corollary}{corollary}{corollaries}
\Crefname{corollary}{Corollary}{Corollaries}
\crefname{definition}{definition}{definitions}
\Crefname{definition}{Definition}{Definitions}
\crefname{assumption}{assumption}{assumptions}
\Crefname{assumption}{Assumption}{Assumptions}
\crefname{remark}{remark}{remarks}
\Crefname{remark}{Remark}{Remarks}

\providecommand{\R}{}
\renewcommand{\R}{\mathbb{R}}

\newcommand{\Z}{\mathbb{Z}}

\newcommand{\rank}{\operatorname{rank}}
\newcommand{\Id}{I}
\newcommand{\OU}{\mathcal{O}}

\newcommand{\bq}{\boldsymbol q}
\newcolumntype{Y}{>{\raggedright\arraybackslash}X}

\title{On Parameters of Nonlinear Scalar Dynamics\\
from Video: Invariants, Calibration,\\
and Identifiability}

\author{%
\textbf{Wenjie Wang}$^{1,*}$ \quad
\textbf{Yuanyuan Wang}$^{2,*}$ \quad
\textbf{Zixiang Jiang}$^{1}$\\
\textbf{Shaoan Xie}$^{3}$ \quad
\textbf{Mingming Gong}$^{1,2,\dagger}$\\[0.5em]
\normalfont $^{1}$University of Melbourne\\
\normalfont $^{2}$Mohamed bin Zayed University of Artificial Intelligence\\
\normalfont $^{3}$Carnegie Mellon University
}

\iclrfinalcopy
\hypersetup{
  pdftitle={On Parameters of Nonlinear Scalar Dynamics from Video: Invariants, Calibration, and Identifiability},
  pdfauthor={Wenjie Wang, Yuanyuan Wang, Zixiang Jiang, Shaoan Xie, Mingming Gong}
}

\begin{document}
\maketitle
\fancyhead[L]{\fontsize{8}{10}\selectfont
On Parameters of Nonlinear Scalar Dynamics from Video:
Invariants, Calibration, and Identifiability}
\makeatletter
\begingroup
\renewcommand{\thefootnote}{}%
\long\def\@makefntext#1{\noindent#1}%
\footnotetext{Preprint.\quad
  $^{*}$Equal contribution.\quad $^{\dagger}$Corresponding author.}
\endgroup
\makeatother

\begin{abstract}
Physical parameter estimation from video aims to recover the parameters of a known family of governing dynamical equations from pixel observations. 
Existing identifiability theory for this setting
has focused on linear time-invariant (LTI) second-order systems, leaving open
what can be identified for nonlinear scalar dynamics.
We develop an identifiability theory for nonlinear scalar second-order ODEs, organized by how their velocity dependence interacts with changes of the learned state coordinate. Under a shared non-collapsed state map and explicit same-state velocity-coverage conditions, we show that parameter identifiability depends on the ODE family: some parameters are uniquely identifiable, while in other families only invariant parameter combinations are identifiable or external physical calibration is required. For laws that are at most linear in velocity, compatibility forces affine coordinate alignment, yielding explicit parameter relations, invariants, and calibration conditions. This affine conclusion extends to broader finite velocity-feature families when coordinate curvature can be separated from the declared velocity dependence. For families admitting a squared-velocity term, nonlinear coordinate ambiguity can remain; a law-derived normalization instead enables affine comparison between canonical laws. Experiments on synthetic systems and real pendulum and free-fall videos support the predicted parameter relations, coverage effects, and calibration requirements.
\end{abstract}

\section{Introduction}
\label{sec:intro}

Physical parameter estimation from video aims to recover the parameters of a
known family of governing dynamical equations from pixel observations.
Recent video models increasingly seek to predict and interact with the
physical world \citep{bruce2024genie,assran2025vjepa2}, but predictive
success alone does not establish which governing physical parameters are
identifiable from visual observations. A more direct line of work therefore
asks the inverse question: given a prescribed dynamical model, can its
governing parameters be recovered from video? Physics-as-inverse-graphics
methods estimate latent physical states and parameters using known
differential equations \citep{jaques2020physics}, while more recent
encoder-based approaches estimate parameters of known continuous dynamics
directly from video \citep{garcia2025learning}. IRIS further systematizes
this setting by pairing real videos with governing equations and independently
measured physical parameters, enabling evaluation of parameter recovery and
identifiability \citep{khanbayov2026iris}. These developments establish
physical parameter recovery from video as a concrete scientific-inference
problem. Identifiability guarantees have begun to emerge in restricted
settings, most notably for linear time-invariant dynamics, while corresponding
results for nonlinear scalar systems remain substantially less understood.
We discuss broader connections to video prediction, quantitative sensing,
and physics-oriented evaluation in
Appendix~\ref{app:related-work-extended}.

For scalar linear time-invariant (LTI) second-order ODEs,
\citet{wang2026physics} establish affine coordinate alignment and exact
coefficient identification under shared-state-map and trajectory-coverage
assumptions. LTI is a favorable special case: the remaining
family-preserving affine transformations leave its coefficients unchanged.
But this no longer holds in general for nonlinear dynamics. 
For example, in the
cubic Duffing law
$z''+\delta z'+\alpha z+\beta z^3=0$, rescaling the state by
$\hat z=\lambda z$ leaves $\delta$ and $\alpha$ invariant under this
transformation but maps $\beta$ to $\beta/\lambda^2$. Thus even affine
coordinate ambiguity can induce parameter ambiguity, while some
velocity-dependent families can admit broader nonlinear coordinate
ambiguity. This motivates our central question:
\emph{for nonlinear scalar dynamics, which parameter information remains
identifiable from video, and when is external physical calibration sufficient
to resolve the remaining ambiguity?}

Answering this question requires more than fitting a richer ODE. The learned
state coordinate and the dynamical parameters must be analyzed jointly,
because multiple coordinate--parameter descriptions may remain compatible
with the same observed dynamics. Two issues then become central. First,
passive videos constrain this ambiguity only through the states and velocities
realized by the observed trajectories, making identifiability dependent on
motion coverage as well as on the equation form. Second, external physical
measurements may resolve residual parameter ambiguity, but their sufficiency
must be characterized rather than replaced by an arbitrary normalization of
the learned representation.
We therefore develop a unified identifiability analysis for nonlinear scalar
second-order ODEs, organized by how the velocity structure of the governing
law interacts with coordinate changes. The physical and fitted dynamics
belong to the same family fixed before training, and a shared encoder provides
the scalar coordinate used for ODE fitting. Under a common non-collapsed
$C^2$ state map and explicit same-state physical-velocity coverage, we
characterize compatible parameter sets, identifiable invariants, and
sufficient conditions for physical calibration.

\textbf{Contributions.}
\contribnum{1} We extend identifiability theory from LTI dynamics to nonlinear
scalar second-order ODEs with nonlinear state dependence and structured
velocity dependence. We characterize when learned coordinates must be affine
and how the remaining coordinate freedom transforms physical parameters;
for polynomial-velocity families in which a squared-velocity term can absorb
coordinate curvature, we instead obtain affine comparison after a
law-derived canonical normalization. Under explicit representation
conditions, we further establish realizability of the semilinear compatible
parameter sets, yielding necessary and sufficient identification criteria.
\contribnum{2}  We turn the theory into a constructive tool for family-specific parameter
analysis. 
Coefficient-weight rules expose invariant parameter combinations,
rank conditions give sufficient local criteria for physical calibration,
and an operational audit connects the required family, shared-map, coverage,
canonical-coordinate, parameter-comparison, and calibration conditions to
empirical diagnostics.
\contribnum{3} We validate the theory on synthetic systems and real pendulum
and free-fall videos, supporting the predicted parameter relations, coverage effects, and calibration behavior.


\section{Problem Setup and Identification Targets}
\label{sec:setup}

\textbf{Observations.}
We observe $M$ passive video clips as timestamped frames
$\mathcal D=\{(t_{m,k},x_k^{(m)})\}$ at known recorded times.
For the structural analysis, each clip samples an ideal frame path
$x^{(m)}(t)$ and has an unobserved physical trajectory
$z^{(m)}\in C^2(I^{(m)};\mathbb R)$.
Clips may have different initial conditions; no controlled input is applied
or observed. Full sampling notation is given in Appendix~\ref{app:notation}.


\textbf{Declared dynamics.}
Before fitting, we fix a parameter space $\Theta$ and a declared family of
laws $F_\xi:D_\xi\times\mathbb R\rightarrow\mathbb R$, $\xi\in\Theta$,
where each $D_\xi$ is an open state interval. The physical trajectories
share one unknown parameter $\theta\in\Theta$: each satisfies
$z''=F_\theta(z,z')$ on its interval, so all realized states lie in
$D_\theta$.
The primary class studied in Section~\ref{sec:scalar-theory} is the semilinear family
\begin{equation}
    z''+a_\xi(z)z'+b_\xi(z)=0,
    \qquad \xi\in\Theta,
    \label{eq:semilinear}
\end{equation}
with continuous $a_\xi,b_\xi$ on $D_\xi$; equivalently,
$F_\xi(u,v)=-a_\xi(u)v-b_\xi(u)$.
Section~\ref{sec:canonical-theory} treats velocity-dependent extensions;
the notation $F_\xi$ does not extend the guarantees to arbitrary
nonlinear families.

\textbf{Direct encoder and shared state map.}
A single encoder $E_\phi$, shared across all clips, maps each frame to a
scalar latent, giving sampled outputs $\hat z_k^{(m)}=E_\phi(x_k^{(m)})$
and ideal paths $\hat z^{(m)}(t)=E_\phi(x^{(m)}(t))$. The learned latent
state $\hat z$ is not assumed to equal the physical state $z$;
Definition~\ref{def:state-consistency} below is the hypothesis relating
the two.

Write $R_z=\bigcup_{m=1}^M z^{(m)}(I^{(m)})$ for the realized
physical-state range; the domain condition above gives
$R_z\subset D_\theta$. Fix a nonempty open interval $U\subset R_z$.
All coverage assumptions and identification claims are stated on $U$:
the data constrain the declared laws only at states the videos visit.

\begin{definition}[Latent-state consistency]
\label{def:state-consistency}
The encoder trajectories are latent-state-consistent on $U$ if there
exists one map $f\in C^2(U)$, common to all clips, such that
\begin{equation}
    \hat z^{(m)}(t)
    =
    f\!\left(z^{(m)}(t)\right)
    \quad
    \text{for every $(m,t)$ with }z^{(m)}(t)\in U.
    \label{eq:state-consistency}
\end{equation}
The map is non-collapsed on $U$ when $f'\not\equiv0$ there; $f'$ may
still vanish at individual states.
\end{definition}

Definition~\ref{def:state-consistency} adapts the state-consistency
hypothesis of \citet{wang2026physics} to multiple clips and the window
$U$. Controlled capture, such as a static camera, fixed exposure and
white balance, near-constant lighting, and a uniform background,
supports this hypothesis in practice; it remains a modeling assumption rather than a
certified property.

\textbf{Identification target.}
A fitted parameter $\eta\in\Theta$ specifies a law in the same declared
family, imposed directly on $\hat z$, with $f(U)\subset D_\eta$.
At interior points with $z^{(m)}(t)\in U$, its continuous-time residual is
$\mathcal R_{\eta,\phi}^{(m)}(t):=
\hat z^{(m)\prime\prime}(t)-F_\eta(\hat z^{(m)}(t),\hat z^{(m)\prime}(t))$.
A triple $(E_\phi,f,\eta)$ is an \emph{admissible report on $U$} when
it satisfies Definition~\ref{def:state-consistency}, the common map is
non-collapsed, all domain conditions hold, and the residual vanishes at
the covering points: the interior trajectory points that realize the
velocity coverage stated in each result.
Our target is any parameter functional $\mathcal J$ satisfying
$\mathcal J(\eta)=\mathcal J(\theta)$ for every admissible report; taking
$\mathcal J$ to be the identity gives full parameter identification.
When dynamics leave residual freedom, we ask which external physical
calibrations identify further quantities.

\textbf{Fitting objective.}
In practice, $(\phi,\eta)$ is estimated from the sampled frames: within
each clip, time derivatives of $\hat z$ are formed by three-point
centered differences, giving a sampled residual
$\widehat r_{m,k}(\phi,\eta)$ at interior indices
(Appendix~\ref{app:sampled_definitions}). The fitting objective is
\begin{equation}
\label{eq:training_objective}
\mathcal L_{\mathrm{total}}(\phi,\eta)
=\frac{1}{N_{\mathrm{int}}}\sum_{m=1}^M\sum_{k=1}^{T_m-1}
\widehat r_{m,k}(\phi,\eta)^2+\lambda_{\mathrm{var}}\mathcal L_{\mathrm{var}},
\end{equation}
where $N_{\mathrm{int}}=\sum_{m=1}^M(T_m-1)$, $\lambda_{\mathrm{var}}>0$,
and the per-clip variance-floor penalty $\mathcal L_{\mathrm{var}}$
discourages a collapsed, near-constant encoder output.
The identifiability results in this paper concern the ideal
continuous-time setting; in practice we minimize this objective or the
variants stated with each experiment, and
Appendix~\ref{app:residual_accounting} quantifies the discretization and
sampling gap between the two.

\section{The semilinear core: orbits, invariants, and anchors}
\label{sec:scalar-theory}

We study parameter identification in the semilinear family
\eqref{eq:semilinear}, under the setup of Section~\ref{sec:setup}.
Trajectory coverage first restricts the learned coordinate to an affine map;
the declared family then determines invariant quantities and how physical
anchors can further constrain compatible parameters.

\subsection{Coordinate Alignment and Compatible Parameters}

\paragraph{Velocity coverage.}
The videos have \emph{three-slope level coverage} on a nonempty open interval
$U$ if at least three pairwise-distinct physical velocities are realized
at interior trajectory points at every $u\in U$. The points may come from one clip
or several clips; the requirement concerns velocities at the same physical
state, not the number of videos or repeated observations alone.

At the covering points, where $\hat z^{(m)}=f(z^{(m)})$ by
Definition~\ref{def:state-consistency}, the physical and fitted laws are
\begin{align}
 z^{(m)\prime\prime}(t)+a_\theta(z^{(m)}(t))z^{(m)\prime}(t)+b_\theta(z^{(m)}(t))&=0, \label{eq:true-semilinear-main}\\
 \hat z^{(m)\prime\prime}(t)+a_\eta(\hat z^{(m)}(t))\hat z^{(m)\prime}(t)+b_\eta(\hat z^{(m)}(t))&=0. \label{eq:latent-semilinear-main}
\end{align}

\begin{theorem}[Three-slope affine collapse]
\label{thm:semilinear-affine-collapse}
Let $U\subset R_z$ be a nonempty open interval and suppose the
videos have three-slope level coverage on $U$.  Assume latent-state consistency
through one $f\in C^2(U)$ with $f(U)\subset D_\eta$ and
$f'\not\equiv0$ on $U$.  Let the coefficient functions in
\eqref{eq:true-semilinear-main}--\eqref{eq:latent-semilinear-main} be continuous
on their respective domains, and assume the physical and encoder-induced paths
satisfy those two laws at all covering points.  Then
$f(u)=\lambda u+\tau$ on $U$ for constants $\lambda\neq0$ and
$\tau\in\R$; that is, $\hat z^{(m)}(t)=\lambda\,z^{(m)}(t)+\tau$ whenever
$z^{(m)}(t)\in U$. Moreover, for every $u\in U$,
\begin{equation}
\label{eq:compat-a-main}
 a_\eta(\lambda u+\tau)=a_\theta(u),
 \qquad
 b_\eta(\lambda u+\tau)=\lambda b_\theta(u)
 .
\end{equation}
\end{theorem}

Fix a covering point with state $z^{(m)}(t)=u$ and velocity
$z^{(m)\prime}(t)=v$. Substituting $\hat z^{(m)}=f(z^{(m)})$ into
\eqref{eq:latent-semilinear-main} and eliminating $z^{(m)\prime\prime}$
with \eqref{eq:true-semilinear-main} gives
\begin{equation}
\label{eq:semilinear-residual-main}
0=f''(u)v^2
 +f'(u)\bigl(a_\eta(f(u))-a_\theta(u)\bigr)v
 +b_\eta(f(u))-f'(u)b_\theta(u).
\end{equation}
Three distinct realized velocities are roots of this quadratic polynomial,
so its coefficients vanish. In particular, $f''=0$ throughout $U$;
non-collapse then gives a nonzero affine slope. The remaining identities
are \eqref{eq:compat-a-main}, which we call the \emph{orbit equations}.
The full proof is in Appendix~\ref{app:proofs-scalar}.
Three slopes are sharp for a family-uniform affine-collapse guarantee over
unrestricted continuous semilinear coefficients
(Proposition~\ref{prop:two-slope-sharpness}); narrower families
may permit conclusions from fewer slopes.

\paragraph{From coordinates to parameters.}
Affine alignment does not alone determine $\eta$: the transformed law must
also match a member of the declared coefficient family on the transported
interval.
Call $(\lambda,\tau)\in\R^*\times\R$ an \emph{affine witness} for
$\eta\in\Theta$ when $\lambda U+\tau\subset D_\eta$ and
\eqref{eq:compat-a-main} holds on $U$. Define
\[
\OU_U(\theta)
 :=\{\eta\in\Theta:\eta\text{ admits an affine witness on }U\}.
\]
This is the \emph{family-compatible outer set} on $U$; it need not be a
group orbit or an equivalence class for general interval-local
coefficients.

\begin{theorem}[Affine compatibility and singleton identification]
\label{thm:orbit-characterization}
Fix a video collection whose physical trajectories are generated by $\theta$
in the semilinear family of Section~\ref{sec:setup}, with three-slope level
coverage on a nonempty open
interval $U\subset R_z$. Every admissible report $(E_\phi,f,\eta)$ on
$U$ satisfies $\eta\in\OU_U(\theta)$. Thus $\OU_U(\theta)=\{\theta\}$
implies $\eta=\theta$ for every such report.
\end{theorem}

Any functional constant on $\mathcal O_U(\theta)$ is therefore identifiable.
Under explicit representation assumptions, the next result shows that this
outer set is also realizable.

\begin{proposition}[Realizability of affine-compatible parameters]
\label{prop:affine-realizability}
Under the semilinear setup and coverage assumptions of
Theorem~\ref{thm:orbit-characterization}, let $\mathscr E$ be an encoder class.
Assume:
\begin{enumerate}[label=(\roman*),nosep,leftmargin=1.6em]
\item there exists $E_0\in\mathscr E$ such that
$E_0(x^{(m)}(t))=z^{(m)}(t)$ whenever $z^{(m)}(t)\in U$;
\item $E\in\mathscr E$ implies $\lambda E+\tau\in\mathscr E$ for every
$\lambda\ne0$ and $\tau\in\mathbb R$.
\end{enumerate}
For every $\eta\in\mathcal O_U(\theta)$ and each of its affine witnesses
$(\lambda,\tau)$, the encoder $E_\eta=\lambda E_0+\tau$ realizes an
admissible report with $f(u)=\lambda u+\tau$. Consequently,
\begin{equation}
\{\eta\in\Theta:\exists\,(E,f,\eta)\text{ admissible on }U,
\ E\in\mathscr E\}
=\mathcal O_U(\theta).
\label{eq:realizable-compatible-set}
\end{equation}
\end{proposition}

\begin{corollary}[Exact identification criterion]
\label{cor:realizable-invariants}
Under Proposition~\ref{prop:affine-realizability}, a parameter functional
$\mathcal J$ is identified across admissible reports in $\mathscr E$ if and
only if it is constant on $\mathcal O_U(\theta)$. In particular, the full
parameter is identified if and only if $\mathcal O_U(\theta)=\{\theta\}$.
\end{corollary}

In other words, Theorem~\ref{thm:orbit-characterization} shows that every
admissible report must lie in $\mathcal O_U(\theta)$, but this alone does
not imply that every point in the set is achievable by an encoder.
Proposition~\ref{prop:affine-realizability} supplies the converse under
assumptions (i)--(ii): if the physical coordinate can be represented and
the encoder class is closed under affine output transformations, then every
$\eta\in\mathcal O_U(\theta)$ is realized by an admissible encoder.
Hence, under these representation assumptions, $\mathcal O_U(\theta)$ is
exactly the parameter ambiguity left by the videos, and a quantity is
identifiable precisely when it is constant on this set. Without these
assumptions, only the containment result of
Theorem~\ref{thm:orbit-characterization} is guaranteed.

\subsection{Parameter Invariants}
\label{subsec:orbit-calculus-main}

Once translation is excluded, the degree of each monomial in a polynomial
family sets its response to scaling; this is a calculation within that
subclass, not an assumption on all semilinear laws.

\textbf{Translation lock.}
A polynomial family is \emph{translation locked} on a parameter stratum
when every compatible affine witness has $\tau=0$. A structural top gap
suffices: the maximal declared degree $n\ge1$ in the damping or restoring
support has a nonzero coefficient at $\theta$ while degree $n-1$ is
structurally absent (Lemma~\ref{lem:translation-lock});
Appendix~\ref{app:helmholtz-branches} illustrates the unlocked case on
the Helmholtz family.

Once $\tau=0$, the parameter transformation is diagonal in the
polynomial coefficients.

\begin{proposition}[Weight rule for a translation-locked polynomial family]
\label{thm:weight-calculus}
Let $U$ be a nonempty open interval and consider a translation-locked
polynomial family with fixed structural supports $I_a,I_b$.  Write the physical
law as $a_\theta(u)=\sum_{i\in I_a}A_i(\theta)u^i$ and
$b_\theta(u)=\sum_{j\in I_b}B_j(\theta)u^j$.  Assign weights
$w(A_i)=i$ and $w(B_j)=j-1$.  For any compatible report
$\eta\in\OU_U(\theta)$ with scale witness $u\mapsto\lambda u$,
\begin{equation}
\label{eq:weight-scaling}
 A_i(\eta)=\lambda^{-w(A_i)}A_i(\theta),
 \qquad
 B_j(\eta)=\lambda^{-w(B_j)}B_j(\theta).
\end{equation}
Every weight-zero coefficient, and every well-defined product or ratio whose
net weight is zero, is therefore invariant across compatible reports.
\end{proposition}

\noindent\textbf{Example: Duffing invariants, ambiguity, and calibration.}
Returning to cubic Duffing,
$z''+\delta z'+\alpha z+\beta z^3=0$ with $\beta\ne0$,
translation is locked and every compatible report satisfies
\[
 (\delta_\eta,\alpha_\eta,\beta_\eta)
   =(\delta_\theta,\alpha_\theta,\beta_\theta/\lambda^2).
\]
Hence $\delta$, $\alpha$, and $\operatorname{sign}(\beta)$ are invariant,
whereas the magnitude of $\beta$ depends on the latent scale.
Importantly, this is a realizable ambiguity rather than only an outer
possibility: by Proposition~\ref{prop:affine-realizability}, every allowed
scale is realized by an encoder on the same videos, so $|\lambda|\ne1$
produces genuinely different cubic coefficients even under exact
continuous compatibility. For quintic Duffing with $\beta\gamma\ne0$,
$\gamma_\eta=\gamma_\theta/\lambda^4$, making
$\gamma/\beta^2$ invariant. A single known nonzero unsigned physical
displacement between states in $U$ fixes $|\lambda|$ and therefore
calibrates these scale-dependent coefficients, while coordinate reflection
may remain. Thus the physical parameters can be calibrated without selecting
a unique latent coordinate. Table~\ref{tab:regimes} summarizes the
corresponding conclusions for the other families.

\subsection{Physical Calibration and Family-Specific Recovery}
\label{sec:regimes}
An external anchor supplies physical information not fixed by the ODE fit.
For a translation-locked family, a signed pairing of
$u_\star\in U\setminus\{0\}$ with $\hat z_\star=f(u_\star)$ fixes
$\lambda=\hat z_\star/u_\star$ and hence all coefficients through
\eqref{eq:weight-scaling}. An unsigned amplitude pairing fixes only
$|\lambda|$; reflection remains if both sign branches are admitted, and only
odd-weight coefficients change under it. 
Variance normalization is not a physical anchor, and anchors outside $U$
require a separately justified extension of the witness.

For general families, an anchor differential with rank equal to the dimension
of the regular lifted witness--parameter set gives sufficient local
calibration (Proposition~\ref{prop:anchor-rank}); disconnected branches require
separate checks. Table~\ref{tab:regimes} distinguishes parameter recovery from full
coordinate calibration.
All conclusions are relative to $U$ and the stated parameterization.
For the polynomial and real-analytic examples, the coefficient identities
extend across the connected common domain.

\begin{table}[t]
\centering
\small
\caption{Family-specific coordinate ambiguity, identifiable parameter
information, and sufficient physical calibration under the assumptions of
Theorem~\ref{thm:semilinear-affine-collapse}. Details are given in
Appendix~\ref{app:regime-proofs}.}
\label{tab:regimes}
\setlength{\tabcolsep}{3.5pt}
\renewcommand{\arraystretch}{0.65}
\begin{tabularx}{\linewidth}{@{}>{\raggedright\arraybackslash}p{0.37\linewidth}
 >{\raggedright\arraybackslash}p{0.17\linewidth}
 >{\raggedright\arraybackslash}p{0.17\linewidth}Y@{}}
\toprule
\textbf{Declared family} & \textbf{Coordinate gauge} &
\textbf{Invariant information} & \textbf{Sufficient parameter calibration} \\
\midrule
\textbf{LTI}\newline
$z''+\delta z'+\alpha z=0$ &
$\lambda u+\tau$,\newline$\alpha\tau=0$ &
$\delta,\alpha$ & None \\
\addlinespace
\textbf{Angular pendulum}\newline
$z''+\delta z'+\omega^2\sin z=0$ ($\omega>0$) &
$\sigma u+2\pi k$ &
$\delta,\omega$ & None \\
\addlinespace
\textbf{Normalized Van der Pol}\newline
$z''+\mu(z^2-1)z'+z=0$ ($\mu\neq0$) &
$\sigma u$ &
$\mu$ & None \\
\addlinespace
\textbf{Cubic Duffing}\newline
$z''+\delta z'+\alpha z+\beta z^3=0$ ($\beta\neq0$)&
$\lambda u$ &
$\delta,\alpha,$\newline$\operatorname{sign}(\beta)$ &
One unsigned amplitude anchor for $\beta$ \\
\addlinespace
\textbf{Quintic Duffing}\newline
$\begin{gathered}
 z''+\delta z'+\alpha z\\[-0.1em]
 {}+\beta z^3+\gamma z^5=0 
\end{gathered}$ ($\beta\gamma\neq0$) &
$\lambda u$ &
$\delta,\alpha,$\newline$\operatorname{sign}(\beta),$\newline$\gamma/\beta^2$ &
One unsigned amplitude anchor for $\beta,\gamma$ \\
\bottomrule
\end{tabularx}
\end{table}

\subsection{Approximate Residuals}
\label{subsec:robustness}
Theorem~\ref{thm:near-level-stability} bounds the affine mismatch by
residual size, level-matching error, and physical-velocity design
conditioning. These are deterministic transfers, not pixel-to-parameter
rates. Appendix~\ref{app:proofs-robust} gives the bounds and shows why
encoder-coordinate conditioning is no substitute for physical-unit
conditioning; Appendix~\ref{app:sampled-interface} covers discretization
and approximate state consistency.

\section{Beyond Semilinearity: Extending Parameter Identification}
\label{sec:canonical-theory}

The semilinear core allows nonlinear state dependence but only the velocity
features $1$ and $v$. To extend its coordinate-to-parameter argument,
consider the chain rule
\[
 \hat z''=f'(z)z''+\underbrace{f''(z)(z')^2}_{\text{coordinate curvature}}.
\]
If the declared velocity features exclude $v^2$, sufficient coverage again
forces raw affine alignment. If the family admits $v^2$, we instead compare
laws after a law-derived coordinate change, under the polynomial-velocity
hypotheses of Section~\ref{sec:connection-boundary}. Both routes constrain
the original ODE parameters; unique recovery remains family-dependent.

\subsection{When Raw Affine Alignment Still Holds}
\label{subsec:feature-separation}

Consider a fixed velocity-feature space $\mathsf V$ and laws
\begin{equation}
\label{eq:feature-law}
 z''=F_\xi(z,z'),
 \qquad F_\xi(u,\cdot)\in\mathsf V
 \quad (u\in D_\xi).
\end{equation}
The space $\mathsf V$ is finite-dimensional, consists of continuous functions,
contains $1$, and is \emph{dilation stable}:
$\varphi(\alpha\,\cdot)\in\mathsf V$ for every $\varphi\in\mathsf V$ and
$\alpha\in\R$, including zero. This closure keeps a rescaled velocity law
within the same feature space.

\begin{theorem}[Curvature-column exclusion]
\label{thm:feature-separation}
Retain the interval, latent-state-consistency, non-collapse, and domain
assumptions of Theorem~\ref{thm:semilinear-affine-collapse}.
Replace its two semilinear laws by \eqref{eq:feature-law}, satisfied exactly
at the covering points, and its three-slope coverage by the rank condition
below. Suppose $\mathsf V$ has the properties above and $v^2\notin\mathsf V$.
For any basis $\varphi_1,\ldots,\varphi_K$, require at every $u\in U$
realized physical velocities $v_1(u),\ldots,v_{N_u}(u)$ at interior points
such that
\begin{equation}
\label{eq:feature-rank-matrix}
 M_u:=\bigl[\,\varphi_1(v_i)\ \cdots\ \varphi_K(v_i)\ v_i^2\,\bigr]_{i=1}^{N_u},
 \qquad \rank M_u=K+1.
\end{equation}
Then $f(u)=\lambda u+\tau$ on $U$, with $\lambda\ne0$ and $\tau\in\R$, and
\begin{equation}
\label{eq:feature-orbit-equation}
 F_\eta(\lambda u+\tau,\lambda v)=\lambda F_\theta(u,v)
 \quad (u\in U,\ v\in\R).
\end{equation}
\end{theorem}

Full rank separates $f''(u)v^2$ from the declared features; the proof is in
Appendix~\ref{app:connection-proofs}. For
$\mathsf V=\operatorname{span}\{1,v\}$ with basis $(1,v)$, the design is
$N_u\times3$; three distinct velocities select a nonsingular Vandermonde
submatrix and recover the semilinear result.

\paragraph{Odd quadratic drag.}
For $F_\xi(u,v)=c_{0,\xi}(u)+c_{1,\xi}(u)v+c_{d,\xi}(u)v|v|$,
any four distinct velocities including a positive and a negative value
give full rank. Affine compatibility then implies
\[
 c_{d,\eta}(\lambda u+\tau)=\frac{c_{d,\theta}(u)}{|\lambda|}.
\]
One-sided velocities cannot separate $v|v|$ from $v^2$; the rank proof and
an exact one-sided construction, with its family restriction, are given in
Appendix~\ref{app:connection-proofs}.

Whenever the theorem applies, \eqref{eq:feature-orbit-equation}, with
$\lambda U+\tau\subset D_\eta$, defines a family-compatible outer set and
yields the same route to invariants and calibration as
Section~\ref{sec:scalar-theory}. Orbit terminology applies to restricted laws
with transported domains; the parameter projection at fixed $U$ need not be
an equivalence class for general interval-local families.
\vspace{-2mm}

\subsection{When Affine Comparison Requires Normalization}
\label{sec:connection-boundary}
\label{subsec:connection-slice}
\label{subsec:canonical-reports}

A squared-velocity channel can absorb coordinate curvature, so raw
affine alignment need not hold; the remedy is a corrective coordinate
computed from each fitted law.

\textbf{A coordinate determined by the law.}
Consider the fixed polynomial-velocity family
\begin{equation}
\label{eq:poly-velocity-law}
 z''=F_\xi(z,z'):=\sum_{r=0}^{p}c_{r,\xi}(z)(z')^r,
 \qquad \xi\in\Theta.
\end{equation}
The finite degree $p$ and the coefficient forms are declared before
fitting; $c_{r,\xi}$ are continuous on $D_\xi$, missing channels are
zero-padded, polynomial dependence is required in velocity. 
The key point is that each law determines its own corrective
coordinate: for each $\xi$, choose $u_{\xi,\star}\in D_\xi$ and let
$\psi_\xi$ solve
\begin{equation}
\label{eq:normalizer-ode}
 \psi_\xi''+c_{2,\xi}\psi_\xi'=0,
 \qquad \psi_\xi'>0,
\end{equation}
with $\psi_\xi(u_{\xi,\star})=0$ and $\psi_\xi'(u_{\xi,\star})=1$.
Then $\psi_\xi$ maps $D_\xi$ diffeomorphically onto
$J_\xi:=\psi_\xi(D_\xi)$, and in the coordinate $q=\psi_\xi(z)$ the
squared-velocity channel vanishes. We call the transformed equation the
\emph{canonical law} of $\xi$; its coefficients $\bar c_{r,\xi}$ are
explicit (Appendix~\ref{app:connection-proofs}). Normalization is a
calculation from the fitted law, not physical calibration.
The model case is a pendulum with angle $q$ observed through the
projection $y=\sin q$ on $|q|<\pi/2$~(Figure~\ref{fig:exp-projection-geometry}): the projected law has
$c_2(y)=-y/(1-y^2)$, and the normalizer is exactly
$\psi(y)=\arcsin y$ up to affine, so the canonical coordinate is the
angle itself (Appendix~\ref{app:connection-details}). The narrow angular-pendulum
family itself excludes the projected law, so this example does not
establish ambiguity within that family.

\textbf{Parameter identification after normalization.}
Comparing each law in its own canonical coordinate restores the affine
picture of Section~\ref{sec:scalar-theory}. Call the canonical laws of
$\theta$ and $\eta$ \emph{affine-compatible} on $\psi_\theta(U)$ if one
affine map $A(q)=\lambda q+\tau$, $\lambda\ne0$, carries
$\psi_\theta(U)$ into $J_\eta$ and matches their coefficients; the exact
identities are \eqref{eq:normalized-affine-orbit} in
Appendix~\ref{app:connection-proofs}. Projecting to the declared
parameters gives the canonical family-compatible outer set
\begin{equation}
\label{eq:canonical-parameter-orbit}
 \mathcal O_U^{\mathrm{can}}(\theta)
 :=\{\eta\in\Theta:\text{the canonical laws are affine-compatible on }
 \psi_\theta(U)\}.
\end{equation}
This is a set of \emph{original parameter values} $\eta$: normalization
changes how laws are compared, not the inference target. The
transported-domain and equivalence-class caveats of
Section~\ref{subsec:feature-separation} apply here.

\begin{theorem}[Canonical compatibility containment for passive video]
\label{thm:canonical-report-orbit}
Let $U\subset R_z$ be a nonempty open interval, with the
coefficient continuity and domains of \eqref{eq:poly-velocity-law}.
Set $d_v=\max\{2,p\}$. Assume at least $d_v+1$ pairwise-distinct physical
velocities are realized at interior trajectory points at every $u\in U$.
Then every admissible report of Section~\ref{sec:setup} satisfies
$\eta\in\mathcal O_U^{\mathrm{can}}(\theta)$ and
\begin{equation}
\label{eq:canonical-encoder-affine}
 \psi_\eta\!\left(\hat z^{(m)}(t)\right)
 =A\!\left(\psi_\theta(z^{(m)}(t))\right)
 =\lambda\psi_\theta(z^{(m)}(t))+\tau
\end{equation}
on every covered frame in $U$, for one affine $A$ common to all clips.
\end{theorem}

The proof combines passive coefficient matching and exact raw--canonical
compatibility (Proposition~\ref{prop:poly-coefficient-matching} and
Theorem~\ref{thm:connection-slice}, Appendix~\ref{app:connection-proofs}).
Functionals constant on this outer set are identifiable; a singleton
guarantees $\eta=\theta$. Constant viscous damping is already invariant
before normalization (Corollary~\ref{cor:linear-damping-invariant}).
Physical anchors constrain $f=\psi_\eta^{-1}\circ A\circ\psi_\theta$,
so a raw amplitude anchor need not be a canonical scale constraint.
Local anchor counting proceeds as in Section~\ref{sec:regimes}
(Proposition~\ref{prop:anchor-rank}).


Table~\ref{tab:result-map} summarizes the two theoretical routes and their
conditions, and Figure~\ref{fig:nonlinear-training-pipeline} summarizes the
workflow from joint training to parameter recovery. These results cover the
stated feature-space and polynomial-velocity settings rather than arbitrary
scalar ODEs; Appendix~\ref{app:web-tool} provides an automated tool for analyzing
new declared families under the same framework.

\section{Experiments}
\label{sec:experiments}\label{page:experiments-start}

We test our theoretical predictions using synthetic and real videos.
A shared encoder maps each frame to a scalar latent $z$; the encoder and
dynamical coefficients are learned jointly without state-coordinate targets.
Section~\ref{sec:exp-orbit} tests coordinate ambiguity and invariant parameter
information; Section~\ref{sec:exp-boundary} tests the roles of velocity
coverage and canonical normalization; Section~\ref{sec:exp-real} examines
physical parameter recovery through anchor calibration on three real cases.


\noindent\textbf{Evaluation metrics.} \label{sec:exp-metrics}
We use three metrics:
(1) {coordinate NRMSE} $e_{\rm map}$ measures held-out coordinate recovery
after a non-test-fitted alignment;
(2) {parameter relative error} $e_\theta$ measures recovery at the
corresponding parameter-orbit point or in calibrated physical units; and
(3) {dynamics relative RMSE} $e_{\rm dyn}$ measures agreement of the
transformed law with reference acceleration or force.
Definitions and reporting conventions are in
Appendix~\ref{app:metric-conventions}, with interpretation limits in
Appendix~\ref{app:operational-details}.

\subsection{Coordinate ambiguity and parameter information}
\label{sec:exp-orbit}

To test the parameter information predicted in Section~\ref{sec:scalar-theory},
we simulate the ODE forms in Table~\ref{tab:regimes} and the quadratic/Helmholtz
family (Appendix~\ref{app:helmholtz-branches}), rendering pendulum or
spring-oscillator videos (Figure~\ref{fig:exp-synth-videos}). Each family has
fixed physical coefficients (Table~\ref{tab:exp-synth-actions}) and rendering,
with ten initial-condition (IC) collections crossed with ten optimizer seeds,
giving 100 fits. IC collections vary the initial physical states; optimizer
seeds control model initialization and training sampling. Allowed coordinate maps
fitted on training references are frozen for held-out coordinate and
parameter comparisons. Appendix~\ref{app:exp-synth-protocol} gives the
observations, clip splits, and training details.


\begin{figure}[!t]
\centering
\vspace{-2mm}
\includegraphics[width=0.98\linewidth]{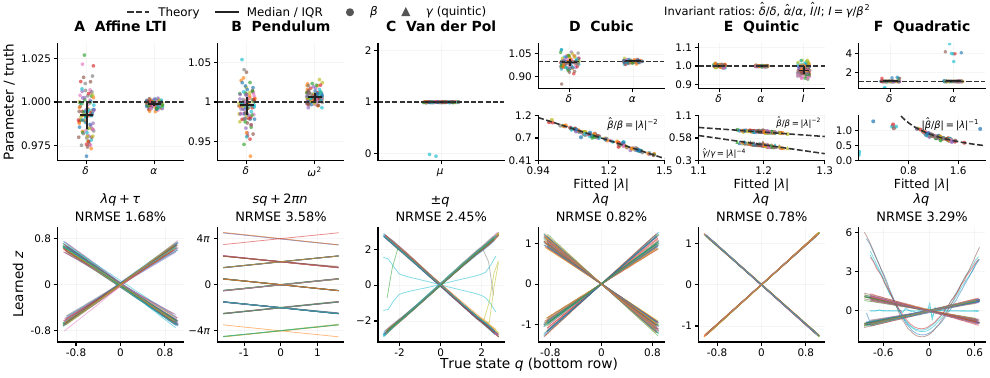}
\vspace{-3mm}
\caption{\textbf{Coordinate maps explain parameter orbits.}
Top: invariant ratios and coefficient orbits; bottom: learned coordinates
and allowed affine maps. Colors match across rows; D--F reuse
training-fitted scales for testing.}
\label{fig:exp-orbit}
\vspace{-3mm}
\end{figure}

Figure~\ref{fig:exp-orbit} supports the predicted link between coordinate
ambiguity and parameter recovery. Affine LTI, pendulum, and normalized
Van der Pol retain their respective affine, sign/periodic, and sign gauges
while recovering nearly invariant coefficients.
Cubic and quintic coefficients follow
$\hat\beta/\beta\simeq|\lambda|^{-2}$ and
$\hat\gamma/\gamma\simeq|\lambda|^{-4}$, while $\delta,\alpha$ and the
quintic ratio $\gamma/\beta^2$ remain nearly invariant
(Eq.~\eqref{eq:weight-scaling}); quadratic follows
$\hat\beta/\beta\simeq\lambda^{-1}$ on its observed origin branch.
Using the same training-fitted $\lambda$ for coordinate and parameter-orbit
comparisons gives median coordinate NRMSE of $0.78$--$3.58\%$ and parameter
relative error of $0.31$--$2.56\%$ (Table~\ref{tab:app-exp-synth-cohort}).
To examine ODE-family misspecification, we also fit affine-LTI models to the
same observations and ICs with matched training settings: the declared
nonlinear family gives lower test dynamics error $e_{\rm dyn}$ in 92--100 of 100
paired fits per family (Appendix~\ref{app:exp-synth-adequacy}).

\subsection{From velocity coverage to canonical coordinates}
\label{sec:exp-boundary}

Using synthetic videos, we test Section~\ref{sec:canonical-theory}'s two
routes: velocity coverage for raw affine alignment, and canonical
normalization for laws with a squared-velocity channel.

\noindent\textbf{Velocity coverage and parameter recovery.}
We test Section~\ref{subsec:feature-separation} on spring-cart videos governed by
$q''=-kq-\delta q'-\kappa q'|q'|$, with
$(k,\delta,\kappa)=(0.8,0.18,0.22)$, and jointly learn an encoder and
$z''=c_0(z)+c_1(z)z'+c_d(z)z'|z'|$.
Full rank of $[1,v,v|v|,v^2]$ predicts affine coordinates and the drag
relation $c_d=-\kappa/|\lambda|$. All conditions use 24 training clips and 5,000 updates; observations,
protocol details, and additional results are provided in
Appendix~\ref{app:exp-synth-coverage}.

\begin{table}[t]
\vspace{-3mm}
\centering\small
\setlength{\tabcolsep}{4pt}
\caption{\textbf{Accurate coordinates do not ensure accurate drag coefficients.}
Ten IC collections $\times$ five optimizer seeds per condition.
Rank is the minimum for $[1,v,v|v|,v^2]$ over shared positions.
$K$ counts speed magnitudes; $e_\theta$ tests the scale-corrected mean drag coefficient.}
\label{tab:exp-coverage}
\begin{tabular}{@{}lccc@{}}
\toprule
Velocity design & Rank & Coordinate $e_{\rm map}$ (\%) & Parameter $e_\theta$ (\%)\\
\midrule
One-way, $K=6$ & $3/4$ & 1.18 [1.08, 1.40] & 44.97 [15.07, 59.87]\\
Two-way, $K=1$ & $2/4$ & 1.08 [0.91, 1.34] & 6.95 [4.85, 8.65]\\
Two-way, $K=2$ & $4/4$ & 1.07 [0.93, 1.22] & 1.34 [0.85, 2.60]\\
\bottomrule
\end{tabular}
\vspace{-3mm}
\end{table}

In Table~\ref{tab:exp-coverage}, coordinate NRMSE stays near $1\%$
while parameter errors differ substantially. One-way motion gives
$v|v|=v^2$, so six speeds still cannot separate drag from coordinate curvature. Two directions and two magnitudes give full rank
and reduce median drag error to $1.34\%$. Accurate coordinate fitting
alone therefore does not ensure accurate parameter recovery;
Figure~\ref{fig:exp-coverage-orbit} reports all six coverage conditions.


{\setlength{\intextsep}{4pt}
\begin{table}[t]
\vspace{-4mm}
\centering
\caption{\textbf{Law-derived normalization.} Median [Q25,Q75] over three fits.}
\label{tab:exp-normalization}
\small
\setlength{\tabcolsep}{10pt}
\begin{tabular}{@{}ccc@{}}
\toprule
Raw coordinate error $e_{\rm map}$ (\%) &
Canonical coordinate error $e_{\rm map}$ (\%) &
Restoring-force error $e_{\rm dyn}$ (\%)\\
\midrule
5.633 [5.629, 5.635] &
0.305 [0.303, 0.305] &
0.193 [0.183, 0.199]\\
\bottomrule
\end{tabular}
\vspace{-2mm}
\end{table}}

\noindent\textbf{Canonical normalization.}
We test whether the fitted-law normalization in Section~\ref{sec:connection-boundary}
removes coordinate distortion under the projected observation $y=\sin q$
(Figure~\ref{fig:exp-projection-geometry}). One encoder and law are jointly
trained on three equal-duration videos, released near $45^\circ,60^\circ,70^\circ$.
Raw and Canonical use the same held-out videos and matched affine readouts
fitted on training references (Appendix~\ref{app:exp-synth-canonical}).
From the result, normalization reduces coordinate NRMSE from $5.633\%$ to $0.305\%$
while recovering the restoring force to $0.193\%$ relative RMSE
(Table~\ref{tab:exp-normalization}). Figure~\ref{fig:exp-normalization-residuals}
shows the removal of systematic distortion. Appendix~\ref{app:exp-synth-canonical-results}
reports detailed results, followed by additional velocity-coverage and
repeated-observation controls in Appendix~\ref{app:norm-velocity-controls}.

\subsection{Physical parameter recovery from real videos}
\label{sec:exp-real}\label{subsec:exp-real}
We use three real-video cases---real pendulum, side-view fall, and overhead
fall---to test invariant coefficients, affine parameter orbits, and canonical
normalization, respectively, and then obtain physical parameter estimates.
The real pendulum recordings come
from IRIS~\citep{khanbayov2026iris}; the side-view and overhead falls are
self-captured. Each case contains multiple videos from repeated recordings
under multiple initial conditions. All three cases
use encoder-only training with an ODE-residual loss and no frame-reconstruction
target. Observations, data details, training procedures, full evaluation,
and additional results and discussion are provided in
Appendix~\ref{app:exp-real}.

Table~\ref{tab:exp-real-parameters} summarizes physical parameter recovery
after anchor calibration. \textit{(1) Pendulum.} With known gravity, the nonlinear pendulum
coefficient yields a length of $0.5150$ m against the IRIS reference
$0.50$ m, reducing median parameter error from the matched LTI model's
$12.92\%$ to $2.92\%$. This comparison shows the estimation cost of
replacing the nonlinear restoring law by its small-angle approximation.
\textit{(2) Side fall.} A ball-diameter anchor converts the scale-corrected
acceleration to $\hat g=9.370\,\mathrm{m/s^2}$, with median error
$5.10\%$. 
\textit{(3) Overhead fall.}
The learned curvature coefficient is
$\hat\rho=2.072\pm0.024$, close to the projective value $\rho=2$.
Canonical normalization reduces median coordinate NRMSE from
$15.85\%$ to $5.57\%$; a fixed post-fit release-height calibration
then gives $\hat g=10.146\,\mathrm{m/s^2}$, with median error $3.35\%$.
These results illustrate how invariant recovery, affine parameter orbits,
and canonical normalization combine with external calibration in
real-video settings.

\begin{table}[t]
\vspace{-2mm}
\centering\small
\caption{\textbf{Real-video coordinate and physical parameter recovery.}
Estimates are mean $\pm$ SD; errors are median [Q25,Q75].
$L$ is in m and $g$ in $\mathrm{m/s^2}$.}
\label{tab:exp-real-parameters}
\setlength{\tabcolsep}{4pt}
\resizebox{\linewidth}{!}{\begin{tabular}{@{}llcccc@{}}
\toprule
Experiment & Parameter & Reference & Estimate & $e_\theta$ (\%) & $e_{\rm map}$ (\%)\\
\midrule
Pendulum (nonlinear) & $L$ & 0.50 & $0.5150\pm0.0012$ & 2.92 [2.82, 3.17] & 5.04 [4.38, 5.83]\\
Pendulum (LTI) & $L_{\rm eff}$ & 0.50 & $0.5646\pm0.0005$ & 12.92 [12.83, 12.97] & 7.13 [5.46, 7.45]\\
Side fall & $g$ & 9.81 & $9.370\pm0.771$ & 5.10 [3.21, 7.42] & 3.75 [3.22, 5.14]\\
Overhead fall & $g$ & 9.81 & $10.146\pm0.029$ & 3.35 [3.26, 3.39] & 5.57 [5.57, 5.60]\\
\bottomrule
\end{tabular}}
\vspace{-3mm}
\end{table}

\label{page:experiments-end}

\section{Discussion and limitations}
\label{sec:discussion}

We developed an identifiability analysis beyond LTI for the nonlinear scalar
families studied here. 
It characterizes compatible parameter values,
identifies invariant combinations, and gives sufficient conditions for
recovering further parameter information through physical calibration.
Synthetic and real-video experiments support the predicted parameter
relations, coverage effects, and calibration behavior. The central
implication is that parameter identification need not determine a unique
state coordinate: residual coordinate freedom can leave the quantities of
interest unchanged. At the studied quadratic-velocity boundary, canonical
normalization enables law comparison, while physical calibration supplies
separate information about the original parameters.

\noindent\textbf{Limitations.}
The proposed analysis covers prescribed scalar autonomous ODE families and the stated
velocity structures on regular coordinate intervals. It does not yet extend
to coupled multidimensional systems, where finite compact smooth trajectories
leave off-trajectory coordinate freedom that the scalar argument does not
resolve. Appendix~\ref{app:multidim-proofs} explains this boundary, why the
scalar setting remains informative, and the gap between structural
identification and empirical parameter recovery.



\bibliographystyle{iclr2027_conference}
\bibliography{refs}

\begin{thebibliography}{41}
\providecommand{\natexlab}[1]{#1}
\providecommand{\url}[1]{\texttt{#1}}
\expandafter\ifx\csname urlstyle\endcsname\relax
  \providecommand{\doi}[1]{doi: #1}\else
  \providecommand{\doi}{doi: \begingroup \urlstyle{rm}\Url}\fi

\bibitem[Ahuja et~al.(2022)Ahuja, Hartford, and Bengio]{ahuja2022properties}
Kartik Ahuja, Jason Hartford, and Yoshua Bengio.
\newblock Properties from mechanisms: An equivariance perspective on
  identifiable representation learning.
\newblock In \emph{International Conference on Learning Representations}, 2022.

\bibitem[Assran et~al.(2025)Assran, Bardes, Fan, Garrido, Howes, Komeili,
  Muckley, Rizvi, Roberts, Sinha, Zholus, Arnaud, Gejji, Martin, Hogan, Dugas,
  Bojanowski, Khalidov, Labatut, Massa, Szafraniec, Krishnakumar, Li, Ma,
  Chandar, Meier, LeCun, Rabbat, and Ballas]{assran2025vjepa2}
Mido Assran, Adrien Bardes, David Fan, Quentin Garrido, Russell Howes, Mojtaba
  Komeili, Matthew Muckley, Ammar Rizvi, Claire Roberts, Koustuv Sinha, Artem
  Zholus, Sergio Arnaud, Abha Gejji, Ada Martin, Francois~Robert Hogan, Daniel
  Dugas, Piotr Bojanowski, Vasil Khalidov, Patrick Labatut, Francisco Massa,
  Marc Szafraniec, Kapil Krishnakumar, Yong Li, Xiaodong Ma, Sarath Chandar,
  Franziska Meier, Yann LeCun, Michael Rabbat, and Nicolas Ballas.
\newblock {V-JEPA}~2: Self-supervised video models enable understanding,
  prediction and planning.
\newblock \emph{arXiv preprint arXiv:2506.09985}, 2025.

\bibitem[Bai et~al.(2023)Bai, Sezen, Yilmaz, and Qin]{bai2023bridge}
Yongsheng Bai, Halil Sezen, Alper Yilmaz, and Rongjun Qin.
\newblock Bridge vibration measurements using different camera placements and
  techniques of computer vision and deep learning.
\newblock \emph{Advances in bridge engineering}, 4\penalty0 (1):\penalty0 25,
  2023.

\bibitem[Baumgartner et~al.(2026)Baumgartner, Lei, Watson, and
  Posner]{baumgartner2026disentangling}
Markus~W. Baumgartner, Anson Lei, Joe Watson, and Ingmar Posner.
\newblock Disentangling dynamical systems: Causal representation learning meets
  local sparse attention.
\newblock In \emph{Proceedings of the Fifth Conference on Causal Learning and
  Reasoning}, volume 323 of \emph{Proceedings of Machine Learning Research},
  pp.\  119--165. PMLR, 2026.

\bibitem[Borgqvist et~al.(2026)Borgqvist, Browning, Ohlsson, and
  Baker]{borgqvist2026framing}
Johannes~G. Borgqvist, Alexander~P. Browning, Fredrik Ohlsson, and Ruth~E.
  Baker.
\newblock Framing local structural identifiability and observability in terms
  of parameter-state symmetries.
\newblock \emph{arXiv preprint arXiv:2603.11387}, 2026.

\bibitem[Bruce et~al.(2024)Bruce, Dennis, Edwards, Parker-Holder, Shi, Hughes,
  Lai, Mavalankar, Steigerwald, Apps, et~al.]{bruce2024genie}
Jake Bruce, Michael~D Dennis, Ashley Edwards, Jack Parker-Holder, Yuge Shi,
  Edward Hughes, Matthew Lai, Aditi Mavalankar, Richie Steigerwald, Chris Apps,
  et~al.
\newblock Genie: Generative interactive environments.
\newblock In \emph{Forty-first international conference on machine learning},
  2024.

\bibitem[Brunton et~al.(2016)Brunton, Proctor, and
  Kutz]{brunton2016discovering}
Steven~L. Brunton, Joshua~L. Proctor, and J.~Nathan Kutz.
\newblock Discovering governing equations from data by sparse identification of
  nonlinear dynamical systems.
\newblock \emph{Proceedings of the National Academy of Sciences}, 113\penalty0
  (15):\penalty0 3932--3937, 2016.
\newblock \doi{10.1073/pnas.1517384113}.

\bibitem[Champion et~al.(2019)Champion, Lusch, Kutz, and
  Brunton]{champion2019data}
Kathleen Champion, Bethany Lusch, J.~Nathan Kutz, and Steven~L. Brunton.
\newblock Data-driven discovery of coordinates and governing equations.
\newblock \emph{Proceedings of the National Academy of Sciences}, 116\penalty0
  (45):\penalty0 22445--22451, 2019.
\newblock \doi{10.1073/pnas.1906995116}.

\bibitem[Chen et~al.(2018)Chen, Rubanova, Bettencourt, and
  Duvenaud]{chen2018neural}
Ricky T.~Q. Chen, Yulia Rubanova, Jesse Bettencourt, and David~K. Duvenaud.
\newblock Neural ordinary differential equations.
\newblock In \emph{Advances in Neural Information Processing Systems},
  volume~31, 2018.

\bibitem[D{\'\i}az-Seoane et~al.(2023)D{\'\i}az-Seoane, Rey~Barreiro, and
  Villaverde]{diazseoane2023strike}
Sandra D{\'\i}az-Seoane, Xabier Rey~Barreiro, and Alejandro~F Villaverde.
\newblock Strike-goldd 4.0: user-friendly, efficient analysis of structural
  identifiability and observability.
\newblock \emph{Bioinformatics}, 39\penalty0 (1):\penalty0 btac748, 2023.

\bibitem[Dong et~al.(2023)Dong, Goodbrake, Harrington, and
  Pogudin]{dong2023differential}
R.~Dong, C.~Goodbrake, H.~Harrington, and G.~Pogudin.
\newblock Differential elimination for dynamical models via projections with
  applications to structural identifiability.
\newblock \emph{SIAM Journal on Applied Algebra and Geometry}, 7\penalty0
  (1):\penalty0 194--235, 2023.
\newblock \doi{10.1137/22M1469067}.

\bibitem[Foo et~al.(2023)Foo, Heyd, and Merker]{foo2023normal}
Wei~Guo Foo, Julien Heyd, and Jo{\"e}l Merker.
\newblock Normal forms of second-order ordinary differential equations
  {$y_{xx}=J(x,y,y_x)$} under fibre-preserving maps.
\newblock \emph{Complex Analysis and its Synergies}, 9:\penalty0 10, 2023.
\newblock \doi{10.1007/s40627-023-00121-x}.

\bibitem[Garcia et~al.(2025)Garcia, Warchocki, van Gemert, Brinks, and
  Tomen]{garcia2025learning}
Alejandro~Casta{\~n}eda Garcia, Jan Warchocki, Jan van Gemert, Daan Brinks, and
  Nergis Tomen.
\newblock Learning physics from video: Unsupervised physical parameter
  estimation for continuous dynamical systems.
\newblock In \emph{Proceedings of the Computer Vision and Pattern Recognition
  Conference}, pp.\  27924--27933, 2025.

\bibitem[Gonz{\'a}lez~Laiz et~al.(2025)Gonz{\'a}lez~Laiz, Schmidt, and
  Schneider]{laiz2025self}
Rodrigo Gonz{\'a}lez~Laiz, Tobias Schmidt, and Steffen Schneider.
\newblock Self-supervised contrastive learning performs non-linear system
  identification.
\newblock In \emph{International Conference on Learning Representations}, 2025.

\bibitem[Hofherr et~al.(2023)Hofherr, Koestler, Bernard, and
  Cremers]{hofherr2023neural}
Florian Hofherr, Lukas Koestler, Florian Bernard, and Daniel Cremers.
\newblock Neural implicit representations for physical parameter inference from
  a single video.
\newblock In \emph{Proceedings of the IEEE/CVF Winter Conference on
  Applications of Computer Vision}, pp.\  2093--2103, 2023.

\bibitem[Hong et~al.(2019)Hong, Ovchinnikov, Pogudin, and Yap]{hong2019sian}
Hoon Hong, Alexey Ovchinnikov, Gleb Pogudin, and Chee Yap.
\newblock Sian: software for structural identifiability analysis of ode models.
\newblock \emph{Bioinformatics}, 35\penalty0 (16):\penalty0 2873--2874, 2019.
\newblock \doi{10.1093/bioinformatics/bty1069}.

\bibitem[Jakubczyk(1980)]{jakubczyk1980existence}
Bronis{\l}aw Jakubczyk.
\newblock Existence and uniqueness of realizations of nonlinear systems.
\newblock \emph{SIAM Journal on Control and Optimization}, 18\penalty0
  (4):\penalty0 455--471, 1980.
\newblock \doi{10.1137/0318034}.

\bibitem[Jaques et~al.(2020)Jaques, Burke, and Hospedales]{jaques2020physics}
Miguel Jaques, Michael Burke, and Timothy Hospedales.
\newblock Physics-as-inverse-graphics: Unsupervised physical parameter
  estimation from video.
\newblock In \emph{International Conference on Learning Representations}, 2020.

\bibitem[Jaques et~al.(2022)Jaques, Asenov, Burke, and
  Hospedales]{jaques2022vision}
Miguel Jaques, Martin Asenov, Michael Burke, and Timothy Hospedales.
\newblock Vision-based system identification and 3d keypoint discovery using
  dynamics constraints.
\newblock In \emph{Proceedings of the 4th Annual Learning for Dynamics and
  Control Conference}, volume 168 of \emph{Proceedings of Machine Learning
  Research}, pp.\  316--329, 2022.

\bibitem[Jatavallabhula et~al.(2021)Jatavallabhula, Macklin, Golemo, Voleti,
  Petrini, Weiss, Considine, Parent-L{\'e}vesque, Xie, Erleben,
  et~al.]{jatavallabhula2021gradsim}
Krishna~Murthy Jatavallabhula, Miles Macklin, Florian Golemo, Vikram Voleti,
  Linda Petrini, Martin Weiss, Breandan Considine, J{\'e}r{\^o}me
  Parent-L{\'e}vesque, Kevin Xie, Kenny Erleben, et~al.
\newblock {gradSim}: Differentiable simulation for system identification and
  visuomotor control.
\newblock In \emph{International Conference on Learning Representations}, 2021.

\bibitem[Joseph et~al.(2026)Joseph, Garrido, Balestriero, Kowal, Fel,
  Bakhtiari, Richards, and Rabbat]{joseph2026interpreting}
Sonia Joseph, Quentin Garrido, Randall Balestriero, Matthew Kowal, Thomas Fel,
  Shahab Bakhtiari, Blake Richards, and Mike Rabbat.
\newblock Interpreting physics in video world models.
\newblock \emph{arXiv preprint arXiv:2602.07050}, 2026.

\bibitem[Kandukuri et~al.(2022)Kandukuri, Achterhold, Moeller, and
  Stueckler]{kandukuri2022physical}
Rama~Krishna Kandukuri, Jan Achterhold, Michael Moeller, and Joerg Stueckler.
\newblock Physical representation learning and parameter identification from
  video using differentiable physics.
\newblock \emph{International Journal of Computer Vision}, 130\penalty0
  (1):\penalty0 3--16, 2022.

\bibitem[Kang et~al.(2025)Kang, Yue, Lu, Lin, Zhao, Wang, Huang, and
  Feng]{kang2024far}
Bingyi Kang, Yang Yue, Rui Lu, Zhijie Lin, Yang Zhao, Kaixin Wang, Gao Huang,
  and Jiashi Feng.
\newblock How far is video generation from world model: A physical law
  perspective.
\newblock In \emph{Proceedings of the 42nd International Conference on Machine
  Learning}, volume 267 of \emph{Proceedings of Machine Learning Research},
  pp.\  28991--29017, 2025.

\bibitem[Kanko et~al.(2021)Kanko, Laende, Strutzenberger, Brown, Selbie,
  DePaul, Scott, and Deluzio]{kanko2021markerless}
Robert~M Kanko, Elise~K Laende, Gerda Strutzenberger, Marcus Brown, W~Scott
  Selbie, Vincent DePaul, Stephen~H Scott, and Kevin~J Deluzio.
\newblock Assessment of spatiotemporal gait parameters using a deep learning
  algorithm-based markerless motion capture system.
\newblock \emph{Journal of biomechanics}, 122:\penalty0 110414, 2021.

\bibitem[Khanbayov et~al.(2026)Khanbayov, Barhdadi, Serpedin, and
  Kurban]{khanbayov2026iris}
Rasul Khanbayov, Mohamed~Rayan Barhdadi, Erchin Serpedin, and Hasan Kurban.
\newblock Iris: A real-world benchmark for inverse recovery and identification
  of physical dynamic systems from monocular video.
\newblock In \emph{European Conference on Computer Vision}, pp.\  366--383.
  Springer, 2026.

\bibitem[Kossovskiy \& Zaitsev(2018)Kossovskiy and
  Zaitsev]{kossovskiy2018normal}
Ilya Kossovskiy and Dmitri Zaitsev.
\newblock Normal form for second order differential equations.
\newblock \emph{Journal of Dynamical and Control Systems}, 24\penalty0
  (4):\penalty0 541--562, 2018.
\newblock \doi{10.1007/s10883-017-9380-9}.

\bibitem[Kruglikov(2009)]{kruglikov2009point}
Boris Kruglikov.
\newblock Point classification of second-order {ODEs}: {Tresse} classification
  revisited and beyond.
\newblock In Boris Kruglikov, Valentin Lychagin, and Eldar Straume (eds.),
  \emph{Differential Equations: Geometry, Symmetries and Integrability},
  volume~5 of \emph{Abel Symposia}, pp.\  199--221. Springer, Berlin,
  Heidelberg, 2009.
\newblock \doi{10.1007/978-3-642-00873-3_10}.

\bibitem[Lusch et~al.(2018)Lusch, Kutz, and Brunton]{lusch2018deep}
Bethany Lusch, J.~Nathan Kutz, and Steven~L. Brunton.
\newblock Deep learning for universal linear embeddings of nonlinear dynamics.
\newblock \emph{Nature Communications}, 9:\penalty0 4950, 2018.
\newblock \doi{10.1038/s41467-018-07210-0}.

\bibitem[Muratore \& Mathis(2026)Muratore and Mathis]{muratore2026extracting}
Paolo Muratore and Mackenzie~Weygandt Mathis.
\newblock Extracting governing equations from latent dynamics via multi-view
  contrastive learning.
\newblock \emph{arXiv preprint arXiv:2606.13260}, 2026.

\bibitem[Mustafa(2015)]{mustafa2015position}
Omar Mustafa.
\newblock Position-dependent mass lagrangians: Nonlocal transformations,
  {Euler--Lagrange} invariance and exact solvability.
\newblock \emph{Journal of Physics A: Mathematical and Theoretical},
  48\penalty0 (22):\penalty0 225206, 2015.
\newblock \doi{10.1088/1751-8113/48/22/225206}.

\bibitem[Olver(1995)]{olver1995equivalence}
Peter~J. Olver.
\newblock \emph{Equivalence, Invariants and Symmetry}.
\newblock Cambridge University Press, Cambridge, 1995.
\newblock ISBN 978-0-521-47811-3.
\newblock \doi{10.1017/CBO9780511609565}.

\bibitem[Paliathanasis et~al.(2025)Paliathanasis, Moyo, and
  Leach]{paliathanasis2025geometric}
A.~Paliathanasis, S.~Moyo, and P.~G.~L. Leach.
\newblock A geometric interpretation for the algebraic properties of
  second-order ordinary differential equations.
\newblock \emph{Mathematical Methods in the Applied Sciences}, 48:\penalty0
  6912--6917, 2025.
\newblock \doi{10.1002/mma.10726}.

\bibitem[Raue et~al.(2009)Raue, Kreutz, Maiwald, Bachmann, Schilling,
  Klingmüller, and Timmer]{raue2009structural}
A.~Raue, C.~Kreutz, T.~Maiwald, J.~Bachmann, M.~Schilling, U.~Klingmüller, and
  J.~Timmer.
\newblock Structural and practical identifiability analysis of partially
  observed dynamical models by exploiting the profile likelihood.
\newblock \emph{Bioinformatics}, 25\penalty0 (15):\penalty0 1923--1929, 08
  2009.
\newblock ISSN 1367-4803.
\newblock \doi{10.1093/bioinformatics/btp358}.

\bibitem[Rubanova et~al.(2019)Rubanova, Chen, and Duvenaud]{rubanova2019latent}
Yulia Rubanova, Ricky T.~Q. Chen, and David~K. Duvenaud.
\newblock Latent ordinary differential equations for irregularly-sampled time
  series.
\newblock In \emph{Advances in Neural Information Processing Systems},
  volume~32, 2019.

\bibitem[Sussmann(1976)]{sussmann1976existence}
H{\'e}ctor~J. Sussmann.
\newblock Existence and uniqueness of minimal realizations of nonlinear
  systems.
\newblock \emph{Mathematical Systems Theory}, 10\penalty0 (1):\penalty0
  263--284, 1976.
\newblock \doi{10.1007/BF01683278}.

\bibitem[Tragoudaras et~al.(2025)Tragoudaras, Zhang, Cherniavskii, Vozikis,
  Nijdam, Prinzhorn, Bodracska, Sebe, Zadaianchuk, and
  Gavves]{tragoudaras2025evaluating}
Antonios Tragoudaras, Chenyu Zhang, Daniil Cherniavskii, Antonios Vozikis,
  Thijmen Nijdam, Derck~WE Prinzhorn, Mark Bodracska, Nicu Sebe, Andrii
  Zadaianchuk, and Efstratios Gavves.
\newblock Evaluating newtonian mechanics in video generative models with real
  physical systems.
\newblock \emph{arXiv preprint arXiv:2504.02918}, 2025.

\bibitem[Uhlrich et~al.(2023)Uhlrich, Falisse, Kidzi{\'n}ski, Muccini, Ko,
  Chaudhari, Hicks, and Delp]{uhlrich2023opencap}
Scott~D Uhlrich, Antoine Falisse, {\L}ukasz Kidzi{\'n}ski, Julie Muccini,
  Michael Ko, Akshay~S Chaudhari, Jennifer~L Hicks, and Scott~L Delp.
\newblock Opencap: Human movement dynamics from smartphone videos.
\newblock \emph{PLoS computational biology}, 19\penalty0 (10):\penalty0
  e1011462, 2023.

\bibitem[Villaverde \& Massonis(2021)Villaverde and
  Massonis]{villaverde2021scaling}
Alejandro~F. Villaverde and Gemma Massonis.
\newblock On testing structural identifiability by a simple scaling method:
  Relying on scaling symmetries can be misleading.
\newblock \emph{PLOS Computational Biology}, 17\penalty0 (10):\penalty0
  e1009032, 2021.
\newblock \doi{10.1371/journal.pcbi.1009032}.

\bibitem[Wang et~al.(2026)Wang, Wang, Zhang, and Gong]{wang2026physics}
Yuanyuan Wang, Wenjie Wang, Kun Zhang, and Mingming Gong.
\newblock Physics from video: Identifiability of time-invariant second-order
  odes under minimal trajectory conditions.
\newblock In \emph{International Conference on Machine Learning}, 2026.

\bibitem[Xu \& Brownjohn(2018)Xu and Brownjohn]{xu2018vision}
Yan Xu and James~MW Brownjohn.
\newblock Review of machine-vision based methodologies for displacement
  measurement in civil structures.
\newblock \emph{Journal of Civil Structural Health Monitoring}, 8\penalty0
  (1):\penalty0 91--110, 2018.

\bibitem[Yao et~al.(2024)Yao, Muller, and Locatello]{yao2024marrying}
Dingling Yao, Caroline Muller, and Francesco Locatello.
\newblock Marrying causal representation learning with dynamical systems for
  science.
\newblock In \emph{Advances in Neural Information Processing Systems},
  volume~37, pp.\  71705--71736, 2024.
\newblock \doi{10.52202/079017-2290}.

\end{thebibliography}

\appendix

\section{Notation and Domain Conventions}
\label{app:notation}

\subsection{Notation}
\label{app:notation_index}

Table~\ref{tab:notation_index} collects the main symbols. Primes on
continuous-time paths denote time derivatives; primes on a state map or
normalizer denote derivatives with respect to its scalar argument.
The operators $D_{1,m},D_{2,m}$ denote numerical differences instead.

\begingroup
\small
\setlength{\LTleft}{0pt}
\setlength{\LTright}{0pt}
\renewcommand{\arraystretch}{1.12}
\begin{longtable}{@{}>{\raggedright\arraybackslash}p{0.27\linewidth}
                    >{\raggedright\arraybackslash}p{\dimexpr0.73\linewidth-2\tabcolsep\relax}@{}}
\caption{Notation index.}\label{tab:notation_index}\\
\toprule
Symbol & Meaning \\
\midrule
\endfirsthead
\multicolumn{2}{l}{\tablename~\thetable\ (continued)}\\
\toprule
Symbol & Meaning \\
\midrule
\endhead
\midrule
\multicolumn{2}{r}{Continued on next page}\\
\endfoot
\bottomrule
\endlastfoot
$\mathcal D$ & Observed collection of timestamped video frames. \\
$m,M,I^{(m)},T_m$ & Clip index, number of clips, underlying time interval,
and final frame index; clip $m$ has $T_m+1$ frames. \\
$t_{m,k},\Delta t_m$ & Recorded timestamp and known uniform within-clip
sampling interval. All dynamical derivatives use this time coordinate. \\
$x^{(m)}(t),x_k^{(m)}$ & Ideal frame path and its recorded sample at $t_{m,k}$. \\
$z^{(m)}(t)$ & Unobserved data-generating scalar physical coordinate. \\
$E_\phi$ & Shared per-frame scalar encoder. Its output
enters the ODE residual directly. \\
$\hat z^{(m)}(t),\hat z_k^{(m)}$ & Ideal encoder trajectory and its sampled values. \\
$\Theta;\ \xi,\theta,\eta$ & Declared parameter space; a generic member,
the physical parameter, and a fitted encoder-coordinate parameter. \\
$F_\xi,D_\xi$ & Declared right-hand-side law and its open state domain;
$z''=F_\xi(z,z')$. The form of $F_\xi$ is specialized by each result. \\
$a_\xi,b_\xi$ & Semilinear coefficients in
$z''+a_\xi(z)z'+b_\xi(z)=0$. \\
$c,\hat c$ & Physical and reported equilibrium coordinates in the affine-LTI
family; distinct from the residual-coefficient vector $c(u)$. \\
$f$ & Analysis-only shared state map, $f\in C^2(U)$; $\hat z=f(z)$ on
the stated interval. \\
$R_z$ & Union of realized physical state values, without velocity information; $R_z\subset D_\theta$. \\
$U$ & Nonempty open physical-state interval on which coverage and
identification claims are asserted. \\
$K=[u_-,u_+]\Subset U$ & Nondegenerate compact interval contained in $U$,
used for quantitative stability; $K$ also denotes a basis size or a
speed-magnitude count where stated locally. \\
$\mathcal R_{\eta,\phi}^{(m)}(t)$ & Continuous-time fitted-law residual. \\
$\mathcal J(\xi)$ & Parameter functional whose agreement across admissible
reports is the identification target. \\
$D_{1,m},D_{2,m}$ & Three-point centered first- and second-difference
operators within clip $m$. \\
$\widehat r_{m,k},N_{\mathrm{int}}$ & Sampled residual and total number
$\sum_m(T_m-1)$ of interior sample indices. \\
$\mathcal L_{\mathrm{ODE}},\mathcal L_{\mathrm{var}},\mathcal L_{\mathrm{total}}$
& Residual loss, variance-floor penalty, and total sampled fitting objective. \\
$p,d_v,c_{r,\xi}$ & Declared velocity degree, augmented degree
$d_v=\max\{2,p\}$, and continuous coefficient functions in the
polynomial-velocity specialization. \\
$\mathsf V$ & Fixed finite-dimensional velocity-feature space, with
constants and dilation stability where required. \\
$\boldsymbol\phi_{\rm sl}(v)$ & Semilinear augmented feature vector
$(v^2,v,1)^\top$. \\
$\Phi_u,\sigma_{\min}(\Phi_u),\sigma_\star$ & Physical-velocity design,
its smallest singular value, and a positive lower bound. For semilinear
matching, rows are $\boldsymbol\phi_{\rm sl}(v_i)^\top$. Raw conditioning is
scale-dependent. \\
$M_u$ & Augmented feature design in the curvature-exclusion theorem,
including the extra $v^2$ column. \\
$\psi_\xi,J_\xi$ & Boundary-only normalizer satisfying
$\psi_\xi''+c_{2,\xi}\psi_\xi'=0$ and its image
$J_\xi=\psi_\xi(D_\xi)$. \\
$\bar c_{r,\xi}$ & Coefficients of the canonical law with quadratic
channel set to zero. \\
$A(q)=\lambda q+\tau$ & Affine witness between canonical laws; the raw
witness is $\psi_\eta^{-1}\circ A\circ\psi_\theta$. \\
$g$ & Canonicalized state map
$g=\psi_\eta\circ f\circ\psi_\theta^{-1}$; distinct from the gravity
constant $g$ in the real-video experiments. \\
$\mathcal O_U(\theta)$ & Family-compatible affine parameter outer set
on $U$, obtained by projecting out the affine witness. \\
$\mathcal O_U^{\mathrm{can}}(\theta)$ & Canonical compatible parameter
outer set, restricted back to the declared $\Theta$. \\
$\mathcal S_U(\theta)$ & Lifted affine witness--parameter set retaining
$(\lambda,\tau,\eta)$ rather than only $\eta$. \\
$\mathcal S_U^{\mathrm{can}}(\theta)$ & Canonical lifted set; raw physical
anchors are evaluated through the law-dependent inverse normalizer. \\
$d_{\mathcal S}$ & Local manifold dimension used by the anchor-rank criterion. \\
$I_a,I_b;\ A_i,B_j$ & Structural polynomial supports and coefficients
of $u^i$ in damping and $u^j$ in restoring terms. \\
$w(A_i)=i,\ w(B_j)=j-1$ & Scaling weights in the translation-locked
semilinear polynomial family. \\
$\varepsilon,h,V,\delta_v$ & Residual tolerance, level mismatch,
physical-velocity bound, and pairwise velocity separation; $h$ also
denotes a height where stated locally. \\
$\mathcal T_{\mathrm{sel}},\varepsilon_{\mathrm{sc}}$ & Selected clip--time
pairs and uniform value/first-/second-derivative error for approximate
continuous state consistency. \\
$c(u),L_c$ & Semilinear residual-coefficient vector and its Lipschitz constant. \\
$c_0,\Lambda$ & Bounds $0<c_0\le |f'|\le\Lambda$ used only in
quantitative results; $c_0$ here is not $c_{0,\xi}$. \\
$\xi_j,\sigma_0$ & Additional sample perturbation and its bound in
Appendix~\ref{app:residual_accounting}; $\xi_j$ is distinct from a law index $\xi$. \\
$\delta_0,\delta_1,\delta_2$ & Local value, first-difference, and
second-difference errors relative to a smooth reference path. \\
\end{longtable}
\endgroup

\subsection{Domain and Regularity Conventions}
\label{app:domain_conventions}

Each declared law is evaluated only on its stated domain. A shared map
used on $U\subset R_z$ must satisfy $f(U)\subset D_\eta$.
An affine witness $u\mapsto\lambda u+\tau$ must map $U$ into $D_\eta$;
a canonical affine witness must map $\psi_\theta(U)$ into $J_\eta$.
These inclusions apply to all displayed coefficient identities and
inverse-normalizer expressions. Sampled residuals likewise require their
central encoded states to lie in the fitted coefficient domain.

A law identity proved on $U$ is only a restricted-law statement unless
a polynomial or real-analytic identity theorem is explicitly invoked
on the connected common domain. No such extension is assumed for
general continuous or merely smooth coefficients. Quantitative results
state their compact intervals, bounds, and additional regularity locally.
In particular, the $C^4$ condition used for centered-difference error
bounds is not a standing hypothesis of the exact identification theory.
Here $K\Subset U$ means that $K$ is compact and contained in the open
interval $U$.

Latent-state consistency in Definition~\ref{def:state-consistency} is
imposed along the continuous trajectories
whenever the physical state lies in $U$, whereas exact fitted-law
compatibility is required only at the covering points specified by each
result. The local condition is local in the physical coordinate, not a
restriction to parameters near $\theta$.

For general interval-local coefficients, the fixed-$U$ parameter
compatibility projection need not be symmetric and is therefore called an
outer set. Literal law orbits carry their domains along with the coordinate
change. Where a global law action is defined, intersecting an ambient law
orbit with the declared family provides an orbit-based interpretation.
These distinctions do not change the report-containment conclusions.

For semilinear laws, the right-hand-side convention is
$F_\xi(u,v)=-a_\xi(u)v-b_\xi(u)$. In polynomial-velocity notation this
means $c_{0,\xi}=-b_\xi$, $c_{1,\xi}=-a_\xi$, and all higher channels
zero. Missing polynomial channels are zero-padded when matching two laws.

\section{Extended related work}
\label{app:related-work-extended}
\label{sec:related}

We organize the literature around the distinction between predicting visual
motion, estimating dynamical parameters, and establishing which parameter
information is identifiable. These objectives connect the application
background in the Introduction to the representation, dynamical, and
calibration assumptions used in our analysis.

\noindent\textbf{Video prediction, physical evaluation, and measurement.}
Genie learns action-controllable environments from videos without action
labels \citep{bruce2024genie}, while V-JEPA~2 learns predictive video
representations and uses an action-conditioned model for robotic planning
\citep{assran2025vjepa2}. Physics-oriented evaluations examine a different
aspect of these models: whether their generated motion respects physical
principles \citep{kang2024far,tragoudaras2025evaluating}. In particular, Morpheus
uses real physical settings and conservation-law-based metrics to evaluate
Newtonian consistency. Work on interpreting video encoders also studies
whether physical variables can be decoded from learned representations
\citep{joseph2026interpreting}. Prediction, physical consistency, and variable
decodability are relevant to physical understanding, but none alone implies
unique recovery of a prescribed equation's parameters. Our task concerns
parameter information determined by observations and a declared dynamical
family, rather than the physical plausibility of generated videos.

Vision-based sensing provides a practical motivation for quantitative
inference from video. \citet{xu2018vision} organize structural-displacement
measurement around camera calibration, target tracking, and conversion to
physical displacement; \citet{bai2023bridge} investigate camera placements
and visual methods for bridge-vibration measurement. In biomechanics,
\citet{kanko2021markerless} evaluate video-based spatiotemporal gait
measurements, and OpenCap combines smartphone videos, learned motion
estimation, and musculoskeletal simulation to estimate human movement
kinematics and dynamics \citep{uhlrich2023opencap}. These works illustrate
the value of geometric and physical calibration, rather than implying that
all physical quantities follow from pixels alone. Our scalar results do
not cover their full structural or multibody models; they address which
parameter quantities dynamics constraints can determine and what external
measurements add when the state coordinate is learned.

\noindent\textbf{Physical parameter estimation from video.}
Physics-as-inverse-graphics estimates latent states and physical parameters
using known differential equations and a generative observation model
\citep{jaques2020physics}. Differentiable-physics approaches couple visual
inference to simulation \citep{kandukuri2022physical}; gradSim additionally
combines differentiable multiphysics simulation and rendering for
pixel-based system identification and control
\citep{jatavallabhula2021gradsim}. Other methods use dynamics constraints
for three-dimensional keypoint discovery \citep{jaques2022vision} or combine
neural implicit appearance models with parameterized ODEs for single-video
inference \citep{hofherr2023neural}. More directly related to our fitting
interface, \citet{garcia2025learning} estimate parameters of known continuous
governing equations without frame prediction and introduce Delfys75 for
real-video evaluation. IRIS expands this evaluation setting with controlled
real videos, governing equations, independently measured parameters, and
protocols addressing accuracy, identifiability, extrapolation, robustness,
and equation selection \citep{khanbayov2026iris}. These methods already
address nonlinear physical models. Our contribution is not to introduce
nonlinear video-based estimation, but to characterize its parameter
identification guarantees in the stated shared-encoder setting. We fix the
family in advance, so governing-equation selection is outside our task.

\noindent\textbf{Learning coordinates and governing equations.}
Sparse identification of nonlinear dynamics selects active terms from a
candidate equation library \citep{brunton2016discovering}. SINDy
autoencoders jointly learn coordinates and sparse equations
\citep{champion2019data}, while learned Koopman embeddings seek coordinates
with a simpler dynamical evolution \citep{lusch2018deep}. Neural ODE and
latent ODE models provide continuous-time learning architectures
\citep{chen2018neural,rubanova2019latent}. These approaches motivate treating
representation and dynamics jointly, but optimizing predictive fit or
sparsity does not by itself specify which physical coefficients must agree
across compatible coordinates. We instead fix the equation family and
analyze its parameter transformations. A sparse or normalized
representation is not automatically a physically calibrated one.

\noindent\textbf{Identifiable dynamical representations and the LTI case.}
Nonlinear realization theory studies state descriptions up to isomorphism
under its minimality and regularity assumptions
\citep{sussmann1976existence,jakubczyk1980existence}. In representation
learning, \citet{ahuja2022properties} characterize ambiguity through
transformations compatible with known mechanisms or mechanism classes.
\citet{yao2024marrying} study time-invariant, trajectory-specific physical
parameters and connect their identification to causal representation
learning; \citet{baumgartner2026disentangling} use local causal structure in
parameter disentanglement. DynCL obtains latent-system identification
through contrastive learning under stochastic-transition assumptions
\citep{laiz2025self}. DYSCO studies affine recovery from noisy multi-view
observations under asymptotic assumptions and also examines sparse symbolic
recovery along affine coefficient orbits \citep{muratore2026extracting}.
These are substantive nonlinear identification results, not merely
predictive models. Our complementary question concerns deterministic
compatibility of a shared scalar frame encoder with a fixed second-order
ODE family, using realized same-state velocity coverage.

The closest direct video-based starting point is
\citet{wang2026physics}, which establishes affine coordinate alignment and
coefficient identification for scalar LTI second-order ODEs. We extend the
structural analysis to nonlinear state and structured velocity dependence,
where affine alignment need not fix every coefficient and canonical-law
comparison may be needed. We distinguish family-compatible outer sets from
the parameter reports realizable by an encoder class. For semilinear laws,
true-coordinate representability and closure under invertible affine output
transformations make these sets exactly realizable
(Proposition~\ref{prop:affine-realizability}), yielding necessary and
sufficient identification criteria. This existence result is distinct from
statistical consistency or convergence of a learning objective. The
LTI-specific damping-regime and stochastic finite-sample results remain
separate contributions of \citet{wang2026physics}.

\noindent\textbf{Structural identifiability, invariants, and calibration.}
Structural identifiability asks whether a specified state--output model
uniquely determines parameters from ideal observations, including when some
states are unobserved. SIAN and STRIKE-GOLDD provide computational tests for
their respective model classes
\citep{hong2019sian,diazseoane2023strike}; differential elimination offers
another algebraic route \citep{dong2023differential}.
\citet{borgqvist2026framing} connect locally identifiable combinations to
invariants of output-preserving parameter--state symmetries.
\citet{villaverde2021scaling} show that examining scaling symmetries alone
can miss non-identifiability. Our coefficient-weight rule is therefore used
after affine reduction and a family-specific translation-lock argument,
not as a test for arbitrary nonlinear systems. Our added structure is the
explicit relation between passive-trajectory coverage, admissible encoder
coordinates, and compatible parameter sets. External anchors constrain a
lifted set retaining both the coordinate witness and the parameters
(Appendix~\ref{app:additional-consequences}). The restricted-Jacobian
criterion is a sufficient local condition; it neither settles all singular
cases nor removes disconnected branches automatically. The calculus and
audits support family-specific analysis, rather than constituting a general
automatic structural-identifiability solver.

Structural uniqueness and accuracy under finite noisy data are different
questions. Profile-likelihood analysis explicitly studies structural and
practical non-identifiability and parameter uncertainty
\citep{raue2009structural}. Our semilinear residual-stability results are
conditional deterministic bounds, not confidence intervals or
pixel-to-parameter statistical rates. The reporting checks in
Appendix~\ref{app:operational-details} distinguish fitted-law agreement,
state-consistency evidence, velocity coverage, admissible parameter
comparisons, and physical calibration. They connect the theory to empirical
evaluation without treating finite audits as proofs of continuous-time
premises.

\noindent\textbf{ODE equivalence and coordinate normalization.}
Equivalence and invariants of differential equations are classical topics
\citep{olver1995equivalence,kruglikov2009point}, with modern normal-form
results for second-order equations under point and fibre-preserving maps
\citep{kossovskiy2018normal,foo2023normal}. Geometric elimination of quadratic
derivative terms and transformations of position-dependent-mass
Lagrangians provide context for our normalization
\citep{paliathanasis2025geometric,mustafa2015position}. We use the classical
one-dimensional elimination of the squared-velocity channel; the
normalizing integral itself is not new. Our analysis restricts
transformations to the scalar state while preserving recorded time,
connects law compatibility to passive coefficient-matching conditions,
and restricts canonical affine compatibility back to the original
parameter family. The resulting fixed-interval parameter sets are not
automatically equivalence classes. Canonical normalization is determined 
by the law and does not itself supply physical units. 
It can be applied after fitting for law comparison, or incorporated directly into an integral fitting objective.

\clearpage
\section{Sampled Objectives and Error Accounting}
\label{app:sampled-interface}

Figure~\ref{fig:nonlinear-training-pipeline} connects the shared-encoder
objective in Eq.~\eqref{eq:training_objective} to parameter recovery.
Training uses frames and timestamps. The fitted law determines any required
canonical normalizer; family-compatible parameter sets then determine which
targets are already identified and which need external calibration.
Synthetic reference states are used separately to evaluate these predictions.
Appendix~\ref{app:operational-details} states what these comparisons establish.
Below we specify the sampled objective and variance floor, then relate sampled
to continuous residuals. Domain conventions are in
Appendix~\ref{app:domain_conventions}.

\begin{figure}[H]
\centering
\begingroup
\usetikzlibrary{positioning,calc}
\definecolor{pipeblue}{RGB}{47,91,127}
\definecolor{pipegreen}{RGB}{43,111,94}
\definecolor{pipeorange}{RGB}{180,103,55}
\begin{tikzpicture}[x=1cm,y=1cm,>=stealth,
  every node/.style={font=\small},
  stage/.style={draw=pipeblue!60,fill=pipeblue!4,rounded corners=3pt,align=center,inner sep=6pt},
  recover/.style={stage,draw=pipegreen!65,fill=pipegreen!4},
  flow/.style={->,draw=pipeblue,line width=.8pt},
  rflow/.style={->,draw=pipegreen,line width=.8pt}]
\node[anchor=west,font=\small\bfseries,text=pipeblue] at (-1.6,1.35)
  {(a) Learn a shared coordinate and dynamical law};
\foreach \xx/\ang/\lab/\id in {0/-32/{k-1}/a,4.5/0/{k}/b,9/32/{k+1}/c} {
  \draw[pipeblue!60,fill=white] (\xx-.60,0) rectangle (\xx+.60,.60);
  \fill[black!60] (\xx,.50) circle (.02);
  \draw[black!70,line width=.7pt] (\xx,.50) -- ++(\ang-90:.38)
    node[circle,fill=pipeblue,inner sep=1.6pt] {};
  \node[above,font=\footnotesize] at (\xx,.60) {$x^{(m)}_{\lab}$};
  \node[stage,minimum width=1.45cm,inner sep=3pt] (enc\id) at (\xx,-.48) {$E_\phi$};
  \draw[flow] (\xx,0) -- (enc\id.north);
  \node (z\id) at (\xx,-1.15) {$\hat z^{(m)}_{\lab}$};
  \draw[flow] (enc\id.south) -- (z\id.north);
}
\node[font=\footnotesize,text=pipeblue] at (2.25,-.48) {same $\phi$};
\node[font=\footnotesize,text=pipeblue] at (6.75,-.48) {same $\phi$};
\node[stage,text width=11.8cm] (train) at (4.5,-2.32)
  {Within-clip derivatives $\longrightarrow$ declared-law residual
   $\widehat r_{m,k}=D_2\hat z-F_\eta(\hat z,D_1\hat z)$\\[3pt]
   Jointly optimize $\phi,\eta$ using Eq.~\eqref{eq:training_objective}:\quad
   $\mathcal L_{\rm total}=N_{\rm int}^{-1}\!\sum_{m,k}\widehat r_{m,k}^{\,2}
     +\lambda_{\rm var}\mathcal L_{\rm var}$};
\foreach \id in {a,b,c} {
  \draw[flow] (z\id.south) -- (z\id.south |- train.north);
}
\node[anchor=west,font=\small\bfseries,text=pipegreen] at (-1.6,-3.34)
  {(b) Recover the identifiable parameter target};
\node[recover,font=\footnotesize,text width=5.52cm,minimum height=2.70cm] (raw) at (1.35,-5.08)
  {\textbf{Raw-coordinate route}\\[4pt]
   Semilinear laws, or velocity features\\
   excluding $v^2$, with sufficient coverage\\[4pt]
   $r=\hat z$; allowed maps are affine\\[4pt]
   Sections~\ref{sec:scalar-theory} and~\ref{subsec:feature-separation}};
\node[recover,font=\footnotesize,text width=5.52cm,minimum height=2.70cm] (can) at (7.65,-5.08)
  {\textbf{Canonical route}\\[2pt]
   Polynomial-velocity laws admitting $c_2v^2$;\\
   sufficient coverage on a regular interval\\[2pt]
   $r=\psi_\eta(\hat z)$,\quad $\psi_\eta''+c_{2,\eta}\psi_\eta'=0$\\
   removes the squared-velocity channel\\[2pt]
   Section~\ref{sec:connection-boundary}, Theorem~\ref{thm:canonical-report-orbit}};
\draw[flow] (train.east) -- ++(.24,0) |- (4.5,-3.55) -| (raw.north);
\draw[flow] (4.5,-3.55) -| (can.north);
\node[recover,font=\footnotesize,text width=11.8cm,minimum height=1.25cm] (orbit) at (4.5,-7.40)
  {\textbf{Determine the family-compatible parameter set and its invariants}\\[4pt]
   Apply the allowed coordinate action to the raw or canonical law.\\
   Check whether the target is constant on this set, including its branches.};
\draw[rflow] (raw.south) -- (raw.south |- orbit.north);
\draw[rflow] (can.south) -- (can.south |- orbit.north);
\node[stage,font=\footnotesize,draw=pipeorange!65,fill=pipeorange!5,
  text width=8.3cm,minimum height=1.55cm] (anchor) at (5.95,-9.52)
  {\textbf{Use physical anchors to identify the target}\\[4pt]
   Constrain the remaining freedom and physical units;\\
   check target uniqueness and resolve relevant branches.\\[3pt]
   Section~\ref{sec:regimes}, Proposition~\ref{prop:anchor-rank}};
\draw[rflow] (orbit.south -| anchor.north) --
  node[right,font=\footnotesize,text=pipeorange] {target varies} (anchor.north);
\node[recover,font=\footnotesize,text width=11.8cm,minimum height=1.25cm] (out) at (4.5,-11.64)
  {\textbf{Report the recovered parameter or physical quantity}\\[4pt]
   An invariant target, or an anchor-calibrated estimate such as $L$ or $g$.};
\draw[rflow] (anchor.south) -- (anchor.south |- out.north);
\draw[rflow] (orbit.south -| 0,-8) -- (0,0 |- out.north);
\node[align=center,font=\footnotesize,text=pipegreen,fill=white,inner sep=3pt]
  at (0,-9.52) {Target already\\constant.\\[3pt]No anchor needed};
\end{tikzpicture}
\endgroup
\caption{\textbf{From joint video fitting to parameter recovery.}
The declared family and coverage conditions determine the raw or canonical
route. In the canonical route, $\psi_\eta(u_\star)=0$ and
$\psi_\eta'(u_\star)=1$ fix a normalizer, and
$r''=\psi_\eta'F_\eta+\psi_\eta''(\hat z')^2$ is the transformed law.
Normalization is computed from the fitted law; external anchors are added
only when the desired target is not yet identified. Parameter sets are
understood with the containment and realizability qualifications in
Appendix~\ref{app:operational-details}. Synthetic coordinate and parameter
errors evaluate these predictions using reference states. The diagram uses
$z$ for physical state and $\hat z$ for learned state; experimental plots
use $q$ and $z$, respectively.}
\label{fig:nonlinear-training-pipeline}
\end{figure}
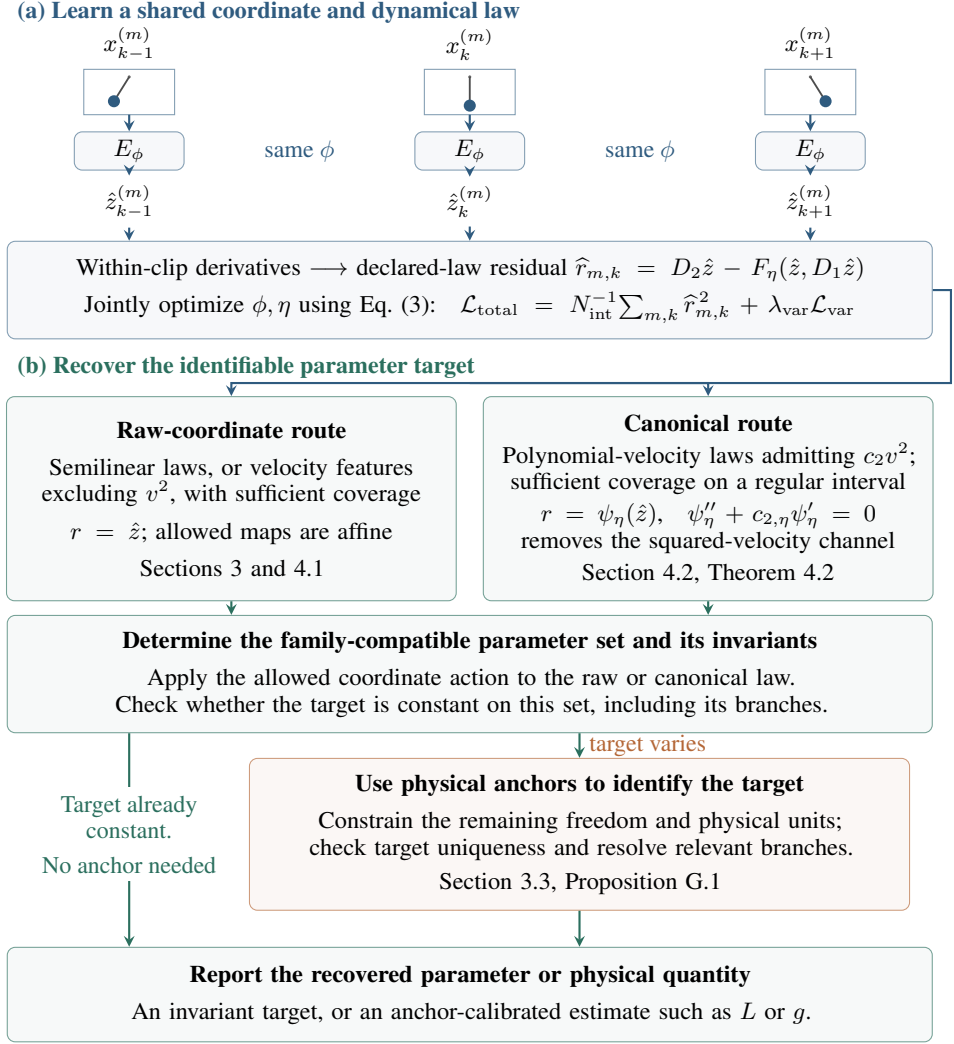

\subsection{Numerical Derivatives and Anti-Collapse Regularization}
\label{app:sampled_definitions}

Clip $m=1,\ldots,M$ supplies frames
$x_k^{(m)}=x^{(m)}(t_{m,k})$, $k=0,\ldots,T_m$, on a uniform time grid
$t_{m,k}=t_{m,0}+k\Delta t_m$ with known $\Delta t_m>0$.
Thus $T_m$ is the final frame index and each clip has $T_m+1$ frames.
For the sampled objective, assume $T_m\ge2$ so every included clip
has an interior three-frame stencil. The sampled encoder outputs are
$\hat z_k^{(m)}=E_\phi(x_k^{(m)})=\hat z^{(m)}(t_{m,k})$.

Following the encoder-only interface of \citet{wang2026physics}, define
for any sequence $q^{(m)}=(q_k^{(m)})_{k=0}^{T_m}$
\begin{equation}
\begin{aligned}
    (D_{1,m}q^{(m)})_k
    &=\frac{q_{k+1}^{(m)}-q_{k-1}^{(m)}}{2\Delta t_m},\\
    (D_{2,m}q^{(m)})_k
    &=\frac{q_{k+1}^{(m)}-2q_k^{(m)}+q_{k-1}^{(m)}}{\Delta t_m^2},
    \qquad k=1,\ldots,T_m-1.
\end{aligned}
\label{eq:AB_centered_stencil}
\end{equation}
No endpoint derivative or cross-clip stencil is used.

At interior indices, the sampled fitted-law residual is
\begin{equation}
    \widehat r_{m,k}(\phi,\eta)
    =
    (D_{2,m}\hat z^{(m)})_k
    -
    F_\eta\!\left(
        \hat z_k^{(m)},
        (D_{1,m}\hat z^{(m)})_k
    \right).
    \label{eq:setup_sampled_residual}
\end{equation}
The objective in Eq.~\eqref{eq:training_objective} can be written as
\begin{equation}
    \mathcal L_{\mathrm{total}}(\phi,\eta)
    =
    \underbrace{
        \frac{1}{N_{\mathrm{int}}}
        \sum_{m=1}^M\sum_{k=1}^{T_m-1}
        \widehat r_{m,k}(\phi,\eta)^2
    }_{\mathcal L_{\mathrm{ODE}}}
    +
    \lambda_{\mathrm{var}}\mathcal L_{\mathrm{var}},
    \qquad
    N_{\mathrm{int}}=\sum_{m=1}^M(T_m-1),
    \label{eq:direct_objective}
\end{equation}
where $\lambda_{\mathrm{var}}>0$ and
$\mathcal L_{\mathrm{var}}$ is a variance-floor penalty that
discourages constant encodings.

For a polynomial-velocity family, the residual in
\eqref{eq:setup_sampled_residual} is
\begin{equation}
    \widehat r_{m,k}(\phi,\eta)
    =(D_{2,m}\hat z^{(m)})_k
    -\sum_{r=0}^p c_{r,\eta}(\hat z_k^{(m)})
      \bigl((D_{1,m}\hat z^{(m)})_k\bigr)^r.
    \label{eq:AB_polynomial_residual}
\end{equation}
In the semilinear case it is
$(D_{2,m}\hat z^{(m)})_k+
 a_\eta(\hat z_k^{(m)})(D_{1,m}\hat z^{(m)})_k+
 b_\eta(\hat z_k^{(m)})$.

To specify the penalty in \eqref{eq:direct_objective}, define
\begin{equation}
    \widehat\mu_m=\frac{1}{T_m+1}\sum_{k=0}^{T_m}\hat z_k^{(m)},
    \qquad
    \widehat{\operatorname{Var}}_m
    =\frac{1}{T_m+1}\sum_{k=0}^{T_m}
      (\hat z_k^{(m)}-\widehat\mu_m)^2,
    \label{eq:AB_sample_moments}
\end{equation}
and set
\begin{equation}
    \mathcal L_{\mathrm{var}}
    =\frac1M\sum_{m=1}^M
      \left[\max\left\{0,
      s_{\mathrm{floor}}-
      \sqrt{\widehat{\operatorname{Var}}_m+\epsilon_{\mathrm{stab}}}
      \right\}\right]^2,
    \label{eq:AB_variance_floor}
\end{equation}
where $\lambda_{\mathrm{var}}>0$, $\epsilon_{\mathrm{stab}}>0$ is numerical,
and $s_{\mathrm{floor}}>\sqrt{\epsilon_{\mathrm{stab}}}$. A constant encoding
$c$ has zero ODE residual whenever $F_\eta(c,0)=0$, equivalently
$b_\eta(c)=0$ in the semilinear family. The penalty discourages such
encodings but does not prove non-collapse on a particular interval,
latent-state consistency, or identifiability. Its scale is a fitting
convention, not a physical anchor.

\subsection{From Sampled Residuals to Continuous Compatibility}
\label{app:residual_accounting}

The explicit constants below concern the semilinear residual. A general
smooth $F_\eta$ admits an analogous transfer with suitable local
Lipschitz bounds. Additional regularity is stated here, rather than
assumed for all exact results.

\paragraph{Approximate continuous state consistency.}
For each clip, let
$e^{(m)}(t)=\hat z^{(m)}(t)-f(z^{(m)}(t))$, and assume this error is $C^2$
on a neighborhood of every selected time. Let
$\mathcal T_{\mathrm{sel}}$ collect the selected interior clip--time pairs,
and define
\begin{equation}
    \varepsilon_{\mathrm{sc}}
    =\sup_{(m,t)\in\mathcal T_{\mathrm{sel}}}
      \max_{r=0,1,2}
      \left|\frac{d^r}{dt^r}e^{(m)}(t)\right|.
    \label{eq:AB_state_consistency_error}
\end{equation}
Suppose $|z'|\le V$, $|f'|\le\Lambda$, and both $f(z)$ and
$f(z)+e^{(m)}$ lie in a compact interval $J\subset D_\eta$ at the selected
points. Assume $\|a_\eta\|_{\infty,J}\le A$ and that $a_\eta,b_\eta$
are respectively $L_a,L_b$-Lipschitz on $J$. Write
$R_\eta[y]=y''+a_\eta(y)y'+b_\eta(y)$. Then
\begin{equation}
    \left|R_\eta[\hat z^{(m)}](t)
           -R_\eta[f\circ z^{(m)}](t)\right|
    \le \varepsilon_{\mathrm{sc}}
      \bigl(1+A+L_a\Lambda V+L_b\bigr).
    \label{eq:AB_state_consistency_transfer}
\end{equation}
The bound follows by expanding the difference and using
$|y'|\le\Lambda V$. Thus a residual tolerance for the smooth
encoder path transfers to the state-consistent reference path by adding
the right-hand side of \eqref{eq:AB_state_consistency_transfer}.

\begin{remark}[Centered-stencil error accounting]
\label{rem:AB_stencil_error}
Fix one clip and an interior sample time $t_k$, suppressing $m$.
Let $\bar z(t)=\hat z(t)$ be the smooth reference encoder path and
$e(t)=\bar z(t)-f(z(t))$. To allow an additional, disjoint sample-level
perturbation, write the sequence supplied to the stencil as
\begin{equation*}
    \hat z_j=\bar z(t_j)+\xi_j,
    \qquad |\xi_j|\le\sigma_0.
\end{equation*}
The sampling model of Section~\ref{sec:setup} is the case $\xi_j=0$.
If acquisition or encoder errors are already included in $e$, they
must not also be counted as a separate $\xi_j$ perturbation.
Suppose $\bar z\in C^4([t_{k-1},t_{k+1}])$, and let
$M_r=\|\bar z^{(r)}\|_\infty$ on this interval. With $D_1,D_2$ denoting
the corresponding within-clip operators,
\begin{equation}
\begin{aligned}
    \delta_0&:=|\hat z_k-\bar z(t_k)|\le\sigma_0,\\
    \delta_1&:=|(D_1\hat z)_k-\bar z'(t_k)|
        \le \frac{M_3}{6}\Delta t^2+\frac{\sigma_0}{\Delta t},\\
    \delta_2&:=|(D_2\hat z)_k-\bar z''(t_k)|
        \le \frac{M_4}{12}\Delta t^2+\frac{4\sigma_0}{\Delta t^2}.
\end{aligned}
\label{eq:AB_stencil_bounds}
\end{equation}
Hence truncation bias is second order, while sample perturbations are
amplified at orders $\Delta t^{-1}$ and $\Delta t^{-2}$.

Define, at $t_k$,
\begin{equation*}
\begin{aligned}
    R_{\mathrm{cont}}&=\bar z''+a_\eta(\bar z)\bar z'+b_\eta(\bar z),\\
    R_{\mathrm{disc},k}&=(D_2\hat z)_k
       +a_\eta(\hat z_k)(D_1\hat z)_k+b_\eta(\hat z_k).
\end{aligned}
\end{equation*}
If $\hat z_k,\bar z(t_k)\in J\subset D_\eta$,
$\|a_\eta\|_{\infty,J}\le A$, $a_\eta,b_\eta$ are respectively
$L_a,L_b$-Lipschitz on $J$, and $|\bar z'(t_k)|\le\bar V$, then
\begin{equation}
    |R_{\mathrm{disc},k}-R_{\mathrm{cont}}|
    \le\delta_2+A\delta_1+(L_a\bar V+L_b)\delta_0.
    \label{eq:AB_discrete_continuous_transfer}
\end{equation}
This follows by adding and subtracting
$a_\eta(\hat z_k)\bar z'(t_k)$. Thus
$|R_{\mathrm{disc},k}|\le\varepsilon_{\mathrm{emp}}$ yields a continuous
residual bound after adding the right-hand side of
\eqref{eq:AB_discrete_continuous_transfer}.
\end{remark}

\paragraph{Combining the error layers.}
Under their respective hypotheses,
\eqref{eq:AB_discrete_continuous_transfer} first transfers a sampled
residual bound to the smooth encoder path, and
\eqref{eq:AB_state_consistency_transfer} then transfers it to $f\circ z$.
The near-level stability results separately account for state mismatch
$h$ and amplify the resulting errors through the physical-velocity design
conditioning. These are deterministic transfers, not a conversion of
encoder-coordinate singular values into physical-unit conditioning.
The stochastic analysis of the LTI-specific centered-difference estimator
in \citet[Theorem~4.8]{wang2026physics} is separate from these nonlinear
orbit and deterministic error calculations.

\section{Interpreting the experimental evidence}
\label{app:operational-details}

The experiments evaluate three linked questions: whether a shared coordinate
map transfers to held-out clips, whether the coefficients follow the
family-specific parameter relations, and whether external anchors recover
physical quantities. The metrics in Appendix~\ref{app:metric-conventions}
measure accuracy for these comparisons. Table~\ref{tab:gauge-card} connects
the comparisons to the assumptions of the theory, so that empirical
agreement can be interpreted at the appropriate level.

\begin{table}[H]
\centering\small
\caption{\textbf{Empirical checks and the conclusions they support.}
The checks probe the stated premises; finite observations do not certify
their continuous-time versions.}
\label{tab:gauge-card}
\setlength{\tabcolsep}{4pt}
\renewcommand{\arraystretch}{1.08}
\begin{tabularx}{\linewidth}{@{}>{\raggedright\arraybackslash}p{.17\linewidth}YY@{}}
\toprule
Question & Evidence to report & Scope of the conclusion \\
\midrule
Declared family &
Coefficient restrictions and domains; test-law error and matched-family
comparisons. &
A small fitting residual alone does not establish that the declared family
contains the physical law. \\
Shared state map &
One frozen map across clips; held-out coordinate error and, when available,
derivative comparisons. &
Coordinate agreement supports transfer of the map. A shared encoder alone
does not establish latent-state consistency. \\
Velocity coverage &
Physical coordinate, level tolerance, distinct speeds and signs, feature
rank and conditioning. &
The coefficient-matching result requires its stated coverage. Failure of a
sufficient rank condition does not itself prove nonidentifiability. \\
Canonical normalization &
Fitted $c_2$, normalizer basepoint, regular interval, and transformed-law
comparison. &
The comparison is restricted to the valid chart. Normalization removes the
squared-velocity channel but does not determine physical units. \\
Parameter recovery &
Allowed coordinate action and branches, parameter-orbit error, and the
specified invariants. &
Agreement of selected coefficients or invariants need not establish full
law equivalence or identify every parameter. \\
Physical calibration &
Anchor values, units, domain, remaining branches, and a uniqueness argument. &
Calibrated estimates use additional physical information. A regular
full-rank anchor map gives local uniqueness, with branches checked separately. \\
\bottomrule
\end{tabularx}
\end{table}

\paragraph{Coordinate comparison, normalization, and calibration.}
An invariant comparison tests a specified parameter functional; an aligned
comparison uses the coordinate and coefficient transformations declared by
the family. A calibrated comparison adds external physical information.
In Figure~\ref{fig:nonlinear-training-pipeline}, normalization is computed
from the learned law before these comparisons: it selects a canonical
coordinate but leaves its affine freedom. For the projected pendulum,
the endpoints of $|y|<1$ are chart singularities, and comparisons remain
within the regular interval. A geometry check for a fixed observable
supports that observable's model; it does not also validate a learned
encoder. These distinctions apply to both synthetic and real videos.

\paragraph{Structural conclusions and finite-data accuracy.}
The general containment results provide compatible parameter outer sets;
constancy on such a set is sufficient for identification. Variation on an
outer set establishes realizable ambiguity only with an additional
construction. For the semilinear case, Proposition~\ref{prop:affine-realizability}
and Corollary~\ref{cor:realizable-invariants} supply this connection under
representability and affine-closure assumptions. The same converse is not
asserted for canonical sets solely from raw--canonical law equivalence.
Empirical agreement does not establish optimizer convergence to all
compatible reports.

\paragraph{Error and uncertainty.}
The stability bounds require physical-velocity conditioning, which
encoder-coordinate singular values cannot replace. Appendix~\ref{app:sampled-interface}
separates sampling and state-consistency errors.
Proposition~\ref{prop:anchor-rank} treats anchors as exact and guarantees
local uniqueness at regular points; it does not quantify anchor measurement
error. Repeated fits with shared observations or anchors describe protocol
variability, not confidence intervals over independent physical measurements.
The experimental appendices report parameter accuracy and initialization
sensitivity separately.

\section{Sharpness of three-slope coverage}
\label{app:coverage-sufficiency}

Three-slope level coverage is the geometric hypothesis used by the semilinear
affine reduction. This appendix shows the requirement is sharp: two slopes admit
non-affine aliases.

\begin{proposition}[Two-slope sharpness for the continuous semilinear class]
\label{prop:two-slope-sharpness}
Let $U\subset R_z$ be a nonempty open interval, let
$a_\theta,b_\theta\in C(U)$, and suppose the physical covered points obey
\[
 z''+a_\theta(z)z'+b_\theta(z)=0.
\]
Assume that at every $u\in U$ the set of covered velocities is exactly
$\{v_1(u),v_2(u)\}$, where $v_1,v_2\in C(U)$ and
$v_1(u)\neq v_2(u)$.  For any $C^2$ diffeomorphism
$f:U\to f(U)$, define continuous target coefficients by
\begin{align}
 a_\eta(f(u))
 &:=a_\theta(u)-\frac{f''(u)}{f'(u)}
       \bigl(v_1(u)+v_2(u)\bigr),
 \label{eq:two-slope-a-alias}\\
 b_\eta(f(u))
 &:=f'(u)b_\theta(u)+f''(u)v_1(u)v_2(u).
 \label{eq:two-slope-b-alias}
\end{align}
Then the transformed coordinate $\hat z=f(z)$ satisfies
$\hat z''+a_\eta(\hat z)\hat z'+b_\eta(\hat z)=0$ at every covered point.
Consequently, two same-level velocities cannot force affine collapse uniformly
over unrestricted continuous semilinear coefficient functions.
\end{proposition}

\begin{proof}
Because $f$ is a diffeomorphism, \eqref{eq:two-slope-a-alias}--
\eqref{eq:two-slope-b-alias} define single-valued continuous functions on the
open interval $f(U)$.  At a covered point with state $u$ and velocity $v$, the
target-law residual is
\begin{align*}
 &\hat z''+a_\eta(\hat z)\hat z'+b_\eta(\hat z)\\
 &\quad=f''(u)v^2
 +f'(u)\bigl(a_\eta(f(u))-a_\theta(u)\bigr)v
 +b_\eta(f(u))-f'(u)b_\theta(u)\\
 &\quad=f''(u)\bigl(v-v_1(u)\bigr)\bigl(v-v_2(u)\bigr).
\end{align*}
It vanishes on both realized branches.  Since $f$ may be chosen nonlinear,
the data do not imply affine collapse in this unrestricted class.
\end{proof}

For a conservative branch pair $v_2=-v_1$, the construction simplifies to
\[
 a_\eta\circ f=a_\theta,
 \qquad
 b_\eta\circ f=f'b_\theta-f''v_1^2.
\]
For example, one orbit of $z''+z=0$ with amplitude $A$ has
$v_1(u)^2=A^2-u^2$ on any $U\Subset(-A,A)$, so every such nonlinear $f$
is compatible with the continuous semilinear law
$a_\eta\circ f=0$ and
$b_\eta\circ f=f'u-f''(A^2-u^2)$. This law-space obstruction proves that two
slopes do not guarantee affine collapse. No claim is made that a particular
encoder architecture realizes every such nonlinear witness.

\section{Canonical coordinates and the projected-pendulum example}
\label{app:connection-details}

Section~\ref{sec:connection-boundary} removes the squared-velocity channel
using a coordinate computed from each law. The explicit integral
\eqref{eq:normalizer-definition}, its positive derivative, and uniqueness up
to an invertible affine map are proved together in
Appendix~\ref{app:connection-proofs}. The transformed coefficients
\eqref{eq:normalized-coefficients} need not be polynomial in the canonical
state. This normalization permits comparison of laws without providing a
physical calibration of the state.

\paragraph{Regular-chart intuition.}
For the angular pendulum, $y=\sin q$ on $|q|<\pi/2$ creates a
squared-velocity term through the chain rule. The inverse chart
$q=\arcsin y$ removes it, while its derivative diverges at $y=\pm1$.
Those endpoints are outside the regular chart. The full laws
\eqref{eq:angular-pendulum}--\eqref{eq:projected-pendulum} and their derivation
appear in Corollary~\ref{cor:projected-pendulum} and its proof.
This illustrates a change of law representation; the narrow
angular-pendulum family does not admit the projected law as another report.

\paragraph{Parameter-recovery routes.}
Table~\ref{tab:result-map} summarizes the assumptions and conclusions of
Sections~\ref{sec:scalar-theory}--\ref{sec:canonical-theory}.
The semilinear case is contained in the curvature-excluding feature route;
the canonical route has the separate polynomial-velocity scope stated in
Section~\ref{sec:connection-boundary}.

\begin{table}[H]
\centering
\caption{Parameter-recovery routes within a fixed declared family, conditional
on a shared non-collapsed state map, exact compatibility, and the stated domain
and coverage assumptions. The semilinear row is a special case of the
finite-feature route; the rows are not an exhaustive partition of scalar ODEs.}
\label{tab:result-map}
\small
\setlength{\tabcolsep}{4pt}
\renewcommand{\arraystretch}{1.10}
\begin{tabularx}{\linewidth}{@{}>{\raggedright\arraybackslash}p{0.32\linewidth}>{\raggedright\arraybackslash}p{0.25\linewidth}Y@{}}
\toprule
\rowcolor{gray!10}
\textbf{Family and required design} & \textbf{Coordinate relation} &
\textbf{Parameter conclusion} \\
\midrule
Semilinear; three distinct physical velocities at each covered state
(Theorem~\ref{thm:semilinear-affine-collapse}) &
Raw coordinates are affine-related &
Affine compatible outer set; invariant combinations, singleton identification,
and sufficient local physical anchors \\
\addlinespace
Finite dilation-stable library with constants, excluding $v^2$;
full augmented rank (Theorem~\ref{thm:feature-separation}) &
Raw coordinates are affine-related &
Affine compatibility analysis, restricted to the declared family \\
\addlinespace
Polynomial in velocity, admitting $c_2(z)v^2$;
$d_v+1$ distinct physical velocities per state
(Theorem~\ref{thm:canonical-report-orbit}) &
Raw relation may be nonlinear; canonical relation is affine &
Canonical compatible outer set in the original parameter space;
invariants and local physical-anchor analysis \\
\addlinespace
Required coverage or shared-map premise fails or remains unverified &
The affected guarantee is unavailable &
Acquire additional evidence, supply a separate family-specific argument,
or withhold the affected claim \\
\bottomrule
\end{tabularx}
\end{table}

\section{Physical anchors and local calibration}
\label{app:additional-consequences}

External anchors constrain coordinate witnesses as well as reported parameters.
A single local criterion applies to the raw and canonical compatibility sets.

\begin{proposition}[Local anchor-rank criterion]
\label{prop:anchor-rank}
Let $\mathcal S$ be a $d_{\mathcal S}$-dimensional $C^1$ manifold near
$s_0$, and let $\mathcal A:\mathcal S\to\R^k$ be a $C^1$ vector of external
scalar-anchor predictions. If
\[
 \rank D\mathcal A_{s_0}=d_{\mathcal S},
\]
then $s_0$ is locally the unique point satisfying
$\mathcal A(s)=\mathcal A(s_0)$.
If $k<d_{\mathcal S}$ and $\mathcal A$ has locally constant maximal rank $k$,
its local level set has dimension $d_{\mathcal S}-k>0$.
Thus fewer than $d_{\mathcal S}$ scalar anchors are insufficient at such
regular points.
\end{proposition}

The proof is in Appendix~\ref{app:proofs-orbit-calculus}.
Full differential rank is sufficient for local uniqueness, not necessary:
singular but locally injective anchor maps are possible. Disconnected branches
and global uniqueness require separate checks.

\paragraph{Raw affine witnesses.}
The family-compatible outer set projects out the affine witness. To retain
the coordinate information on which physical anchors act, define
\begin{equation}
\label{eq:gauge-parameter-certificate}
\mathcal S_U(\theta)
:=\{(\lambda,\tau,\eta)\in\R^*\times\R\times\Theta:
(\lambda,\tau)\text{ is an affine witness for }\eta\text{ on }U\}.
\end{equation}
Its parameter projection is $\OU_U(\theta)$, but the lifted set also retains
coordinate symmetries that leave the mechanism parameters unchanged.
At a regular point where $\mathcal S_U(\theta)$ is locally a $C^1$ manifold,
Proposition~\ref{prop:anchor-rank} applies to any $C^1$ anchor predictions
restricted to that manifold. A paired physical state $u_\star\in U$ predicts
the encoded value $\lambda u_\star+\tau$. Anchors outside $U$ require a
separately justified extension of the witness.

\paragraph{Canonical witnesses.}
For the polynomial-velocity family, the affine witness relates law-dependent
normalized coordinates. Define
\begin{equation}
\label{eq:canonical-lifted-set}
\mathcal S_U^{\mathrm{can}}(\theta)
:=\left\{(\lambda,\tau,\eta)\in\R^*\times\R\times\Theta:
\begin{array}{l}
 A(q)=\lambda q+\tau,\quad A(\psi_\theta(U))\subset J_\eta,\\[-0.1em]
 A\text{ satisfies \eqref{eq:normalized-affine-orbit}}
\end{array}\right\}.
\end{equation}
Its parameter projection is $\mathcal O_U^{\mathrm{can}}(\theta)$.
The associated raw-coordinate witness is
\begin{equation}
\label{eq:canonical-raw-witness-lifted}
 f_{\lambda,\tau,\eta;\theta}(u)
 =\psi_\eta^{-1}\bigl(\lambda\psi_\theta(u)+\tau\bigr).
\end{equation}
Apply Proposition~\ref{prop:anchor-rank} wherever this lifted set is locally a
$C^1$ manifold and the anchor predictions are $C^1$ on it.
Raw position and amplitude anchors must be evaluated through
\eqref{eq:canonical-raw-witness-lifted}; because $\psi_\eta$ depends on the
reported law, an unsigned raw amplitude anchor does not automatically fix
$|\lambda|$. In the semilinear slice, where $c_r=0$ for every $r\ge2$,
choosing identity normalizers identifies the raw and canonical lifted sets.

Both applications concern the full lifted compatibility set. Full rank
therefore gives local uniqueness in every compatible report subset, whereas
the lower bound on the number of anchors need not apply to a smaller subset
realized by a particular encoder architecture. Parameter recovery can require
less information than fixing the entire coordinate witness. For interval-local
laws, the fixed-$U$ parameter projection is an outer set, not automatically
an orbit or an equivalence class.

\section{Proofs for the scalar gauge-orbit theory}
\label{app:proofs-scalar}

We use the admissible-report convention of Section~\ref{sec:setup} and
Definition~\ref{def:state-consistency}: the common map holds along the
continuous paths on $U$, is non-collapsed there, and satisfies
$f(U)\subset D_\eta$. The fitted semilinear law holds at the covering points.
Affine witnesses and the outer set $\OU_U(\theta)$ are as defined in
Section~\ref{sec:scalar-theory}, with domains transported as in
Appendix~\ref{app:domain_conventions}.

\subsection{Proof of Theorem~\ref{thm:semilinear-affine-collapse}}
\begin{proof}
Fix $u\in U$.  At a covered point with physical state $z=u$ and physical
velocity $z'=v$, latent-state consistency holds on a time neighborhood of the
point because $U$ is open.  Differentiating it twice gives
\[
 \hat z'=f'(u)v,
 \qquad
 \hat z''=f''(u)v^2+f'(u)z''.
\]
Substitution of the physical and encoder-side semilinear laws therefore gives
the quadratic residual $R_u(v)$ in
\eqref{eq:semilinear-residual-main}.  
By three-slope level coverage, there are
three pairwise distinct realized physical velocities $v_1,v_2,v_3$ at level
$u$, and compatibility gives $R_u(v_i)=0$ for all three.  A polynomial of
degree at most two with three distinct roots is identically zero.  Hence its
three coefficients vanish:
\[
 f''(u)=0,
 \qquad
 f'(u)\bigl(a_\eta(f(u))-a_\theta(u)\bigr)=0,
 \qquad
 b_\eta(f(u))-f'(u)b_\theta(u)=0.
\]
Since $u$ was arbitrary, $f''=0$ on the interval $U$, so
$f(u)=\lambda u+\tau$ there.  The non-collapse premise gives $\lambda\neq0$;
therefore division of the middle identity by $f'=\lambda$ and substitution
into the last identity yield
\[
 a_\eta(\lambda u+\tau)=a_\theta(u),
 \qquad
 b_\eta(\lambda u+\tau)=\lambda b_\theta(u),
\]
which are exactly the orbit equations \eqref{eq:compat-a-main}.
\end{proof}

\subsection{Proof of Theorem~\ref{thm:orbit-characterization}}
\begin{proof}
Let $(E_\phi,f,\eta)$ be any admissible state-consistent encoder report on $U$.
Theorem~\ref{thm:semilinear-affine-collapse} applies, so
$f(u)=\lambda u+\tau$ on $U$ with $\lambda\neq0$, and the orbit equations
\eqref{eq:compat-a-main} hold. Thus $\eta\in\OU_U(\theta)$.
If $\OU_U(\theta)=\{\theta\}$, membership immediately gives $\eta=\theta$.
\end{proof}

\subsection{Proof of affine-compatible encoder realizability}
\label{app:proof-realizability}

\begin{proof}[Proof of Proposition~\ref{prop:affine-realizability}]
The inclusion $\subseteq$ in \eqref{eq:realizable-compatible-set}
follows from Theorem~\ref{thm:orbit-characterization}.
For the reverse inclusion, fix $\eta\in\mathcal O_U(\theta)$ and an affine
witness $(\lambda,\tau)$. By definition, $\lambda\ne0$,
$\lambda U+\tau\subset D_\eta$, and, for every $u\in U$,
\begin{equation}
 a_\eta(\lambda u+\tau)=a_\theta(u),\qquad
 b_\eta(\lambda u+\tau)=\lambda b_\theta(u).
 \label{eq:realizability-witness}
\end{equation}
Let $E_\eta=\lambda E_0+\tau$, which belongs to $\mathscr E$ by assumption.
For every clip, on the times with $z^{(m)}(t)\in U$,
\[
 \hat z^{(m)}(t):=E_\eta(x^{(m)}(t))
 =\lambda z^{(m)}(t)+\tau.
\]
The affine map extends smoothly to the state domain, is common to all clips,
and is non-collapsed. Its image on $U$ lies in $D_\eta$. Since $U$ is open
and each physical path is continuous, the displayed equality holds on a
time neighborhood of every interior point under consideration. The
encoder-induced path is therefore $C^2$ there, with
$\hat z^{(m)\prime}=\lambda z^{(m)\prime}$ and
$\hat z^{(m)\prime\prime}=\lambda z^{(m)\prime\prime}$. Using
\eqref{eq:realizability-witness},
\begin{align*}
 &\hat z^{(m)\prime\prime}
   +a_\eta(\hat z^{(m)})\hat z^{(m)\prime}
   +b_\eta(\hat z^{(m)})\\
 &\qquad=\lambda\bigl[z^{(m)\prime\prime}
   +a_\theta(z^{(m)})z^{(m)\prime}
   +b_\theta(z^{(m)})\bigr]=0.
\end{align*}
Thus the reported law holds at every such interior point, including the
covering points. All admissibility conditions are satisfied, proving the
reverse inclusion. The construction uses the same videos and recorded
time, without a time rescaling. Three-slope coverage is needed for the
forward containment, not for this reverse construction.
\end{proof}

\begin{proof}[Proof of Corollary~\ref{cor:realizable-invariants}]
The identity witness gives $\theta\in\mathcal O_U(\theta)$. If $\mathcal J$
is constant on this set, \eqref{eq:realizable-compatible-set} implies
$\mathcal J(\eta)=\mathcal J(\theta)$ for every admissible report in
$\mathscr E$. If it is not constant, some $\eta\in\mathcal O_U(\theta)$
satisfies $\mathcal J(\eta)\ne\mathcal J(\theta)$; Proposition~\ref{prop:affine-realizability}
realizes a report with that value. Taking $\mathcal J$ to be the identity
proves the full-parameter claim.
\end{proof}

\paragraph{Representation and domain qualifications.}
Exact physical-coordinate representation is assumed, not inferred from
shared weights, universal approximation, or finite training error. An
unconstrained final affine layer can supply output-affine closure, but does
not establish the existence of $E_0$. Fixed output bounds, fixed-scale
normalizations or imposed physical anchors can restrict the report class;
the proposition must then be applied only after checking its hypotheses.
No compatibility is asserted outside $U$. Even when realizability holds,
transported-domain restrictions may prevent an interval-local compatible
set from being an equivalence class. The result concerns the semilinear
report convention; it does not automatically extend to nonlinear
raw-to-canonical transformations.

\subsection{Proofs for invariants and anchors}
\label{app:proofs-orbit-calculus}

\begin{proof}[Proof of Proposition~\ref{prop:anchor-rank}]
Restrict the $C^1$ map $\mathcal A$ to the
$d_{\mathcal S}$-dimensional manifold $\mathcal S$ near $s_0$. If its
differential has rank $d_{\mathcal S}$ at $s_0$, that maximal rank persists
locally. The constant-rank theorem then makes the local level set
$\mathcal A(s)=\mathcal A(s_0)$ zero-dimensional, so $s_0$ is locally the unique
representative satisfying the anchors. More generally, if the restricted rank
is constant and equal to $r$ near $s_0$, the local level set has dimension
$d_{\mathcal S}-r$. In particular, $k<d_{\mathcal S}$ independent scalar
anchors leave dimension $d_{\mathcal S}-k>0$. This regular-level argument
does not rule out isolated fibers at singular points, and full differential
rank is a sufficient local criterion, not a necessary condition for an
arbitrary nonlinear anchor to isolate a point.
\end{proof}

\subsection{Translation lock and the weight rule}
\label{app:proofs-weight}

\begin{lemma}[Top-gap test for translation lock]
\label{lem:translation-lock}
Let $U$ be a nonempty open interval and let the declared polynomial family have
fixed damping and restoring supports $I_a$ and $I_b$.  Fix $\theta\in\Theta$.
Every affine witness for every $\eta\in\OU_U(\theta)$ has $\tau=0$ if either
of the following holds:
\begin{enumerate}[leftmargin=1.6em,itemsep=0.1em]
\item $n:=\max I_b\ge1$, $B_n(\theta)\neq0$, and $n-1\notin I_b$; or
\item $m:=\max I_a\ge1$, $A_m(\theta)\neq0$, and $m-1\notin I_a$.
\end{enumerate}
Thus either structural top gap is sufficient to make the surviving affine
witness a pure scaling $u\mapsto\lambda u$.
\end{lemma}

\begin{proof}[Proof of Lemma~\ref{lem:translation-lock}]
Both laws expand over the structural support. Write
\[
b_\eta(w)=\sum_{k\in I_b}B_k(\eta)w^k,
\]
and impose $b_\eta(\lambda u+\tau)=\lambda b_\theta(u)$ on $U$; both sides are polynomials, so the identity holds for all $u\in\R$. Since $n=\max I_b$, no term of $b_\eta(\lambda u+\tau)$ has degree exceeding $n$, and only the index $k=n$ reaches degree $n$: matching gives $B_n(\eta)\lambda^n=\lambda B_n(\theta)$, hence $B_n(\eta)=\lambda^{1-n}B_n(\theta)\neq0$, independently of $\tau$. At degree $n-1$, the index $k=n$ contributes $\binom{n}{1}B_n(\eta)\lambda^{n-1}\tau=n\,B_n(\eta)\lambda^{n-1}\tau$ through the binomial expansion; the index $k=n-1$ is absent because $n-1\notin I_b$; and every index $k<n-1$ contributes only to degrees at most $k<n-1$. On the right the coefficient is $\lambda B_{n-1}(\theta)=0$. Hence $n\,B_n(\eta)\lambda^{n-1}\tau=0$, and since $B_n(\eta)\neq0$, $\lambda\neq0$, and $n\ge1$, we conclude $\tau=0$. With $\tau=0$ the witness reduces to $u\mapsto\lambda u$.

The damping-side top-gap lock is identical. If $a_\eta(\lambda u+\tau)=a_\theta(u)$ and the structural damping support $I_a$ has top degree $m\ge1$ with $A_m(\theta)\neq0$ and $m-1\notin I_a$, then matching the top coefficient gives $A_m(\eta)\lambda^m=A_m(\theta)\neq0$; matching degree $m-1$ gives $mA_m(\eta)\lambda^{m-1}\tau=0$. Thus $\tau=0$ by the same argument.
\end{proof}

\begin{proof}[Proof of Proposition~\ref{thm:weight-calculus}]
By translation lock every witness has $\tau=0$, so for any
$\eta\in\OU_U(\theta)$ with witness scale $\lambda$ the orbit equations
\eqref{eq:compat-a-main} read $a_\eta(\lambda u)=a_\theta(u)$ and
$b_\eta(\lambda u)=\lambda b_\theta(u)$. Since $U$ is open, these polynomial
identities can be matched coefficientwise. Substituting the monomial expansions
and matching the coefficient of $u^i$ (respectively $u^j$),
\[
 A_i(\eta)\lambda^i=A_i(\theta),
 \qquad
 B_j(\eta)\lambda^j=\lambda B_j(\theta),
\]
which is \eqref{eq:weight-scaling} with $w(A_i)=i$ and $w(B_j)=j-1$. A rational
monomial $\prod_k C_k^{p_k}$ transforms by
$\lambda^{-\sum_k p_k w(C_k)}$, so it is invariant whenever its total weight
vanishes and its denominator is nonzero.  Finally, at a paired state anchor
$u_\star\in U\setminus\{0\}$ with $\hat z_\star=f(u_\star)$, translation lock
gives $\hat z_\star=\lambda u_\star$.  A signed anchor therefore fixes
$\lambda=\hat z_\star/u_\star$ and hence every coefficient through
\eqref{eq:weight-scaling}; an unsigned anchor fixes only $|\lambda|$ and may
leave the reflection action on odd weights.
\end{proof}

\subsection{Proofs for the stability results}
\label{app:proofs-robust}

\paragraph{Quantitative slope coverage.}
For $\delta_v>0$ and $V<\infty$, the collection has
\emph{$(\delta_v,V)$-slope coverage on $U$} if for every $u\in U$ there are
three realized velocities $v_1,v_2,v_3$ at level $u$ with
$|v_i-v_j|\ge\delta_v$ for $i\neq j$ and $|v_i|\le V$.

\begin{lemma}[Quantitative exact-level affine collapse]
\label{prop:robust-collapse}
Let $K=[u_-,u_+]\Subset U$ be a nondegenerate compact interval, assume latent-state consistency with $f\in C^2$ near
$K$, $(\delta_v,V)$-slope coverage on $K$, $0<c_0\le|f'|$, exact physical
trajectories, and absolute encoder-side residual at most $\varepsilon$ at
every covering point, i.e. $|R_u(v_i)|\le\varepsilon$ for each covering
velocity $v_i$ at level $u$. Then, for every $u\in K$,
\[
|f''(u)|\le\frac{3\varepsilon}{\delta_v^2},\quad
|a_\eta(f(u))-a_\theta(u)|\le\frac{6V\varepsilon}{c_0\delta_v^2},\quad
|b_\eta(f(u))-f'(u)b_\theta(u)|\le\frac{3V^2\varepsilon}{\delta_v^2}.
\]
For every $u_0\in K$, set
$\ell_{u_0}(u)=f(u_0)+f'(u_0)(u-u_0)$. Then
\[
\sup_{u\in K}|f(u)-\ell_{u_0}(u)|
\le\frac{3\varepsilon}{2\delta_v^2}\operatorname{diam}(K)^2.
\]
\end{lemma}

\begin{proof}[Proof of Lemma~\ref{prop:robust-collapse}]
For each $u\in K$, the exact physical law and the chain rule give the
quadratic residual $R_u(v)$ given by the right-hand side of
\eqref{eq:semilinear-residual-main}. At the three covering velocities, $|R_u(v_i)|\le\varepsilon$. Lagrange
interpolation writes $R_u(v)=\sum_{i=1}^3R_u(v_i)L_i(v)$, where
\[
 L_i(v)=\frac{(v-v_j)(v-v_k)}{(v_i-v_j)(v_i-v_k)},
 \qquad \{i,j,k\}=\{1,2,3\}.
\]
The denominator has absolute value at least $\delta_v^2$, so the quadratic,
linear, and constant coefficients of $L_i$ have absolute values at most
$1/\delta_v^2$, $2V/\delta_v^2$, and $V^2/\delta_v^2$, respectively.
Summing the three residual contributions bounds $|f''|$,
$|f'(a_\eta\circ f-a_\theta)|$, and
$|b_\eta\circ f-f'b_\theta|$ by
$3\varepsilon/\delta_v^2$, $6V\varepsilon/\delta_v^2$, and
$3V^2\varepsilon/\delta_v^2$. Divide the middle bound by $|f'|\ge c_0$.
Finally, Taylor's theorem gives
$|f(u)-\ell_{u_0}(u)|\le\tfrac12\sup_K|f''|\,|u-u_0|^2$, proving the
stated uniform bound for every $u_0\in K$.
\end{proof}

\begin{theorem}[Near-level, design-conditioned affine stability]
\label{thm:near-level-stability}
Let $K=[u_-,u_+]\Subset U$ be a nondegenerate compact interval, where
$K\Subset U$ means that $K$ is compactly contained in $U$. Assume there is a map
$f\in C^2$ on an open neighborhood of $K$ such that
$\hat z^{(m)}=f\circ z^{(m)}$ on a time neighborhood of every selected realized
point used below, and assume that the physical trajectories satisfy
\eqref{eq:true-semilinear-main} exactly at those points.
For $v\in\R$, write
$\boldsymbol\phi_{\rm sl}(v)=(v^2,v,1)^\top$. The chain-rule residual at level
$u$ and velocity $v$ has coefficient form
$\boldsymbol\phi_{\rm sl}(v)^\top c(u)$, where
\begin{equation}
\label{eq:coefficient-residual-vector}
c(u):=\begin{pmatrix}
f''(u)\\
f'(u)\bigl(a_\eta(f(u))-a_\theta(u)\bigr)\\
b_\eta(f(u))-f'(u)b_\theta(u)
\end{pmatrix}.
\end{equation}
Assume $0<c_0\le|f'|$ on $K$, and suppose $c$ is $L_c$-Lipschitz on $K$. This is a direct
regularity assumption on the composite residual coefficients; it follows, for
example, from $f\in C^{2,1}$ near $K$ and local Lipschitz regularity of
$a_\theta,b_\theta,a_\eta,b_\eta$ on the relevant ranges. Fix
$u_0\in K$ and define
$\ell_{u_0}(u):=f(u_0)+f'(u_0)(u-u_0)$. For every target level $u\in K$, suppose there are
three realized points with physical levels $u_i\in K$ and velocities $v_i$
such that
\[
|u_i-u|\le h,
\qquad |v_i|\le V,
\qquad
\Phi_u:=\begin{pmatrix}\boldsymbol\phi_{\rm sl}(v_1)^\top\\
\boldsymbol\phi_{\rm sl}(v_2)^\top\\
\boldsymbol\phi_{\rm sl}(v_3)^\top\end{pmatrix},
\qquad \sigma_{\min}(\Phi_u)\ge\sigma_\star>0.
\]
If the absolute encoder-side residual at each of these three realized points
is at most $\varepsilon$, then
\begin{equation}
\label{eq:near-level-bound}
\|c(u)\|_2\le B,
\qquad
B:=\frac{\sqrt3}{\sigma_\star}
\left(\varepsilon+L_c h\sqrt{V^4+V^2+1}\right).
\end{equation}
Consequently the same right-hand side bounds $|f''(u)|$ and the restoring
compatibility error, and after division by $c_0$ it bounds the damping
compatibility error. Moreover,
\begin{equation}
\label{eq:near-affine-bound}
\sup_{u\in K}|f(u)-\ell_{u_0}(u)|
\le \tfrac12 B\,\operatorname{diam}(K)^2.
\end{equation}
\end{theorem}

\paragraph{Gauge dependence of design conditioning.}
The matrix $\Phi_u$ above uses physical velocities. Under an exact affine map,
$\hat v_i=\lambda v_i$ and the analogous encoder-velocity matrix satisfies
\begin{equation}
\label{eq:design-gauge-transform}
 \widehat\Phi_u
 =\Phi_u\operatorname{diag}(\lambda^2,\lambda,1).
\end{equation}
Thus rank is gauge invariant but the raw smallest singular value is not.
Column normalization yields a scale-free rank/shape diagnostic, whereas the raw
bound \eqref{eq:near-level-bound} remains a physical-unit statement and requires
a calibrated or bounded scale.

\begin{proof}[Proof of Theorem~\ref{thm:near-level-stability}]
At a realized point with physical level $u_i$ and velocity $v_i$, the exact
physical equation and the chain rule give the encoder-side residual
\[
R_{u_i}(v_i)=\boldsymbol\phi_{\rm sl}(v_i)^\top c(u_i).
\]
For a fixed target $u$, add and subtract $c(u)$ to obtain
\begin{align*}
|\boldsymbol\phi_{\rm sl}(v_i)^\top c(u)|
&\le |R_{u_i}(v_i)|
  +\|\boldsymbol\phi_{\rm sl}(v_i)\|_2\,\|c(u_i)-c(u)\|_2\\
&\le \varepsilon+\sqrt{V^4+V^2+1}\,L_c h.
\end{align*}
Stacking the three inequalities and using the Euclidean norm gives
\[
\|\Phi_u c(u)\|_2
\le \sqrt3\left(\varepsilon+L_ch\sqrt{V^4+V^2+1}\right).
\]
Because $\sigma_{\min}(\Phi_u)\ge\sigma_\star$, inversion on the three-dimensional
coefficient space yields \eqref{eq:near-level-bound}. The first and third
components of $c(u)$ are bounded directly; dividing the second component by
$|f'(u)|\ge c_0$ bounds the damping compatibility error. Taking the supremum
of the first-component bound over $K$ and applying Taylor's theorem around any
$u_0\in K$ proves the final near-affine statement.
\end{proof}

\paragraph{Column-normalized diagnostic.}
If each nonzero column is normalized to unit Euclidean norm, then
\[
 \overline{\widehat\Phi}_u
 =\overline\Phi_u\operatorname{diag}(1,\operatorname{sign}\lambda,1),
\]
so the column-normalized singular values are invariant. This normalized matrix
is a scale-free rank/shape diagnostic; it cannot replace the raw physical-unit
conditioning in Theorem~\ref{thm:near-level-stability}, because the residual
coefficients and their error bounds also change units under the gauge.

\paragraph{Approximate state consistency.}
When the encoder path only approximates $f\circ z$, use $f\circ z$ as the
state-consistent reference path. Under the regularity, domain, and error
hypotheses of Appendix~\ref{app:sampled-interface},
\eqref{eq:AB_state_consistency_transfer} gives the additional residual budget
needed for this replacement. The bounds above then apply to the reference
path with that inflated budget, provided their remaining hypotheses hold.

\section{Proofs for feature separation and the canonical connection gauge}
\label{app:connection-proofs}
\label{app:finite-basis}

This appendix supplies the intermediate coefficient-matching and law-equivalence
statements and proves the results used in Section~\ref{sec:canonical-theory}
and Appendix~\ref{app:connection-details}.
All coefficients use the right-hand-side convention: in particular, a
polynomial-velocity law is written as $z''=\sum_r c_r(z)(z')^r$.
All identities are restricted to the open intervals on which the displayed
compositions are defined.

\subsection{Feature separation}

\begin{proof}[Proof of Theorem~\ref{thm:feature-separation}]
Fix $u\in U$.  At any realized point with $z=u$ and $z'=v$, latent-state
consistency holds on a time neighborhood of that point because $U$ is open and
the trajectory is continuous.  It may therefore be differentiated twice at
the point:
\[
 \hat z'=f'(u)v,
 \qquad
 \hat z''=f''(u)v^2+f'(u)z''.
\]
Substituting the physical and encoder-induced feature laws gives
\begin{equation}
\label{eq:feature-residual-proof}
 R_u(v):=f''(u)v^2+f'(u)F_\theta(u,v)
 -F_\eta(f(u),f'(u)v)=0
\end{equation}
at every realized velocity at level $u$.

The function $v\mapsto f'(u)F_\theta(u,v)$ lies in $\mathsf V$ because
$F_\theta(u,\cdot)\in\mathsf V$.  The function
$v\mapsto F_\eta(f(u),f'(u)v)$ also lies in $\mathsf V$: this is dilation
stability with dilation factor $f'(u)$, including the case $f'(u)=0$.
Consequently $R_u\in\mathsf V+\operatorname{span}\{v^2\}$.  Since
$v^2\notin\mathsf V$, this sum is direct. Use the basis
$\varphi_1,\ldots,\varphi_K$ specified in the theorem. In the ordered basis
$(\varphi_1,\ldots,\varphi_K,v^2)$ of the direct sum, the evaluation matrix is
exactly $M_u$ in \eqref{eq:feature-rank-matrix}, which has column rank
$K+1$.  Since \eqref{eq:feature-residual-proof} vanishes at all its rows, the
coefficient vector of $R_u$ is zero, so $R_u\equiv0$ as a function of $v$.
In particular, the coefficient of the last basis element $v^2$ is
$f''(u)=0$.

The state $u$ was arbitrary, hence $f''=0$ on the interval $U$ and
$f(u)=\lambda u+\tau$ there.  Since $f'\not\equiv0$, its constant derivative
$\lambda$ is nonzero.  Returning to $R_u\equiv0$ and substituting this affine
form gives
$F_\eta(\lambda u+\tau,\lambda v)=\lambda F_\theta(u,v)$ for all $u\in U$
and all $v\in\R$, which is \eqref{eq:feature-orbit-equation}.
\end{proof}

\paragraph{Odd drag: rank and affine scaling.}
Consider $\mathsf V=\operatorname{span}\{1,v,v|v|\}$ and four distinct
velocities including both signs. If
$P(v)=a+bv+cv|v|+dv^2$ vanishes at all four, its restrictions to the two
half-lines are quadratics. If three sample velocities lie in the same closed
half-line, that quadratic vanishes identically; the remaining, strictly
opposite-sign velocity then forces the other quadratic to vanish as well.
This also covers a sample at zero. Otherwise the samples are $p,q,-r,-s$
with $p,q,r,s>0$, and
\[
 P(v)=k_+(v-p)(v-q)\quad(v\ge0),\qquad
 P(v)=k_-(v+r)(v+s)\quad(v\le0).
\]
Equality of their constant and linear coefficients gives
$k_+pq=k_-rs$ and $-k_+(p+q)=k_-(r+s)$, which force $k_+=k_-=0$.
Thus the four-column design has full rank, including when zero is sampled.
Since $(\lambda v)|\lambda v|=\lambda|\lambda|v|v|$,
\eqref{eq:feature-orbit-equation} yields
\[
 c_{0,\eta}(\lambda u+\tau)=\lambda c_{0,\theta}(u),\qquad
 c_{1,\eta}(\lambda u+\tau)=c_{1,\theta}(u),\qquad
 c_{d,\eta}(\lambda u+\tau)=\frac{c_{d,\theta}(u)}{|\lambda|}.
\]

\paragraph{One-sided velocities and the declared family.}
For $F(u,v)=c_0(u)+c_1(u)v+c_d(u)v|v|$ with continuous coefficients on $U$,
let $f:U\to f(U)$ be any $C^2$ diffeomorphism. Define coefficients on $f(U)$ by
\[
 \widetilde c_0\circ f=f'c_0,\qquad
 \widetilde c_1\circ f=c_1,\qquad
 \widetilde c_d\circ f=\frac{f''+f'c_d}{f'|f'|}.
\]
For the corresponding $\widetilde F$, direct substitution gives
\[
 f''(u)v^2+f'(u)F(u,v)
 -\widetilde F(f(u),f'(u)v)
 =f''(u)\bigl(v^2-v|v|\bigr).
\]
Hence the transformed law is exact for all physical $v\ge0$, including
orientation-reversing $f$, but generally fails for $v<0$ wherever $f''\ne0$.
For a physical $v\le0$ branch instead, replace the numerator by $f'c_d-f''$.
The unrestricted class of continuous coefficient functions on the transported
interval is closed under these constructions. Membership in a fixed
parametric family, coefficient ties, and any sign restrictions must be checked
separately; one-sided rank failure alone does not prove
non-identifiability in every such family.

In particular, for the coverage experiment of
Appendix~\ref{app:exp-synth-coverage}, the algebraic control
$f(u)=u+5u^3$ certifies this unrestricted one-sided construction, not an
exact alternative within its state-polynomial $5/3/5$ fitted family. Indeed its physical $c_0(u)=-0.8u$ requires
$\widetilde c_0(u+5u^3)=-0.8u-12u^3$. If $\widetilde c_0$ were a polynomial,
degree comparison would make it linear, while its linear term forces
$\widetilde c_0(y)=-0.8y$, whose cubic coefficient after composition is
$-4$, not $-12$. Thus this twin is outside that polynomial family.

\subsection{Polynomial coefficient matching}

For the fixed family \eqref{eq:poly-velocity-law}, set $d_v=\max\{2,p\}$
and zero-pad missing channels, including $c_2\equiv0$ when $p<2$.
The following statement does not require non-collapse.

\begin{proposition}[Passive coefficient matching]
\label{prop:poly-coefficient-matching}
Let $U\subset R_z$ be a nonempty open interval. Assume
latent-state consistency through $f\in C^2(U)$ with
$f(U)\subset D_\eta$, and at least $d_v+1$ pairwise-distinct physical
velocities realized at interior trajectory points at every $u\in U$.
If the physical and encoder-induced paths satisfy the declared laws
at those points, then, for every $u\in U$,
\begin{align}
 c_{r,\eta}(f(u))f'(u)^r
 &=f'(u)c_{r,\theta}(u), && 0\le r\le p,\ r\neq2,
 \label{eq:raw-compat-r}\\
 c_{2,\eta}(f(u))f'(u)^2
 &=f''(u)+f'(u)c_{2,\theta}(u). && \label{eq:raw-compat-2}
\end{align}
When both declared laws exclude the quadratic channel,
\eqref{eq:raw-compat-2} reduces to $f''=0$. With non-collapse this recovers
the primary invertible affine workflow. When the channel is present, the same
identity is the one-dimensional connection transformation law.
\end{proposition}

\begin{proof}[Proof of Proposition~\ref{prop:poly-coefficient-matching}]
Fix $u\in U$.  As in the preceding proof, differentiating the local
factorization $\hat z=f(z)$ twice and substituting the two
polynomial-velocity laws shows that the residual at physical velocity $v$ is
exactly
\begin{equation}
\label{eq:general-residual-polynomial}
 R_u(v)=f''(u)v^2+f'(u)\sum_{r=0}^{p}c_{r,\theta}(u)v^r
 -\sum_{r=0}^{p}c_{r,\eta}(f(u))(f'(u)v)^r.
\end{equation}
It belongs to $\operatorname{span}\{1,v,\ldots,v^{d_v}\}$ with
$d_v=\max\{2,p\}$ and vanishes at every realized velocity used by the
proposition hypothesis.  Its monomial feature matrix contains a square
Vandermonde submatrix and therefore has full column rank, so the coefficient
vector of $R_u$ is zero and hence
$R_u\equiv0$ as a polynomial.

For $r\neq2$, the coefficient of $v^r$ is
\[
 f'(u)c_{r,\theta}(u)-c_{r,\eta}(f(u))f'(u)^r,
\]
so its vanishing gives \eqref{eq:raw-compat-r}.  The coefficient of $v^2$ is
\[
 f''(u)+f'(u)c_{2,\theta}(u)
 -c_{2,\eta}(f(u))f'(u)^2,
\]
and its vanishing gives \eqref{eq:raw-compat-2}.  Whenever both quadratic
coefficients of the declared laws are identically zero, including every
$p\le1$ family, the last identity is $f''(u)=0$.  Since $u$ was arbitrary, all
identities hold on $U$.
\end{proof}

\subsection{Construction and uniqueness of the normalizer}

For each law choose $u_{\xi,\star}\in D_\xi$ and define
\begin{equation}
\label{eq:normalizer-definition}
 \psi_\xi(u)
 :=\int_{u_{\xi,\star}}^u
 \exp\!\left(-\int_{u_{\xi,\star}}^s c_{2,\xi}(r)\,dr\right)ds.
\end{equation}
We verify the claims in Section~\ref{sec:connection-boundary}, including the
basepoint values $\psi_\xi(u_{\xi,\star})=0$ and
$\psi_\xi'(u_{\xi,\star})=1$. Suppress the law index, let $c_2$ be continuous
on the open interval $D$, and write
\[
 h(u):=\exp\!\left(-\int_{u_\star}^{u}c_2(r)\,dr\right).
\]
The fundamental theorem of calculus gives $h\in C^1(D)$,
$h(u)>0$, $h(u_\star)=1$, and $h'(u)=-c_2(u)h(u)$.  Defining
$\psi(u)=\int_{u_\star}^{u}h(s)ds$ therefore gives
$\psi\in C^2(D)$,
\[
 \psi'=h>0,
 \qquad
 \psi''=-c_2\psi',
\]
which proves \eqref{eq:normalizer-ode}.  Strict positivity of $\psi'$ makes
$\psi$ strictly increasing.  The inverse-function theorem gives a local
$C^2$ inverse at every point, strict monotonicity makes these local inverses
agree globally, and the image $J=\psi(D)$ is open.  Hence
$\psi:D\to J$ is a $C^2$ diffeomorphism.

For a trajectory satisfying \eqref{eq:poly-velocity-law}, set
$q=\psi(z)$.  Then
\begin{align*}
 q''
 &=\psi''(z)(z')^2+\psi'(z)z''\\
 &=\bigl(\psi''(z)+c_2(z)\psi'(z)\bigr)(z')^2
   +\sum_{r\neq2}\psi'(z)c_r(z)(z')^r\\
 &=\sum_{r\neq2}\psi'(z)^{1-r}c_r(z)(q')^r,
\end{align*}
where $q'=\psi'(z)z'$ was used in the last line, which identifies the
transformed coefficients. Restoring the law index, the
canonical equation is
\begin{equation}
\label{eq:normalized-law}
 q''=\sum_{\substack{0\le r\le p\\r\ne2}}
 \bar c_{r,\xi}(q)(q')^r,
 \qquad q\in J_\xi.
\end{equation}
Its coefficients are, explicitly,
\begin{equation}
\label{eq:normalized-coefficients}
 \bar c_{r,\xi}(\psi_\xi(u))
 :=\psi_\xi'(u)^{1-r}c_{r,\xi}(u)\quad(0\le r\le p,\ r\ne2),
 \qquad \bar c_{2,\xi}:=0;
\end{equation}
$\bar c_{r,\xi}(q)$ multiplies $(q')^r$ and need not be polynomial in $q$.

Finally, let $\phi$ be any other $C^2$ solution of
$\phi''+c_2\phi'=0$ with $\phi'$ nowhere zero.  Because both derivatives solve
the first-order equation $y'=-c_2y$,
\[
 \left(\frac{\phi'}{\psi'}\right)'
 =\frac{\phi''\psi'-\phi'\psi''}{(\psi')^2}=0.
\]
Thus $\phi'=\lambda\psi'$ for a constant $\lambda\neq0$, and integration
gives $\phi=\lambda\psi+\tau$.  This proves uniqueness up to affine
post-composition, including normalizers defined using different basepoints.

\subsection{Canonical affine compatibility}

Two canonical laws are affine-compatible on $\psi_\theta(U)$ if some
$A(q)=\lambda q+\tau$, $\lambda\ne0$, maps $\psi_\theta(U)$ into $J_\eta$
and satisfies, for every $0\le r\le p$, $r\ne2$,
\begin{equation}
\label{eq:normalized-affine-orbit}
 \bar c_{r,\eta}(A(q))\lambda^r
 =\lambda\bar c_{r,\theta}(q),
 \qquad q\in\psi_\theta(U).
\end{equation}

\begin{theorem}[Canonical connection boundary]
\label{thm:connection-slice}
Let the $\theta$- and $\eta$-laws be two scalar polynomial-velocity laws of the
same declared degree $p$ on open intervals, with missing channels zero-padded,
continuous coefficient functions, and the normalizers of
\eqref{eq:normalizer-ode} under the stated basepoint convention.
Let $U\subset D_\theta$ be a nonempty open interval.
The following are equivalent.
\begin{enumerate}[leftmargin=1.6em,itemsep=0.2em]
\item There is a non-collapsed $f\in C^2(U)$ with $f(U)\subset D_\eta$
satisfying all raw compatibility identities
\eqref{eq:raw-compat-r}--\eqref{eq:raw-compat-2}.
\item There is an affine map $A(q)=\lambda q+\tau$, $\lambda\neq0$, with
$A(\psi_\theta(U))\subset J_\eta$ satisfying the canonical coefficient identities
\eqref{eq:normalized-affine-orbit} on $\psi_\theta(U)$.
\end{enumerate}
The witnesses correspond bijectively through
\begin{equation}
\label{eq:raw-canonical-witness}
 A=\psi_\eta\circ f\circ\psi_\theta^{-1},
 \qquad
 f=\psi_\eta^{-1}\circ A\circ\psi_\theta.
\end{equation}
The choice of basepoint, origin, scale, or orientation for either normalizer
changes $A$ only by invertible affine pre- or post-composition.
\end{theorem}

\begin{proof}[Proof of Theorem~\ref{thm:connection-slice}]
Assume statement~1 and define
$g=\psi_\eta\circ f\circ\psi_\theta^{-1}$ on
$\psi_\theta(U)$.  Put $x=\psi_\theta(u)$.  A first differentiation gives
\begin{equation}
\label{eq:g-prime-proof}
 g'(x)=\frac{\psi_\eta'(f(u))f'(u)}{\psi_\theta'(u)}.
\end{equation}
Since $d/dx=(1/\psi_\theta'(u))d/du$, differentiating once more yields
\begin{align}
g''(x)
&=\frac{1}{\psi_\theta'(u)}
  \frac{d}{du}\left(
  \frac{\psi_\eta'(f(u))f'(u)}{\psi_\theta'(u)}
  \right)\nonumber\\
&=\frac{\psi_\eta'(f(u))}{\psi_\theta'(u)^2}
 \left[f''(u)+f'(u)c_{2,\theta}(u)
 -c_{2,\eta}(f(u))f'(u)^2\right].
\label{eq:g-second-proof}
\end{align}
To obtain the last line, we used
$\psi_\theta''=-c_{2,\theta}\psi_\theta'$ and
$\psi_\eta''=-c_{2,\eta}\psi_\eta'$.  The bracket vanishes by
\eqref{eq:raw-compat-2}, so $g''=0$.  Therefore
$g(x)=\lambda x+\tau$ on the interval $\psi_\theta(U)$.  The factors
$\psi_\theta'$ and $\psi_\eta'$ in \eqref{eq:g-prime-proof} are strictly
positive.  Since $f'\not\equiv0$, $g'\not\equiv0$; because $g'=\lambda$ is
constant, $\lambda\neq0$.  Equation~\eqref{eq:g-prime-proof} then also gives
\[
 f'(u)=\lambda
 \frac{\psi_\theta'(u)}{\psi_\eta'(f(u))}\neq0
 \quad\text{for every }u\in U,
\]
which proves the automatic upgrade from the non-collapse premise to a local
diffeomorphism.

It remains to transform the nonquadratic coefficients.  For $r\neq2$,
using \eqref{eq:normalized-coefficients},
\eqref{eq:g-prime-proof}, and \eqref{eq:raw-compat-r}, we obtain
\begin{align*}
 \bar c_{r,\eta}(g(x))g'(x)^r
 &=\psi_\eta'(f(u))^{1-r}c_{r,\eta}(f(u))
   \left(\frac{\psi_\eta'(f(u))f'(u)}
   {\psi_\theta'(u)}\right)^r\\
 &=\frac{\psi_\eta'(f(u))}{\psi_\theta'(u)^r}
   c_{r,\eta}(f(u))f'(u)^r\\
 &=\frac{\psi_\eta'(f(u))f'(u)}{\psi_\theta'(u)^r}
   c_{r,\theta}(u)\\
 &=g'(x)\psi_\theta'(u)^{1-r}c_{r,\theta}(u)
 =g'(x)\bar c_{r,\theta}(x).
\end{align*}
Substituting $g=A$ and $g'=\lambda$ gives
\eqref{eq:normalized-affine-orbit}, proving statement~2.

Conversely, assume statement~2 and define
$f=\psi_\eta^{-1}\circ A\circ\psi_\theta$.  Each factor is a $C^2$
diffeomorphism on the relevant interval and $A'$ is the nonzero constant
$\lambda$, so $f$ is a $C^2$ diffeomorphism from $U$ onto its image.  Reversing
the coefficient calculation in the preceding paragraph gives
\eqref{eq:raw-compat-r}.  Moreover
$\psi_\eta\circ f\circ\psi_\theta^{-1}=A$ has second derivative zero.
Applying the identity \eqref{eq:g-second-proof} and using positivity of its
prefactor gives \eqref{eq:raw-compat-2}.  Thus statement~1 holds.

The two constructions in \eqref{eq:raw-canonical-witness} are inverse because
$\psi_\theta$ and $\psi_\eta$ are bijective.  If the normalizers are replaced
by $\widetilde\psi_\theta=B_\theta\circ\psi_\theta$ and
$\widetilde\psi_\eta=B_\eta\circ\psi_\eta$, where $B_\theta,B_\eta$ are
invertible affine maps by the uniqueness result above, then the canonical
witness becomes $B_\eta\circ A\circ B_\theta^{-1}$, again invertible affine.
Existence of a canonical affine witness, and hence the canonical
family-compatible outer set, is unchanged.
\end{proof}

\begin{corollary}[Constant linear damping is state-gauge invariant]
\label{cor:linear-damping-invariant}
Under Theorem~\ref{thm:connection-slice}, interpret any absent linear channel
as $c_{1,\xi}\equiv0$. Then
\begin{equation}
\label{eq:linear-damping-pullback}
 c_{1,\eta}(f(u))=c_{1,\theta}(u),\qquad u\in U.
\end{equation}
In particular, if both laws use
$c_{1,\xi}\equiv-\delta_\xi$, then $\delta_\eta=\delta_\theta$, even when the
raw state map is nonlinear and $c_2$ is nonzero.
\end{corollary}

\begin{proof}[Proof of Corollary~\ref{cor:linear-damping-invariant}]
If $p=0$, both zero-padded linear coefficients vanish, so
\eqref{eq:linear-damping-pullback} and the constant-damping conclusion are
immediate. If $p\ge1$, taking $r=1$ in \eqref{eq:raw-compat-r} gives
$c_{1,\eta}(f(u))f'(u)=f'(u)c_{1,\theta}(u)$.  Theorem~\ref{thm:connection-slice}
upgrades the non-collapse premise to $f'(u)\neq0$ at every $u\in U$, so division by
$f'(u)$ proves \eqref{eq:linear-damping-pullback}.  Substituting
$c_{1,\xi}\equiv-\delta_\xi$ gives the constant-damping claim.
\end{proof}

\begin{proof}[Proof of Theorem~\ref{thm:canonical-report-orbit}]
Let $(E_\phi,f,\eta)$ be an admissible report as in
Section~\ref{sec:setup}, with the physical-velocity coverage in
Theorem~\ref{thm:canonical-report-orbit}. The common map satisfies
Definition~\ref{def:state-consistency} along the continuous trajectories
on $U$, so the chain rule applies at the covering points.
Proposition~\ref{prop:poly-coefficient-matching} gives
all raw compatibility identities.  Theorem~\ref{thm:connection-slice} then
gives an affine $A$ satisfying \eqref{eq:normalized-affine-orbit}; by the
definition \eqref{eq:canonical-parameter-orbit},
$\eta\in\mathcal O_U^{\mathrm{can}}(\theta)$.

At a covered frame, state consistency gives
$E_\phi(x^{(m)}(t))=f(z^{(m)}(t))$.  Post-composing by $\psi_\eta$ and using
$\psi_\eta\circ f=A\circ\psi_\theta$ from
\eqref{eq:raw-canonical-witness} yields
\[
 \psi_\eta\!\left(\hat z^{(m)}(t)\right)
 =\psi_\eta(f(z^{(m)}(t)))
 =A(\psi_\theta(z^{(m)}(t))),
\]
which is \eqref{eq:canonical-encoder-affine}.  The same $f$, and therefore the
same $A$, is common to all clips by latent-state consistency.
\end{proof}

\subsection{Projected-pendulum example}

Consider the angular law
\begin{equation}
\label{eq:angular-pendulum}
 q''+\delta q'+\omega^2\sin q=0.
\end{equation}

\begin{corollary}[Projected pendulum: apparent inertia and exact normalizer]
\label{cor:projected-pendulum}
Let $q$ satisfy \eqref{eq:angular-pendulum} and set $y=\sin q$ on a branch
contained in $(-\pi/2,\pi/2)$. Then $y$ satisfies
\begin{equation}
\label{eq:projected-pendulum}
 y''+\delta y'+\omega^2y\sqrt{1-y^2}
 +\frac{y}{1-y^2}(y')^2=0.
\end{equation}
In the right-hand-side convention, $c_2(y)=-y/(1-y^2)$ and the normalizer is
$\psi(y)=\arcsin y$ up to an affine transformation.  The endpoints $y=\pm1$ lie outside the regular chart.
\end{corollary}

\begin{proof}[Proof of Corollary~\ref{cor:projected-pendulum}]
On a branch contained in $(-\pi/2,\pi/2)$, $y=\sin q$ has
$\cos q=\sqrt{1-y^2}>0$.  Differentiation gives
\[
 y'=\cos q\,q',
 \qquad
 y''=\cos q\,q''-\sin q\,(q')^2.
\]
Substitute
$q''=-\delta q'-\omega^2\sin q$,
$q'=y'/\sqrt{1-y^2}$, and $\sin q=y$:
\[
 y''=-\delta y'-\omega^2y\sqrt{1-y^2}
 -\frac{y}{1-y^2}(y')^2.
\]
Moving all terms to the left proves \eqref{eq:projected-pendulum}.  In the
right-hand-side convention, the quadratic coefficient is
$c_2(y)=-y/(1-y^2)$.  Therefore
\[
 \psi'(y)=
 \exp\!\left(-\int c_2(y)dy\right)
 =\frac{1}{\sqrt{1-y^2}}
\]
up to a nonzero multiplicative constant, and integration gives
$\psi(y)=\arcsin y$ up to an affine post-composition.  At $y=\pm1$ this
derivative diverges and the forward map has derivative
$d(\sin q)/dq=0$, so neither the coefficient-domain nor local
diffeomorphism hypotheses hold there.
\end{proof}

\section{Proofs for the family-specific conclusions}
\label{app:regime-proofs}

\subsection{Proofs of Table~\ref{tab:regimes}}
\label{app:representative-reductions}

\begin{proof}
Under the coverage and family assumptions of Table~\ref{tab:regimes},
Theorem~\ref{thm:semilinear-affine-collapse} supplies an affine witness
$f(u)=\lambda u+\tau$ on $U$, with $\lambda\neq0$, satisfying
\eqref{eq:compat-a-main}. Tildes denote reported parameters, with the same
family conventions as the physical parameters.

\noindent\textbf{LTI.}
The orbit equations give $\tilde\delta=\delta$ and
$\tilde\alpha(\lambda u+\tau)=\lambda\alpha u$ on $U$.
Matching coefficients yields $\tilde\alpha=\alpha$ and $\alpha\tau=0$.
Conversely, every $\lambda\neq0$ and $\tau$ satisfying that restriction is a
witness, including arbitrary translations when $\alpha=0$.

\medskip\noindent\textbf{Damped pendulum.}
With $\omega,\tilde\omega>0$, the orbit equations give
\begin{equation}
\tilde\delta=\delta,
\qquad
\tilde\omega^2\sin(\lambda u+\tau)=\lambda\omega^2\sin u
\qquad \forall u\in U.
\end{equation}
Let $\kappa=\lambda\omega^2/\tilde\omega^2$. The identity $\sin(\lambda u+\tau)=\kappa\sin u$ is first obtained only on the covered nonempty interval $U$; by analyticity of sine, it extends to all of $\R$. Differentiating twice gives $-\lambda^2\sin(\lambda u+\tau)=-\kappa\sin u$. Combining with the original identity yields $(\lambda^2-1)\kappa\sin u=0$ for all $u$, so $\lambda=\pm1$. If $\lambda=1$, then $\sin(u+\tau)=\kappa\sin u$ forces $\sin\tau=0$ and $\cos\tau=\kappa>0$, hence $\tau=2\pi k$ and $\kappa=1$. If $\lambda=-1$, the same expansion gives $\tau=2\pi k$ and $\kappa=-1$. In both cases $\tilde\omega=\omega$, and direct substitution verifies every witness $u\mapsto\sigma u+2\pi k$ with $\sigma\in\{\pm1\}$ and $k\in\Z$.

\medskip\noindent\textbf{Normalized Van der Pol.}
Because both physical latent and encoder-induced laws are restricted to the normalized Van der Pol family, the orbit equations are
\begin{equation}
\tilde\mu((\lambda u+\tau)^2-1)=\mu(u^2-1),
\qquad
\lambda u+\tau=\lambda u
\qquad \forall u\in U.
\end{equation}
These polynomial identities are first obtained on $U$ and then extend to all $u\in\R$. The second equation gives $\tau=0$. Matching coefficients in $\tilde\mu(\lambda^2u^2-1)=\mu(u^2-1)$ gives $\tilde\mu=\mu$ and then $\lambda^2=1$; direct substitution verifies both signs.

\medskip\noindent\textbf{Cubic Duffing.}
The orbit equations give
\[
\tilde\delta=\delta,\qquad
\tilde\alpha(\lambda u+\tau)+\tilde\beta(\lambda u+\tau)^3
=\lambda(\alpha u+\beta u^3)
\qquad \forall u\in U.
\]
The polynomial identity extends from $U$ to all $u\in\R$. Its $u^3$ and
$u^2$ coefficients give
$\tilde\beta\lambda^3=\lambda\beta$ and
$3\tilde\beta\lambda^2\tau=0$. Since $\lambda\beta\neq0$, we obtain
$\tau=0$, $\tilde\beta=\beta/\lambda^2$, and then
$\tilde\alpha=\alpha$. Direct substitution verifies every scale $\lambda\neq0$.

\medskip\noindent\textbf{Quintic Duffing.}
The orbit equations give
\begin{equation}
\tilde\delta=\delta,
\qquad
\tilde\alpha(\lambda u+\tau)+\tilde\beta(\lambda u+\tau)^3
+\tilde\gamma(\lambda u+\tau)^5
=\lambda(\alpha u+\beta u^3+\gamma u^5)
\qquad \forall u\in U.
\end{equation}
The polynomial identity is first obtained on the covered nonempty interval $U$ and therefore holds for all $u\in\R$. Its $u^5$ and $u^4$ coefficients give $\tilde\gamma\lambda^5=\lambda\gamma$ and $5\tilde\gamma\lambda^4\tau=0$. Since $\lambda\gamma\neq0$, we have $\tilde\gamma\neq0$ and hence $\tau=0$. Matching the $u^3$ and $u$ coefficients then yields $\tilde\beta=\beta/\lambda^2$ and $\tilde\alpha=\alpha$; the $u^5$ identity also gives $\tilde\gamma=\gamma/\lambda^4$. Direct substitution proves the converse, so this is the exact quintic-Duffing orbit.

The Duffing amplitude-anchor statements follow from the same scale action
$u\mapsto\lambda u$. In the quintic case,
$\tilde\beta=\beta/\lambda^2$ and $\tilde\gamma=\gamma/\lambda^4$, while
$\delta$, $\alpha$, $\operatorname{sign}(\beta)$, and $\gamma/\beta^2$ are
invariant. A signed calibration of a nonzero physical latent value fixes
$\lambda$. An unsigned amplitude calibration fixes only $|\lambda|$, but this
is already sufficient because both nonzero-weight coefficients transform
through even powers of $\lambda$. The same argument for the cubic family
uses only $\tilde\beta=\beta/\lambda^2$. Each anchor must pair a nonzero
physical state in $U$ with its encoder value; an anchor outside $U$ needs a
separately justified extension of the affine witness.
\end{proof}

\subsection{Failure without a top gap: Helmholtz}
\label{app:helmholtz-branches}
A missing degree away from the degree immediately below the top need not lock
translations. For
$z''+\delta z'+\alpha z+\beta z^2=0$ on $\R$, with
$\alpha\beta\neq0$, both coefficient signs allowed, and no restriction on
nonzero scales, the damping identity gives $\tilde\delta=\delta$.
Coefficient matching in
$\tilde\alpha(\lambda u+\tau)+\tilde\beta(\lambda u+\tau)^2
=\lambda(\alpha u+\beta u^2)$ gives
$\tilde\beta=\beta/\lambda$,
$\tilde\alpha+2\beta\tau/\lambda=\alpha$, and
$\tau(\tilde\alpha+\beta\tau/\lambda)=0$. Solving these identities yields
\[
(\tau,\tilde\alpha,\tilde\beta)=(0,\alpha,\beta/\lambda)
\quad\text{or}\quad
(\tau,\tilde\alpha,\tilde\beta)
=\left(\frac{\alpha\lambda}{\beta},-\alpha,\frac{\beta}{\lambda}\right).
\]
Thus $f(u)=u+1$ exchanges $z''+z+z^2=0$ and $y''-y+y^2=0$.
Before calibration, the translation branch flips
$\operatorname{sign}(\alpha)$, and $\delta$ and $\alpha^2$ form a complete
set of parameter-orbit invariants for this unrestricted nonzero stratum:
the two branches connect either sign of $\alpha$, and the nonzero scale
connects any two nonzero values of $\beta$ with matching $\delta,\alpha^2$.
If $0\in U$, an origin anchor $f(0)=0$ removes the translated branch and
makes $\alpha$ itself invariant, but leaves the scale $\lambda$ free, with
$\tilde\beta=\beta/\lambda$. After that origin calibration, one additional
signed state anchor at $u_0\in U\setminus\{0\}$ fixes $\lambda$ and hence
$\beta$. An unsigned anchor fixes only $|\lambda|$ and leaves the admissible
reflection branch unresolved. Anchors outside $U$ require a separately
justified extension of the affine witness to their locations.

\section{Detailed Limitations}
\label{app:multidim-proofs}

This appendix integrates the scope of scalar coverage with the limitations of
parameter recovery. We first explain the obstruction to directly extending
our coordinate argument to multiple dimensions, then discuss what the scalar
setting does cover, and finally distinguish these structural results from
finite-video estimation.

\subsection{Multidimensional scope and the coverage obstruction}

Our results concern a scalar configuration variable $z$; the associated
first-order phase-space variable $(z,z')$ is already two-dimensional.
The restriction is therefore to one physical coordinate, not to one
phase-space component or one unknown parameter. The theory does not directly
cover coupled configuration vectors $\bq\in\R^d$ with $d>1$.

The difficulty is not only additional algebra. In the ideal continuous-time
setting, a nonconstant scalar trajectory can traverse a nonempty open
coordinate interval. By contrast, the compact images of finitely many smooth
trajectories in a higher-dimensional configuration space have empty interior.
They can leave an open region unvisited, where a coordinate transformation
can be changed without affecting its values or derivatives on the observed
curves. The following result makes this limitation precise.

\begin{proposition}[Finite passive-curve obstruction]
\label{prop:multidim-obstruction}
Let $d>1$, let $U\subset\R^d$ be a nonempty open set, and let $C\subset U$ be the compact image of finitely many physical latent trajectories with empty interior. Then there exists a non-affine $C^\infty$ diffeomorphism $h_\varepsilon:U\to h_\varepsilon(U)$ such that $h_\varepsilon(\bq)=\bq$, $Dh_\varepsilon(\bq)=\Id$, and $D^2h_\varepsilon(\bq)=0$ for all $\bq\in C$. Thus realized positions, velocities, accelerations, and pointwise residuals along the clips are unchanged, while the coordinate map is non-affine away from the data. This is not a family-preserving parameter transformation; it only explains why coordinate-level multi-dimensional recovery needs additional structure.
\end{proposition}
\begin{proof}
Because $C$ is compact with empty interior in the open set $U$, choose a point $p\in U\setminus C$ and a radius $r>0$ such that the closed ball $\overline B(p,r)$ is contained in $U\setminus C$. Let $\rho\in C_c^\infty(B(p,r))$ be a nonzero bump function and let $e_1$ be the first coordinate vector. Define
\begin{equation}
 h_\varepsilon(q)=q+\varepsilon \rho(q)e_1.
\end{equation}
Extend $\rho$ by zero to $\R^d$ and choose
$0<|\varepsilon|\operatorname{Lip}(\rho)<1$. Then, for all $q,q'\in\R^d$,
\[
\|h_\varepsilon(q)-h_\varepsilon(q')\|
\ge \bigl(1-|\varepsilon|\operatorname{Lip}(\rho)\bigr)\|q-q'\|,
\]
so $h_\varepsilon$ is globally injective. Its Jacobian is everywhere
invertible, and the inverse-function theorem therefore makes it a smooth
diffeomorphism from $U$ onto its image. The map is non-affine because $\rho$ is
nonzero and compactly supported. Since the support of $\rho$ is disjoint from
$C$, $\rho$, $D\rho$, and $D^2\rho$ vanish on $C$. Hence
$h_\varepsilon(q)=q$, $Dh_\varepsilon(q)=I$, and $D^2h_\varepsilon(q)=0$ for all
$q\in C$.

For a realized physical latent trajectory $q(t)\in C$, the transformed trajectory $h_\varepsilon(q(t))$ has the same value, first derivative, and second derivative as $q(t)$ at every sampled time along the clip. Therefore any pointwise residual evaluated only on the realized clips is unchanged, even though the coordinate map is non-affine away from the data.
\end{proof}

The proposition concerns recovery of a coordinate map on an open region,
not parameter identification in every fixed multidimensional ODE family.
The constructed transformation is not asserted to preserve such a family
or to change its parameters. It therefore does not exclude parameter
recovery under additional restrictions. Rather, it shows why matching
passive trajectories and their derivatives alone cannot determine an
unrestricted coordinate extension. Restricted coordinate classes or richer
observations may change this boundary, but require a separate analysis.

\subsection{Why scalar dynamics remain informative}

The scalar setting is not synonymous with linear dynamics or a single
unknown coefficient. The semilinear family permits nonlinear restoring
forces and state-dependent damping, including the pendulum, cubic and
quintic Duffing, and normalized Van der Pol examples. The extensions treat
specified nonlinear velocity features and polynomial-velocity laws, including
quadratic terms induced by nonlinear observation coordinates. The restriction
limits the number of coupled coordinates, rather than reducing the analysis
to the LTI case.

Dimension also changes how passive coverage can constrain the coordinate
map. On a scalar interval $U$, sufficiently many velocities at each state
turn the chain-rule residual into coefficient identities throughout $U$.
For semilinear laws, these force $f''=0$; at the studied quadratic-velocity
boundary, the corresponding identities force an affine relation after
law-derived normalization. This supports parameter analysis across the
covered interval rather than only along isolated observed points. Merely
visiting every state in $U$ is not sufficient: the required velocity
variation must also be present. The two-slope construction in
Appendix~\ref{app:coverage-sufficiency} shows why repeated visits along the
same restricted velocity branches need not supply it.

Within this scope, the parameter conclusions are already distinct: angular
pendulum parameters can be identified despite residual coordinate
symmetries, while Duffing coefficients can retain scale dependence that
physical calibration resolves. These are different outcomes of the theory,
not failures to obtain the same conclusion for every family. The scope is
nevertheless restricted: the family is prescribed and autonomous, the
velocity dependence must satisfy the stated structural conditions, and
canonical comparison is local to regular coordinate intervals. It is not a
classification of arbitrary nonlinear ODEs or a global multi-chart result.

\subsection{From structural guarantees to empirical recovery}

Within the covered model classes, a separate gap remains between structural
identification and estimation from sampled videos. The exact results assume
a shared non-collapsed $C^2$ map and the required physical-velocity coverage.
Finite observations provide only sampled evidence about these continuous
premises. A low residual alone establishes neither the state map nor the
correctness of the declared family. The audits in
Appendix~\ref{app:operational-details} describe which evidence should be
reported without treating it as a proof of those premises.

The semilinear stability results quantify deterministic effects of residuals,
level mismatch, and design conditioning; Appendix~\ref{app:sampled-interface}
separates numerical differentiation and state-consistency errors. They do not
provide a statistical pixel-to-parameter rate. Likewise, the encoder
construction in Proposition~\ref{prop:affine-realizability} establishes
realizability under its representation assumptions, not convergence of
training to that construction. Structural uniqueness, existence of a
compatible encoder, and numerical accuracy are different claims.

Physical calibration has a similar distinction. The rank criterion in
Appendix~\ref{app:additional-consequences} supplies sufficient local
uniqueness on a regular lifted set, with disconnected branches requiring
separate checks. It does not by itself quantify the effect of measurement
error in the anchors. Extending the recovery guarantees to jointly account
for image-derived errors, motion-coverage conditioning, and physical-anchor
uncertainty would connect the structural analysis more directly to empirical
parameter accuracy.

\section{Evaluation metrics}
\label{app:metric-conventions}
The three primary metrics assess complementary aspects of recovery across the
synthetic and real-video experiments: coordinate agreement, parameter
agreement under the permitted coordinate transformation or calibration,
and agreement of the resulting dynamical function. The formulas below
define dimensionless ratios; smaller values indicate closer agreement.
In Figure~\ref{fig:nonlinear-training-pipeline}, coordinate and parameter
comparisons follow the chosen raw or canonical route; physical-unit
comparisons additionally use the stated anchors. The metrics assess empirical
accuracy within those conventions. Their relation to the conditional
identifiability results is discussed in Appendix~\ref{app:operational-details}.

\paragraph{Coordinate NRMSE $e_{\rm map}$ (\%).}
For each fit, let $c$ index test clips and $t$ the valid frames within a clip:
\begin{equation}
 e_{\rm map}=\operatorname{median}_{c}
 \frac{\operatorname{RMS}_{t}(\widetilde x_{ct}-x_{ct})}
 {\operatorname{std}_{t}(x_{ct})}.
 \label{eq:common-coordinate-error}
\end{equation}
Here $x$ is the physical state in synthetic experiments, or tracked angle,
pixel position, or inverse radius in real videos. The prediction
$\widetilde x$ uses a coordinate map fitted on the stated non-test
references and frozen for evaluation. Unrestricted affine comparisons
fit the encoding as a function of the reference and invert the fitted map;
family-constrained comparisons retain their allowed sign, period, scale,
or branch restrictions. Raw and canonical representations use their
respective fitted maps. NRMSE is computed within each clip, then summarized
by the clip median; the standard deviation uses divisor $N_c$. This is a coordinate-recovery
error: it tests whether a frozen readout of the encoder recovers the
reference physical state on new clips, rather than evaluating an integrated
trajectory rollout.

\paragraph{Parameter relative error $e_\theta$ (\%).}
For coefficients expressed in the same comparison coordinate, use
\begin{equation}
 e_{\theta,j}=\frac{|\hat\theta_j-\theta_j^{\rm ref}|}
 {|\theta_j^{\rm ref}|},\qquad
 e_\theta=\max_{j\in J}e_{\theta,j}.
 \label{eq:common-parameter-error}
\end{equation}
The frozen coordinate map determines the reference point on the parameter
orbit. For example, under $z=\lambda q$, a learned cubic coefficient is
compared with $\beta^{\rm ref}=\beta/\lambda^2$. After external calibration,
the estimate is instead compared in physical units. The six-family
experiment takes the maximum over its nonzero dynamical coefficients
$J$; coverage uses the scale-corrected mean drag coefficient, and real
videos use the calibrated scalar $L$ or $g$. Normalization reports the
individual errors $e_{\theta,j}$. Thus this metric tests parameter
agreement at the coordinate map fixed from reference pairs.

\paragraph{Dynamics relative RMSE $e_{\rm dyn}$ (\%).}
For a transformed law or channel $\widetilde F$ and its physical reference
$F$, define
\begin{equation}
 e_{\rm dyn}=\frac{\operatorname{RMS}_{\mathcal S}(\widetilde F-F)}
 {\operatorname{RMS}_{\mathcal S}(F)}.
 \label{eq:common-dynamics-error}
\end{equation}
The evaluation set $\mathcal S$ contains all held-out frames for the
matched-LTI and overhead experiments, or the stated physical-state grid
for the normalized pendulum's restoring-force and damping channels.
Samples receive equal weights within this set. This metric tests whether
the transformed law reproduces the reference acceleration or force over
the evaluation domain, rather than comparing individual coefficients or
integrated trajectory rollouts. The RMS denominator also accommodates
nonzero constant acceleration such as gravity.

\paragraph{Aggregation across fits.}
All three errors are multiplied by $100$ and reported as percentages.
They are computed per fit before reporting median [Q25, Q75] across
the stated cohort. Quartiles use linear interpolation between ordered
observations. Physical estimates use mean and sample SD (divisor $n-1$).
These summaries describe variability over the stated IC/video groups and
optimizer seeds. Each experimental protocol specifies its coordinate
constraints, reference quantities, and evaluation domain.

\raggedbottom
\section{Synthetic experiments: protocols and supporting results}
\label{app:exp-synth}
The appendix follows the main experiments. Appendix~\ref{app:exp-synth-protocol}
supports Section~\ref{sec:exp-orbit}, including the matched affine-LTI comparison.
Appendices~\ref{app:exp-synth-coverage} and~\ref{app:exp-synth-canonical}
separately support the velocity-coverage and projection experiments in Section~\ref{sec:exp-boundary}.
Each experiment follows the same sequence: observations and data,
theoretical predictions, model and training, evaluation, and results
and discussion. Additional analyses follow the main results where applicable.

\subsection{Six-family experiments}
\label{app:exp-synth-protocol}\label{app:exp-synth-orbits}

\subsubsection{Observations and data}
\label{app:exp-synth-observed}

The six-family experiment tests the coordinate and parameter relations of
Section~\ref{sec:scalar-theory}. The five families in Table~\ref{tab:regimes}
are supplemented by the quadratic/Helmholtz family, whose two affine
branches are described in Appendix~\ref{app:helmholtz-branches}.
The LTI experiment uses an affine extension with a learned equilibrium
$\hat c$, which absorbs translations while leaving $\delta$ and $\alpha$
invariant.
We simulate each prescribed law at fixed physical coefficients and render
the resulting state trajectories. The angular pendulum is shown as a moving
pendulum; the other families use spring-oscillator renderings
(Figure~\ref{fig:exp-synth-videos}). The simulated states provide evaluation
references and are not coordinate targets for encoder training.

\begin{figure}[H]
\centering
\includegraphics[width=\linewidth]{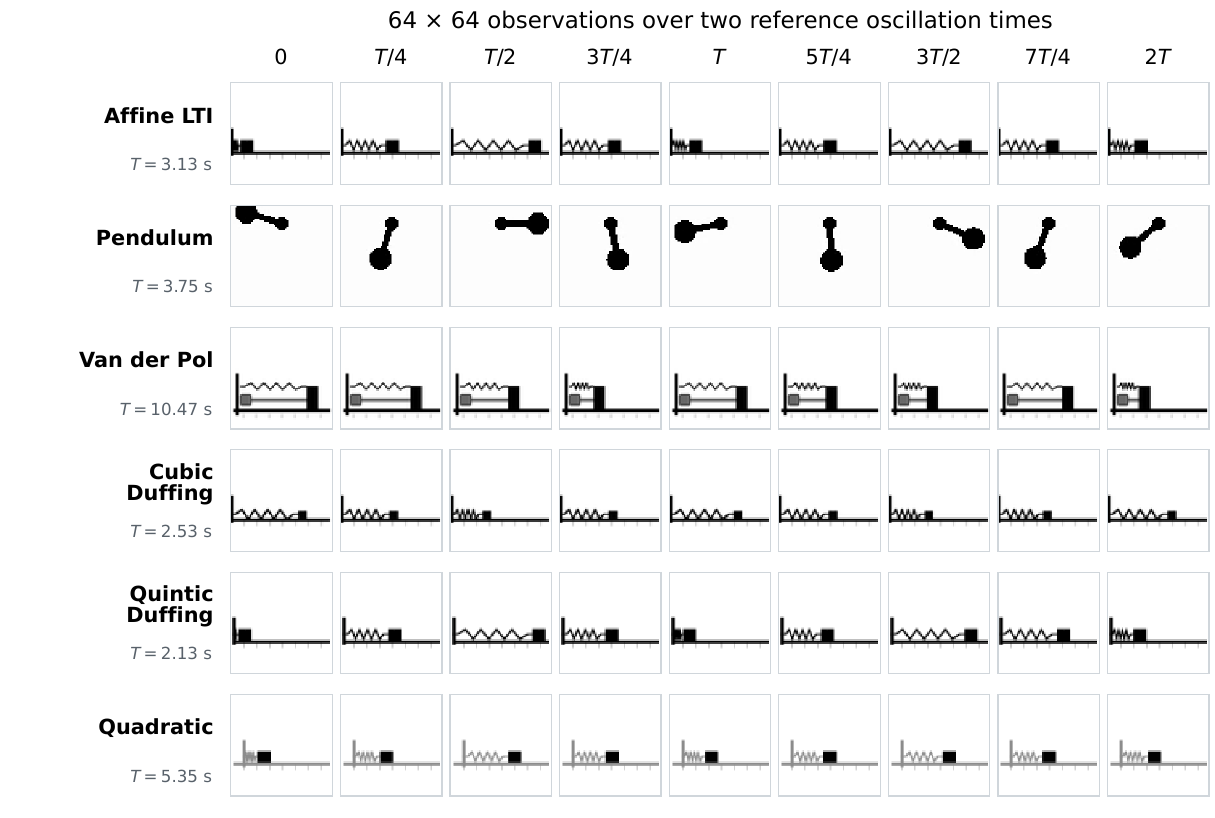}
\caption{\textbf{Temporal evolution over two reference oscillation times.}
Nine observations at $0,T/4,\ldots,2T$ show amplitude and timing changes.
Each row starts at a turning point; $T$ is the first interval between
same-type turning points, not an assumed constant period. Damped systems
ring down, whereas Van der Pol approaches a limit cycle.
All frames are $64\times64$.}
\label{fig:exp-synth-videos}
\end{figure}

\begin{table}[H]
\centering\small
\caption{\textbf{Synthetic observations and data splits.} Clip counts are per
IC collection. Each family uses ten collections $\times$ ten optimizer seeds.
Training clips supply images and timestamps for fitting the encoder and law.
Test clips supply the reported held-out errors (coordinate NRMSE
$e_{\rm map}$ (\%)). The permitted coordinate comparison is fitted using
training reference pairs and frozen on test clips.}
\label{tab:exp-synth-data}
\setlength{\tabcolsep}{4pt}
\begin{tabular}{@{}lll@{}}
\toprule
Family / experiment & Pixels / Hz / duration & Train / test\\
\midrule
Affine LTI & $64^2$ / 60 / 6 s & 16 / 8\\
Pendulum & $64^2$ / 60 / 8 s & 10 / 4\\
Van der Pol & $64^2$ / 60 / 12 s & 6 / 4\\
Cubic Duffing & $64^2$ / 60 / 8 s & 10 / 4\\
Quintic Duffing & $64^2$ / 60 / 3 s & 16 / 8\\
Quadratic / Helmholtz & $64^2$ / 60 / 8 s & 3 / 4\\
\bottomrule
\end{tabular}
\end{table}

For each ODE family, physical coefficients
(Table~\ref{tab:exp-synth-actions}) and the rendering map are fixed.
Table~\ref{tab:exp-synth-initial-states} specifies the nominal initial
position--velocity pairs $(q_0,v_0)$, with $v_0=q'(0)$; the learned
initial value is $z_0=E_\phi(x_0)$ and is not prescribed separately.
Starting from these pairs, ten data seeds generate ten trajectory collections; the different initial conditions
produce different observed video sequences. Each collection contains the
training and test clips listed in Table~\ref{tab:exp-synth-data},
with split assignments inherited from the prescribed trajectory categories.
For the pendulum, initial angles remain fixed and initial angular velocities
receive independent multiplicative perturbations in $[0.95,1.05]$.
For the other five families, initial positions and velocities are perturbed
independently by the same factors.
Each fixed video collection is then fitted with ten optimizer seeds, yielding
$10\times10=100$ fits per family. Data seeds control initial-condition
variation, while optimizer seeds control model initialization and training
sampling.

\begin{table}[H]
\centering\small
\caption{\textbf{Nominal initial states before IC perturbation.}
Pairs are $(q_0,v_0)$ in the physical coordinate. Each $\pm$ includes both
signs; in the pendulum row, angle and velocity signs vary independently.
LTI and quintic use the same IC design.}
\label{tab:exp-synth-initial-states}
\setlength{\tabcolsep}{4pt}
\renewcommand{\arraystretch}{1.12}
\begin{tabularx}{\linewidth}{@{}p{.17\linewidth}YY@{}}
\toprule
Family & Training pairs & Test pairs\\
\midrule
LTI / quintic &
$(\pm a,0)$, $a\in\{0.25,0.55,0.85,1.15\}$;
$(0,\pm b)$, $b\in\{0.8,1.8,2.8,3.4\}$ &
$(\pm a,0)$, $a\in\{0.65,1.00\}$;
$(0,\pm b)$, $b\in\{0.40,1.00\}$\\
Pendulum &
$(\pm\pi/2,\pi h)$, $h\in\{-1/2,-1/4,0,1/4,1/2\}$ &
$(\pm\pi/2,\pm\pi/8)$\\
Van der Pol &
$(0,\pm0.5)$, $(0,\pm2.5)$, $(\pm3,0)$ &
$(0,\pm1.25)$, $(\pm2.75,0)$\\
Cubic &
$(0,\pm\sqrt{4p^2+2p^4})$,
$p\in\{0.40,0.55,0.70,0.85,1.00\}$ &
$(0,\pm\sqrt{4p^2+2p^4})$,
$p\in\{0.625,0.925\}$\\
Quadratic &
$(0,-1.2)$, $(0,0.8)$, $(0,1.2)$ &
$(-0.4,0.65)$, $(0.4,-0.65)$,
$(-0.15,1.0)$, $(0.15,-1.0)$\\
\bottomrule
\end{tabularx}
\end{table}

\subsubsection{Theoretical predictions}
\label{app:exp-synth-theory}
Under the shared-state and velocity-coverage assumptions of
Theorem~\ref{thm:semilinear-affine-collapse}, a compatible coordinate map
is affine. Requiring the transformed law to remain in the declared family
further restricts this map and determines the corresponding parameter orbit
(Table~\ref{tab:exp-synth-actions}). Thus invariant coefficients should be
preserved across fits, while coefficients that depend on coordinate scale
should follow the same scale inferred from the learned coordinate.
The quadratic family also requires checking both permitted affine branches.
These experiments test coordinate recovery and parameter recovery against
the same family-permitted relation.

For experimental plots, write $q$ for physical state and $z=E_\phi(x)$ for
learned state (the theory's $z$ and $\hat z$, respectively), with hats on
learned coefficients. The comparison map is determined by the declared dynamical family.

\begin{table}[H]
\centering\small
\caption{Declared laws, fixed physical coefficients and permitted coordinate actions. Nonzero
$\lambda$ is estimated from training coordinate pairs. The quintic invariant
uses $\beta\ne0$; the plotted physical coefficients are positive.}
\label{tab:exp-synth-actions}
\begin{tabularx}{\linewidth}{@{}p{.40\linewidth}Y@{}}
\toprule
Learned law / physical coefficients & Allowed map / parameter consequence\\
\midrule
$z''+\hat\delta z'+\hat\alpha(z-\hat c)=0$\newline {\footnotesize $(\delta,\alpha,c)=(0.2,4,0)$} &
$z\simeq\lambda q+\tau$; $\hat c\simeq\lambda c+\tau$; $\delta,\alpha$ invariant.\\
$z''+\hat\delta z'+\hat\omega^2\sin z=0$\newline {\footnotesize $(\delta,\omega^2)=(0.15,4)$} &
$z\simeq sq+2\pi n$; $s=\pm1$, $n\in\mathbb Z$; $\delta,\omega^2$ invariant.\\
$z''-\hat\mu(1-z^2)z'+z=0$\newline {\footnotesize $\mu=1.5$} & $z\simeq\pm q$; $\mu$ invariant.\\
$z''+\hat\delta z'+\hat\alpha z+\hat\beta z^3=0$\newline {\footnotesize $(\delta,\alpha,\beta)=(0.2,4,4)$} &
$z\simeq\lambda q$; $\hat\beta\simeq\beta/\lambda^2$;
$\delta,\alpha,\operatorname{sign}\beta$ invariant.\\
$z''+\hat\delta z'+\hat\alpha z+\hat\beta z^3+\hat\gamma z^5=0$\newline {\footnotesize $(\delta,\alpha,\beta,\gamma)=(0.2,4,3,2)$} &
$z\simeq\lambda q$; $\hat\gamma\simeq\gamma/\lambda^4$;
also $\gamma/\beta^2$ invariant.\\
$z''+\hat\delta z'+\hat\alpha z+\hat\beta z^2=0$\newline {\footnotesize $(\delta,\alpha,\beta)=(0.18,1.6,0.8)$} &
Origin: $z\simeq\lambda q$, $\hat\alpha\simeq\alpha$, $\hat\beta\simeq\beta/\lambda$.
Shifted: $z\simeq\lambda(q+\alpha/\beta)$, $\hat\alpha\simeq-\alpha$.\\
\bottomrule
\end{tabularx}
\end{table}

\subsubsection{Model and training}
\label{app:exp-synth-training-details}

LTI, cubic, quintic and quadratic use per-frame CNNs with three convolutional
blocks (base width 32), GroupNorm, GELU, a 128-unit MLP and a scalar output.
Inputs are $64\times64$; cubic frames use Lanczos resizing.
The encoder and dynamical coefficients are optimized jointly. The
$\mathcal L_{\mathrm{var}}$ penalty introduced in Section~\ref{sec:setup} uses the
standard-deviation threshold $s_{\mathrm{floor}}=0.02$
(Eq.~\ref{eq:AB_variance_floor}). Table~\ref{tab:exp-synth-training}
lists the optimization settings for all six families.
The pendulum and Van der Pol scalar output biases are initialized in
$[-3\pi,3\pi]$ and $[-0.5,0.5]$, respectively.

\begin{table}[H]
\centering\small
\caption{\textbf{Training settings for the six-family experiment.}
All use Adam and the final checkpoint after the stated number of updates
(optimizer steps). Learning rates are encoder / dynamical law.
Each family has 100 fits.}
\label{tab:exp-synth-training}
\setlength{\tabcolsep}{5pt}
\begin{tabular}{@{}lcccc@{}}
\toprule
Family & IC collections & Seeds per set & Updates & $\mathrm{lr}_{\mathrm{enc}}/\mathrm{lr}_{\mathrm{ode}}$\\
\midrule
Affine LTI & 10 & 10 & 2,000 & $0.0005/0.003$\\
Pendulum & 10 & 10 & 2,500 & $0.001/0.005$\\
Van der Pol & 10 & 10 & 5,000 & $0.001/0.003$\\
Cubic Duffing & 10 & 10 & 8,000 & $0.0005/0.1$\\
Quintic Duffing & 10 & 10 & 8,000 & $0.0005/0.003$\\
Quadratic / Helmholtz & 10 & 10 & 6,000 & $0.0005/0.003$\\
\bottomrule
\end{tabular}
\end{table}

\subsubsection{Evaluation}
\label{app:exp-synth-evaluation}
After training, freeze the encoder and fitted law. Estimate the
family-permitted map $\hat f$ from training reference pairs $(q,z)$:
LTI uses an affine least-squares map, the pendulum selects a sign and
$2\pi$ shift, and Van der Pol selects a sign. Cubic and quintic use
$\lambda=\sum qz/\sum q^2$. Quadratic fits both permitted branches on
training data and freezes the branch and scale with the smaller coordinate
error (Table~\ref{tab:exp-synth-actions}). No test states are used to
select or refit this map.

On test clips, set $x=q$ and $\widetilde x=\hat f^{-1}(z)$ in
Eq.~\eqref{eq:common-coordinate-error}: compute coordinate NRMSE within
each clip and take the clip median to obtain one $e_{\rm map}$ per fit.
For parameter evaluation, transform the physical law with the same frozen
map and branch to obtain $\theta^{\rm ref}$. Equation~\eqref{eq:common-parameter-error}
then gives $e_\theta$ as the largest relative error over the listed nonzero
dynamical coefficients, retaining coefficient signs and excluding the
LTI equilibrium. Thus the coordinate and parameter checks test the same
predicted relation, rather than choosing a scale from coefficient ratios.
The 100 per-fit errors are summarized using the conventions in
Appendix~\ref{app:metric-conventions}; the acceleration comparison below
uses $e_{\rm dyn}$ from Eq.~\eqref{eq:common-dynamics-error}.

\subsubsection{Results and discussion}

Figure~\ref{fig:exp-orbit} shows the learned coefficients and coordinate
maps. Table~\ref{tab:app-exp-synth-cohort} summarizes the 100 fits per family;
Figures~\ref{fig:exp-synth-errors} and~\ref{fig:exp-synth-orbit-by-ic}
separate these fits by IC collection, showing the ten optimizer seeds
within each collection.

\begin{table}[H]
\centering\small
\caption{\textbf{Parameter relative error and coordinate recovery.}
Percentages are median [Q25, Q75] over all 100 fits per family.}
\label{tab:app-exp-synth-cohort}
\begin{tabular}{@{}lcc@{}}
\toprule
Family (100 fits) & Parameter $e_\theta$ (\%) & Coordinate $e_{\rm map}$ (\%)\\
\midrule
Affine LTI & 0.85 [0.52, 1.65] & 1.68 [1.53, 1.90]\\
Pendulum & 1.44 [0.84, 2.17] & 3.58 [2.27, 5.99]\\
Van der Pol & 0.31 [0.13, 0.61] & 2.45 [1.69, 3.71]\\
Cubic & 2.56 [1.60, 4.60] & 0.82 [0.68, 1.08]\\
Quintic & 1.68 [1.02, 2.29] & 0.78 [0.62, 1.06]\\
Quadratic & 2.47 [1.31, 4.37] & 3.29 [1.96, 6.08]\\
\bottomrule
\end{tabular}
\end{table}

\begin{figure}[H]
\centering
\includegraphics[width=\linewidth]{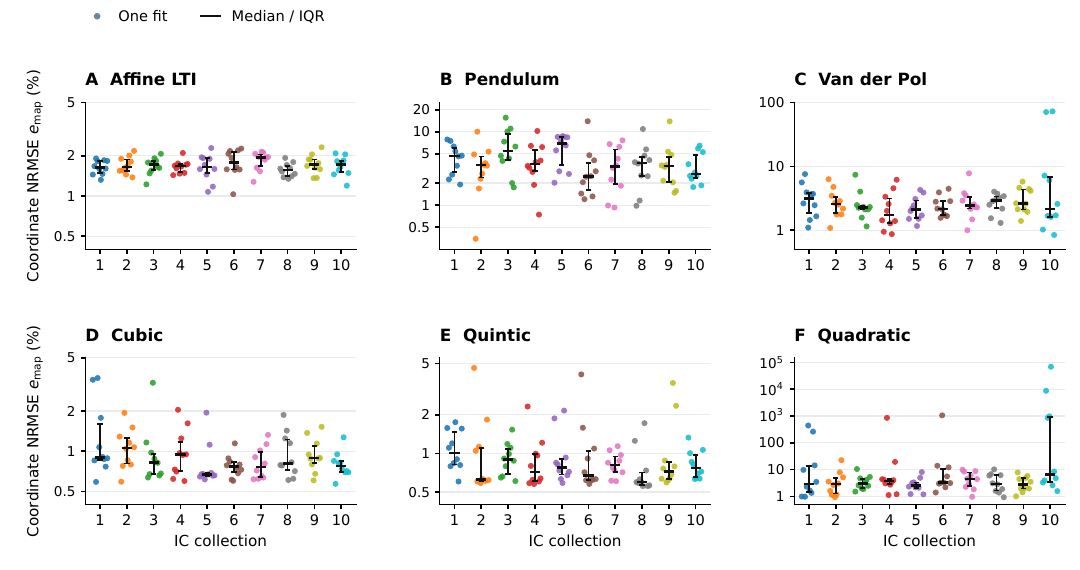}
\caption{\textbf{Test coordinate NRMSE $e_{\rm map}$ by IC collection.}
Each family contains ten IC collections, each with ten optimizer seeds.
Dots show $e_{\rm map}$; black bars
give each collection's median and interquartile range.
Colors match Figure~\ref{fig:exp-orbit};
IC indices are local to each family. The logarithmic vertical limits vary
by family.}
\label{fig:exp-synth-errors}
\end{figure}

\begin{figure}[H]
\centering
\includegraphics[width=\linewidth]{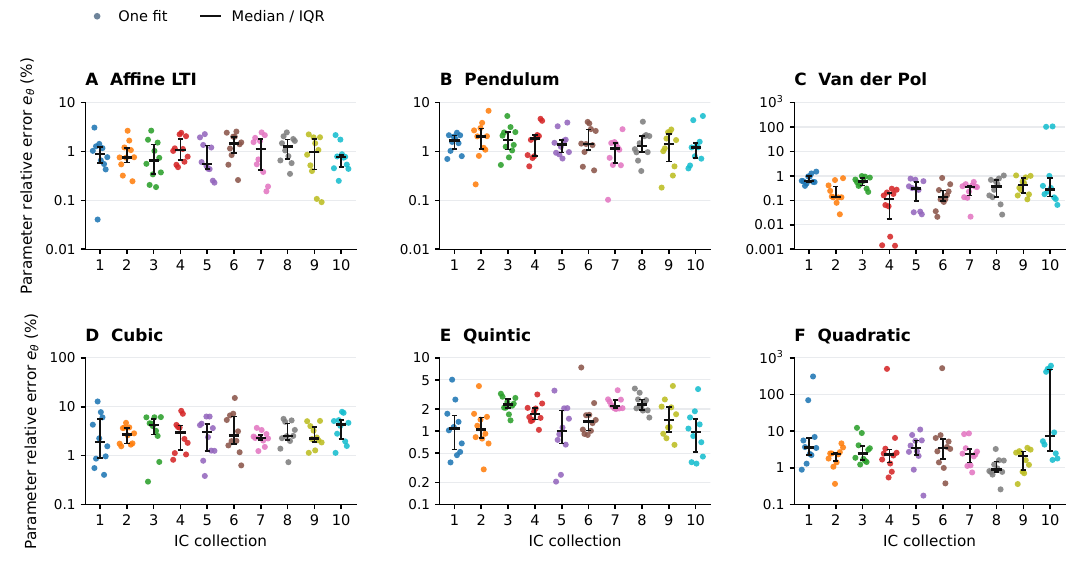}
\caption{\textbf{Parameter relative error $e_\theta$ by IC collection.}
The same fits and IC groups as in Figure~\ref{fig:exp-synth-errors}, now
showing $e_\theta$.
Within each group, horizontal offsets follow optimizer-seed order and
are identical in both figures. Bars show median and interquartile range.}
\label{fig:exp-synth-orbit-by-ic}
\end{figure}

Median parameter relative errors range from $0.31\%$ to $2.56\%$.
Cubic and quintic combine sub-$1\%$ median coordinate NRMSE with the
predicted $\beta/\gamma$ scaling in Figure~\ref{fig:exp-orbit}.
Pendulum and Van der Pol have median coordinate errors of $3.58\%/2.45\%$
and parameter errors of $1.44\%/0.31\%$.
The IC breakdown shows both variation across trajectory collections and
the spread across optimizer seeds within each collection.
These comparisons assess finite-data agreement with the family-specific
relations. Proposition~\ref{prop:affine-realizability} constructs exact
alternative encoders under representability and affine-closure assumptions;
the empirical errors do not certify these assumptions or every continuous-time
premise (Appendix~\ref{app:operational-details}).

\subsubsection{Additional analysis: matched affine-LTI comparison}
\label{app:exp-synth-adequacy}

This experiment probes the scope of affine-LTI modeling and the loss in
dynamical accuracy caused by ODE-family misspecification. We keep the
observations and ICs unchanged, train the encoder jointly with an affine-LTI
law, and then evaluate the fitted coefficients on trajectories from the
held-out test videos. The affine-LTI dataset provides a matched-family
control, while the five nonlinear datasets test the cost of restricting
the learned dynamics to this family.
For each of the six families, we pair the 100 fitted models with 100
affine-LTI controls, $z''+\hat\delta z'+\hat\alpha(z-\hat c)=0$.
Each pair shares the video data, encoder architecture and initial weights,
training-window sampling, optimizer settings, and update budget
(Table~\ref{tab:exp-synth-training}); the LTI equilibrium $\hat c$ is learned
jointly. 
Pairing covers ten IC collections and ten optimizer seeds per family.

For each model, fit $z\simeq\lambda q+\tau$ on training frames and freeze
it on test trajectories. Evaluate $e_{\rm dyn}$ using the stored physical
states and derivatives, with
\[
 \widetilde F=\lambda^{-1}\hat F(\lambda q+\tau,\lambda q'),
 \qquad F=q''.
\]
Here $\hat F$ uses the fitted coefficients without refitting on test
data. Equation~\eqref{eq:common-dynamics-error} measures how accurately
the learned dynamical law reproduces physical acceleration along the
held-out trajectories, including the effect of family misspecification;
it does not integrate a rollout or compare coefficients from different
families. The LTI oracle fits acceleration directly to $(q,q',1)$ on training trajectories
with the matching coefficient-sign constraints and uses the same test
comparison. Wins count paired fits with lower $e_{\rm dyn}$ for the
declared family model.

\begingroup
\setlength{\intextsep}{6pt}
\begin{figure}[H]
\centering
\includegraphics[width=\linewidth]{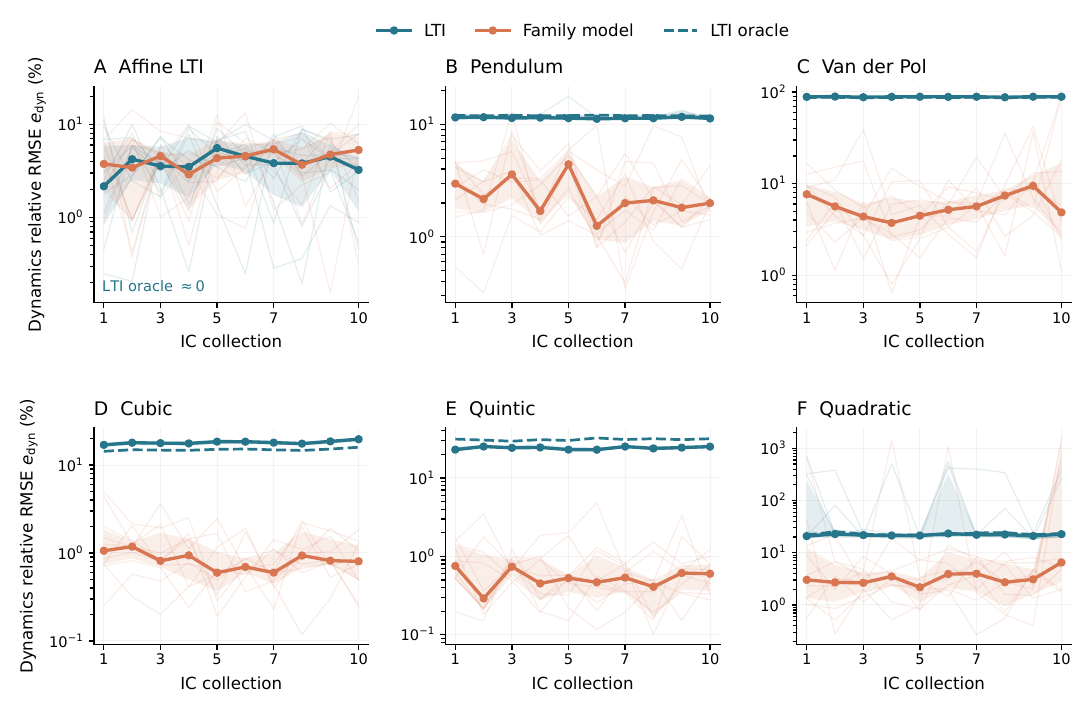}
\caption{\textbf{Matched affine-LTI comparison across IC collections.}
Each family has ten IC collections and ten paired optimizer seeds per
collection. Faint lines show individual seeds; solid curves and bands give
the within-IC median and interquartile range of test dynamics relative
RMSE $e_{\rm dyn}$ (\%), defined in Eq.~\eqref{eq:common-dynamics-error}
and evaluated on physical accelerations. Dashed curves show the training-fitted physical-state LTI oracle.
Both methods in A use the affine-LTI family.}
\label{fig:app-exp-synth-adequacy}
\end{figure}

\begin{table}[H]
\centering\small
\caption{\textbf{Test dynamics relative RMSE $e_{\rm dyn}$ (\%).}
Physical-acceleration comparison using Eq.~\eqref{eq:common-dynamics-error};
median [Q25, Q75] over 100 pairs per family. Wins count pairs with a smaller
$e_{\rm dyn}$ for the declared family model.}
\label{tab:exp-matched-lti}
\begin{tabular}{@{}lccc@{}}
\toprule
Dataset & Affine LTI (\%) & Family model (\%) & Family wins\\
\midrule
Affine LTI & 3.86 [1.96, 6.58] & 4.56 [2.16, 6.84] & 46/100\\
Pendulum & 11.34 [11.20, 11.70] & 2.17 [1.43, 3.89] & 99/100\\
Van der Pol & 88.48 [88.01, 88.79] & 5.80 [3.07, 9.21] & 98/100\\
Cubic & 17.92 [17.46, 18.53] & 0.81 [0.57, 1.39] & 100/100\\
Quintic & 24.29 [23.46, 25.02] & 0.54 [0.33, 0.89] & 100/100\\
Quadratic & 22.03 [20.42, 23.85] & 3.15 [1.70, 6.68] & 92/100\\
\bottomrule
\end{tabular}
\end{table}
\endgroup

Across the five nonlinear datasets, the declared family has lower median
dynamics error $e_{\rm dyn}$ in all ten IC collections and wins 92--100 of the
100 individual pairs. On the affine-LTI dataset, both columns instantiate the
same model family and have comparable error distributions; their small
difference is due to finite-precision divergence from a different parameter
traversal order in global gradient clipping. A controlled replay with a common
accumulation order reproduces matching training trajectories.

\FloatBarrier
\subsection{Velocity coverage and parameter recovery}
\label{app:exp-synth-coverage}\label{app:exp-synth-observations}

\subsubsection{Observations and data}
\label{app:exp-synth-coverage-setup}

The physical system is
$q''=-kq-\delta q'-\kappa q'|q'|$, with spring coefficient $k=0.8$,
linear damping $\delta=0.18$, and quadratic drag $\kappa=0.22$.
Each IC collection contains 24 training and eight test clips of 1 s,
rendered at $96\times96$ pixels and 60 Hz (61 frames).

\begin{figure}[H]
\centering
\includegraphics[width=\linewidth,trim=0 76 0 30,clip]{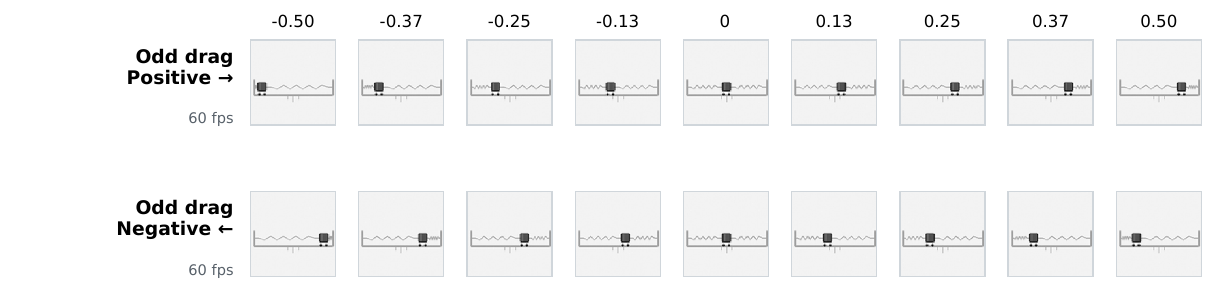}
\caption{Opposite odd-drag crossings. Both rows show $96\times96$ inputs
at 60 Hz, aligned at the same position at $t=0$. The central frames have
the same geometry despite opposite velocities; labels give time in seconds.}
\label{fig:exp-boundary-videos}
\end{figure}

We vary the number $K$ of speed magnitudes and whether trajectories cross
$q=0$ in one or both directions. As in Appendix~\ref{app:exp-synth-protocol},
an \emph{IC collection} groups videos generated from one prescribed set
of physical trajectory conditions. Here the conditions are the crossing
velocities at $q=0$, which determine the simulated trajectories.
Ten data seeds independently perturb
the base magnitudes $\{0.8,1.35,1.95\}$ by factors in $[0.97,1.03]$.
One-way $K=1$ uses the middle magnitude; $K=3$ uses all three;
$K=6$ uses six values spanning the same envelope. Two-way $K=1$ uses
both signs of the middle magnitude; $K=2$ uses both signs of the smallest
and largest magnitudes; $K=3$ uses all three. Appearance repetitions give
24 training clips in every condition. The test pools use crossing speeds
near $\{1.1,1.7\}$ for one-way designs and $\{-1.1,1.1\}$ for two-way
designs. Evaluation below uses the four positive-direction test clips
near $1.1$ shared by all six conditions within each IC collection.

\subsubsection{Theoretical predictions}
\label{app:exp-synth-coverage-theory}
The experiment tests the curvature-column exclusion in
Theorem~\ref{thm:feature-separation}. A nonlinear coordinate map contributes
$f''(q)v^2$ to the acceleration. Although $v|v|$ differs from $v^2$ over
two-way motion, these columns coincide for positive velocities (and differ
only by sign for negative velocities). One-way motion therefore cannot
separate drag from coordinate curvature, regardless of the number of
speeds. With two directions and two distinct speed magnitudes, the
augmented feature matrix $[1,v,v|v|,v^2]$ has rank four, giving the
theorem's raw-affine conclusion under its remaining assumptions.
The corresponding drag relation is $c_d=-\kappa/|\lambda|$
(Eq.~\eqref{eq:exp-coverage-orbit}). The test is whether this change in
coverage improves parameter recovery even when coordinate errors are small.

\subsubsection{Model and training}
\label{app:exp-synth-coverage-training}

Each of the ten IC collections is crossed with five optimizer seeds, giving 50 fits per
condition. A spatial soft-argmax encoder is followed by a learned cubic
polynomial head. It is jointly fitted with
$z''=c_0(z)+c_1(z)z'+c_d(z)z'|z'|$, where the polynomial degrees of
$(c_0,c_1,c_d)$ are $(5,3,5)$. Derivatives use a 31-frame, degree-two
local-polynomial fit. Table~\ref{tab:exp-coverage-training} gives the
optimization settings; evaluation uses the final checkpoint.

\begin{table}[H]
\centering\small
\caption{\textbf{Coverage training settings.} Adam; each coverage condition uses ten IC collections and five optimizer seeds per collection.}
\label{tab:exp-coverage-training}
\begin{tabular}{@{}cccc@{}}
\toprule
IC collections & Seeds per collection & Updates & $\mathrm{lr}_{\mathrm{enc}}/\mathrm{lr}_{\mathrm{ode}}$\\
\midrule
10 & 5 & 5,000 & $0.0005/0.002$\\
\bottomrule
\end{tabular}
\end{table}

\subsubsection{Evaluation}
\label{app:exp-synth-coverage-evaluation}
Fit $z\simeq\lambda q+\tau$ on training references, then use $x=q$ and
$\widetilde x=(z-\tau)/\lambda$ in Eq.~\eqref{eq:common-coordinate-error}
to compute $e_{\rm map}$. Each clip contributes its frames with
$q\in[-0.2,0.2]$; all six conditions use the same four test clips within
each IC collection.

Substitution of $z=\lambda q+\tau$ into the physical law gives
\begin{equation}
 c_0(z)=-k(z-\tau),\qquad c_1(z)=-\delta,\qquad
 c_d(z)=-\frac{\kappa}{|\lambda|}.
 \label{eq:exp-coverage-orbit}
\end{equation}
Let $\bar\kappa=-\operatorname{mean}c_d$ be the mean learned drag
magnitude, evaluated on 101 equally spaced learned-coordinate values
over the encoded training interval $q\in[-0.2,0.2]$.
The parameter comparison uses the scale-corrected mean
$\widehat\kappa_{\mathrm{aligned}}=|\lambda|\bar\kappa$ against
$\kappa=0.22$ in Eq.~\eqref{eq:common-parameter-error}, with
the same training-fitted $\lambda$ as the coordinate
comparison. Figure~\ref{fig:exp-coverage-orbit} shows both errors for all
six designs.

The feature rank concerns $[1,v,v|v|,v^2]$
(Section~\ref{subsec:feature-separation}). It is evaluated at 21 shared
positions in $[-0.2,0.2]$, using the distinct velocities at each position;
Figure~\ref{fig:exp-coverage-orbit}A reports the minimum rank. Rank four separates the declared features from the curvature column.
A smaller rank fails this sufficient condition; it does not by itself
establish nonidentifiability of every parameter.

\subsubsection{Results and discussion}
\label{app:exp-synth-coverage-results}

\begin{figure}[H]
\centering
\includegraphics[width=\linewidth]{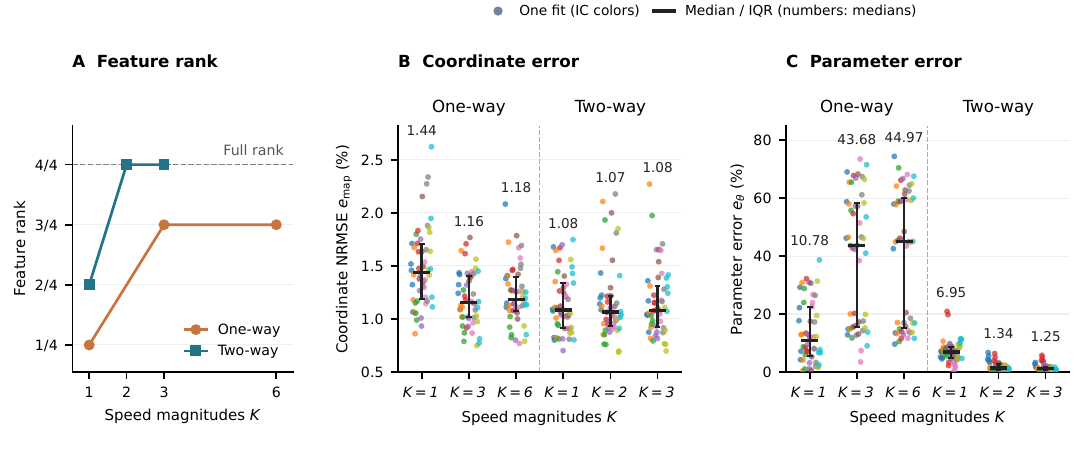}
\caption{\textbf{Velocity coverage, parameter recovery, and coordinate accuracy.}
A: minimum rank of $[1,v,v|v|,v^2]$ over 21 shared positions in
$[-0.2,0.2]$; all ten IC collections give the displayed rank for each condition.
The dashed line marks full rank. B: coordinate NRMSE $e_{\rm map}$
(Eq.~\eqref{eq:common-coordinate-error}) on the same four test clips
within each IC collection. C: parameter relative error $e_\theta$ of the
scale-corrected mean drag coefficient (Eq.~\eqref{eq:common-parameter-error}).
B/C show all 50 fits per condition, colored by IC collection, with matching
horizontal offsets. Black bars show the median and interquartile range
over the 50 fits, using the convention of Figures~\ref{fig:exp-synth-errors}
and~\ref{fig:exp-synth-orbit-by-ic}; numbers give medians (\%).
All axes are linear. $K$ counts speed magnitudes.}
\label{fig:exp-coverage-orbit}
\end{figure}

One-way motion gives $v|v|=v^2$, so adding speeds does not separate the
drag feature from the curvature term introduced by a nonlinear coordinate
map. Median drag-parameter errors are $43.68\%$ and $44.97\%$ for
one-way $K=3,6$, compared with $1.34\%$ and $1.25\%$ for two-way
$K=2,3$, which have full feature rank. Median coordinate NRMSE is
$1.07$--$1.44\%$ across all six conditions. Accurate coordinates therefore
coexist with substantial errors in the recovered mean drag coefficient
when velocity coverage is insufficient.

\FloatBarrier
\subsection{Canonical normalization}
\label{app:exp-synth-canonical}

\subsubsection{Observations and data}
\label{app:exp-synth-canonical-setup}

To test whether canonical normalization corrects nonlinear coordinate
distortion, we simulate a pendulum with angular dynamics
$q''+0.15q'+4\sin q=0$, released from rest at different initial angles.
An orthographic camera looks along the vertical axis and records
$y=\sin q$ on $|q|<\pi/2$, discarding the depth coordinate
(Figure~\ref{fig:exp-projection-geometry}). The projected rod and bob are
rendered as $96\times96$ grayscale frames. Thus the observation depends
nonlinearly on the physical angle even before learning.

\begin{figure}[H]
\centering
\includegraphics[width=.82\linewidth]{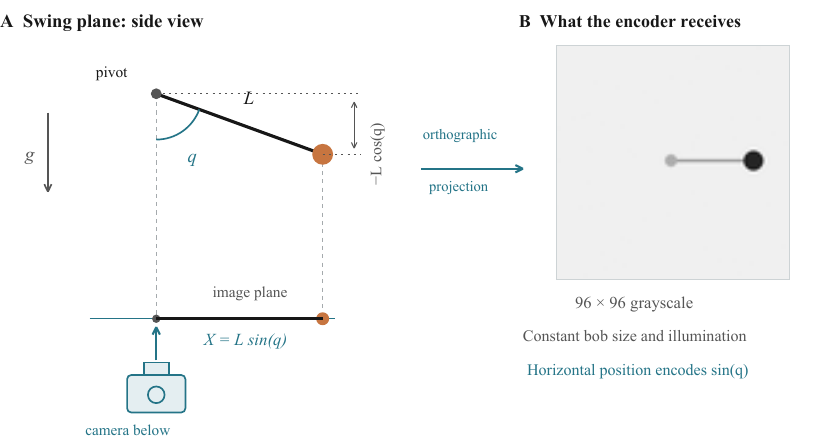}
\par\smallskip
\includegraphics[width=\linewidth,trim=0 0 0 173,clip]{figures/experiments/boundary-video-strips.pdf}
\caption{\textbf{Projected pendulum observations.}
An orthographic camera looks along the vertical axis, retaining
$X=L\sin q$ and discarding the depth coordinate $-L\cos q$.
The projected rod and constant-size bob produce the image coordinate
$y=\sin q$. Bottom: nine $96\times96$ observations from $0$ to $3T/2$
for a $70^\circ$ release; $T$ is the first full oscillation duration.
The illustrative strip continues the original ODE and renderer beyond
the 4.8 s training-clip duration.}
\label{fig:exp-projection-geometry}
\end{figure}

We train on three videos released at approximately $45^\circ,60^\circ,70^\circ$,
and test on held-out releases near $40^\circ,55^\circ,65^\circ$.
Each training angle has one rendering; each test angle has four, giving
3 training and 12 test videos. All videos last 4.8\,s at 30\,Hz (145 frames).
We repeat training from three random initializations on the same video data.
For each fitted model, Raw and Canonical are compared on exactly the same
test videos.

\subsubsection{Theoretical predictions}
\label{app:exp-synth-canonical-theory}
Differentiating $y=\sin q$ introduces the squared-velocity coefficient
$-y/(1-y^2)$ into the projected law. We fit the corresponding structured family in a learned
coordinate with free scale $s>0$ and offset $b$:
\begin{equation}
 z''=-\delta z'-\omega^2(z-b)\sqrt{1-u^2}
 -\rho\frac{z-b}{s^2-(z-b)^2}(z')^2,\qquad u=\frac{z-b}{s}.
 \label{eq:app-exp-synth-projected-law}
\end{equation}
Here $\rho$ controls the squared-velocity channel. The exact projected
law $z=s\sin q+b$ has $(\delta,\omega^2,\rho)=(0.15,4,1)$.
For optimization, the fitted-law parameters are initialized at
$(\delta,\omega^2,\rho,s,b)=(0.12,3.2,0.8,0.82,0)$ in all configurations.


This experiment tests Section~4.2: a squared-velocity channel permits
nonlinear coordinate changes, so the raw coordinate need not be affinely
related to the physical angle. Affine alignment alone therefore cannot in
general remove this distortion. We consequently apply the normalizer in
Eq.~(60), which removes the squared-velocity channel and becomes
$s\arcsin((z-b)/s)$ when $\rho=1$.

The physical reference is the projected law induced by $y=\sin q$, whose
normalizer gives $\psi_\theta(y)=\arcsin y=q$. Theorem~4.2 therefore predicts
a shared affine relation
$\psi_\eta(z)\simeq \lambda q+\tau$ under its coverage assumptions, which we
test through canonical coordinate recovery and the transformed restoring
force and damping. The theorem applies on the regular chart $|z-b|<s$ under
the sufficient coverage condition of at least three distinct physical
velocities at each covered position for this degree-two velocity family.

\subsubsection{Model and training}
\label{app:exp-synth-canonical-training}
We jointly fit the spatial soft-argmax encoder and the law in
Eq.~\eqref{eq:app-exp-synth-projected-law}, without physical-coordinate targets.
Each fit uses 6,000 Adam updates and four 97-frame windows per step;
learning rates are $2\times10^{-4}$ for the encoder and $2\times10^{-3}$
for the law. Derivatives use 21-frame, degree-four local polynomials.
The standard-deviation threshold is $s_{\mathrm{floor}}=0.03$ with the weight $\lambda_{\mathrm{var}}=100$.

\subsubsection{Evaluation}
\label{app:exp-synth-canonical-evaluation}
For the main training trajectories, a numerical first-passage check on
161 positions in $[-0.4,0.4]$ rad finds three distinct velocities and
rank three for $[1,v,v^2]$. The full-test coordinate evaluation and wider
force grid also examine behavior beyond this checked interval.

After freezing the encoder and law, solve $\psi''+c_2\psi'=0$ with
$\psi(b)=0$ and $\psi'(b)=1$:
\begin{equation}
 \psi(z)=s\int_0^{(z-b)/s}(1-u^2)^{-\rho/2}\,du.
 \label{eq:exp-projection-normalizer}
\end{equation}
This map uses the learned $(\rho,s,b)$ and requires no physical anchors.
Cumulative trapezoidal quadrature uses 40,001 nodes spanning the encoded
training/test values and $b$. A defined map requires $s^2-(z-b)^2>0$
throughout that interval; all three main fits pass this check.
For $h(z)=z$ and $h(z)=\psi(z)$ separately, fit $h(z)\simeq a_hq+b_h$ on
training references, then freeze $\widetilde q=[h(z)-b_h]/a_h$ for testing.
Both readouts have identical affine freedom. To assess coordinate accuracy
at release angles unseen in training, we evaluate coordinate NRMSE
$e_{\rm map}$ (Appendix~\ref{app:metric-conventions}) on held-out videos
near $40^\circ,55^\circ,65^\circ$. Compute it within each test clip using
the original test frames, then take the clip median. Binning is used only
to display residuals.

Report coefficient relative errors against
$(\delta,\omega^2,\rho)=(0.15,4,1)$. Let $A(z)$ be the velocity-independent
term of Eq.~\eqref{eq:app-exp-synth-projected-law}. With the frozen canonical
readout, evaluate
\[
 z_q=\psi^{-1}(a_\psi q+b_\psi),\qquad
 \widetilde F(q)=\psi'(z_q)A(z_q)/a_\psi.
\]
Compare $\widetilde F$ with $-4\sin q$ on 361 points in $[-0.9,0.9]$ rad,
using the common dynamics relative RMSE $e_{\rm dyn}$ from
Appendix~\ref{app:metric-conventions}. This measures the restoring-force
channel. The damping channel remains $-\delta$, so its relative error is
the coefficient error of $\delta$. All errors are percentages, summarized
as median [Q25,Q75] across fits.

\subsubsection{Results and discussion}
\label{app:exp-synth-canonical-results}
Table~\ref{tab:exp-normalization} shows that normalization reduces coordinate
NRMSE $e_{\rm map}$ from $5.633\%$ to $0.305\%$ and recovers the restoring
force to $0.193\%$ relative RMSE. Figure~\ref{fig:exp-normalization-residuals}
shows the corrected residuals and their variation across fits;
Table~\ref{tab:app-exp-synth-canonical} reports the corresponding coefficient errors.

{\setlength{\intextsep}{4pt}
\begin{table}[H]
\centering\small
\caption{\textbf{Fitted-law coefficient errors.}
Parameter relative errors $e_{\theta,j}$ against
$(\delta,\omega^2,\rho)=(0.15,4,1)$; median [Q25,Q75] over three fits
on the same IC collection.}
\label{tab:app-exp-synth-canonical}
\setlength{\tabcolsep}{10pt}
\begin{tabular}{@{}ccc@{}}
\toprule
Damping $e_{\theta,\delta}$ & Restoring coefficient $e_{\theta,\omega^2}$ & Quadratic channel $e_{\theta,\rho}$\\
\midrule
0.64\% [0.53\%, 0.68\%] & 0.21\% [0.20\%, 0.23\%] & 1.93\% [1.83\%, 1.95\%] \\
\bottomrule
\end{tabular}
\end{table}}

{\setlength{\intextsep}{4pt}
\begin{figure}[H]
\centering
\includegraphics[width=.88\linewidth]{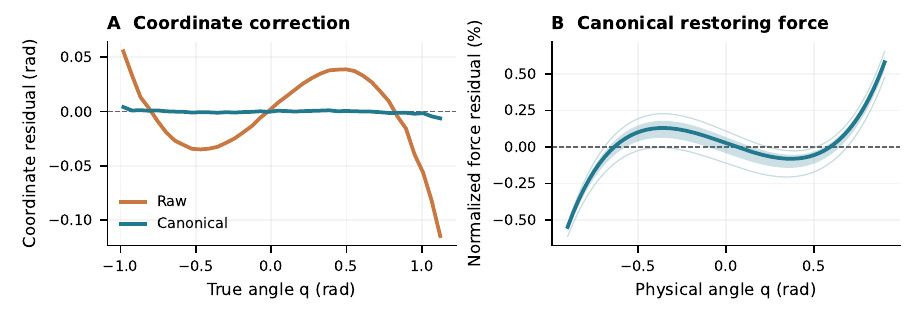}
\caption{\textbf{Coordinate and force recovery across three fits.}
A: held-out angle residuals after training-fitted affine alignment.
B: canonical force residual $(\widetilde F-F)/\mathrm{RMS}(F)$ in percent,
with $F=-4\sin q$. Thick curves: pointwise medians; shading: Q25--Q75;
thin curves: individual fits. Coordinates are binned for display only;
$e_{\rm map}$ uses original test frames.}
\label{fig:exp-normalization-residuals}
\end{figure}}

\subsubsection{Additional analysis: trajectory coverage and repeated observations}
\label{app:norm-velocity-controls}
Table~\ref{tab:norm-velocity-controls} compares one $60^\circ$ video with
one to three passages, and first-passage videos at $50^\circ,60^\circ,70^\circ$
with one or two renderings per angle.
Each condition uses 6,000 updates and five paired initializations.
Held-out releases at $45^\circ,57.5^\circ,75^\circ$ are evaluated on
$U=[-0.4,0.4]$ rad for both $e_{\rm map}$ and force $e_{\rm dyn}$.
Additional renderings vary noise and brightness; the trajectory and
training references for affine calibration remain fixed.

\begin{table}[H]
\centering
\caption{\textbf{Trajectory coverage and repeated observations.}
Errors are median [Q25,Q75] over five fits per row; all 25 maps are defined.
Three-angle rows use one passage per video. Velocity counts give the number
of distinct velocities per state on $U=[-0.4,0.4]$ rad.
$\times2$ denotes two renderings per angle with unchanged trajectories.}
\label{tab:norm-velocity-controls}
\setlength{\tabcolsep}{3pt}
{\small
\begin{tabularx}{\linewidth}{@{}Xccccc@{}}
\toprule
Training observations & Videos & \shortstack{Distinct\\velocities} & Raw $e_{\rm map}$ (\%) & Canonical $e_{\rm map}$ (\%) & Force $e_{\rm dyn}$ (\%)\\
\midrule
$60^\circ$; 1 passage & 1 & 1 & 3.137 [3.092,3.140] & 0.898 [0.886,0.899] & 19.21 [19.19,20.02] \\
$60^\circ$; 2 passages & 1 & 2 & 8.441 [8.434,8.725] & 0.823 [0.560,1.192] & 3.49 [3.38,4.61] \\
$60^\circ$; 3 passages & 1 & 3 & 8.129 [8.114,8.196] & 0.691 [0.657,0.719] & 2.65 [1.91,3.48] \\
\midrule
$50^\circ,60^\circ,70^\circ$ & 3 & 3 & 2.853 [2.839,2.943] & 0.748 [0.672,0.762] & 9.55 [8.59,10.05] \\
$(50^\circ,60^\circ,70^\circ)\times2$ & 6 & 3 & 3.007 [2.956,3.038] & 0.539 [0.527,0.546] & 3.83 [3.22,5.43] \\
\bottomrule
\end{tabularx}}
\end{table}

Extending the single video from one to two and three passages reduces
force error from $19.21\%$ to $3.49\%$ and $2.65\%$, respectively.
Three first-passage videos at $50^\circ,60^\circ,70^\circ$ still give
$9.55\%$ force error. Doubling the renderings from three to six videos
reduces this error to $3.83\%$ with unchanged velocity coverage.
Thus, three distinct velocities alone do not guarantee accurate
finite-data recovery; repeated observations also affect force accuracy.
The passage comparison additionally changes duration and state coverage.

\flushbottom

\raggedbottom
\section{Real-video methods and additional results}
\label{app:exp-real}
We study three real-video experiments: pendulum recordings from the IRIS
benchmark~\citep{khanbayov2026iris}, and side-view and overhead free-fall
recordings collected by us. Figure~\ref{fig:exp-real-acquisition-comparison}
compares the falling-object settings in IRIS with our recordings.
Although IRIS also provides similar side-view and overhead free-fall videos,
these collections do not strictly match the shared-observation setting of
our theory.
The IRIS side-view settings use balls of different sizes and appearances,
and their camera framing also differs. In the overhead recordings, the
viewpoint and framing change across release-height settings. Consequently,
pooling these settings changes the physical-state-to-image relation as
well as the trajectories: the same physical position or height need not
produce the same visual observation. Applying our shared-map analysis to
such a collection would require additional treatment of the object and
camera differences. For the two fall experiments, we therefore record
multiple releases while keeping the ball and camera configuration fixed
within each view and varying the release height.

Real recordings still depart from ideal observations through finite frame
rates, image noise, motion blur, and imperfect projection and calibration.
The three cases examine whether the parameter information characterized
by our theory, combined with physical anchors, supports accurate parameter
estimates under these conditions. Appendix~\ref{app:exp-real-pendulum}
studies pendulum invariants and length recovery, including a matched LTI
comparison; Appendix~\ref{app:exp-real-side} studies the affine acceleration
orbit and diameter calibration in side-view fall; and
Appendix~\ref{app:exp-real-overhead} studies canonical normalization and
height calibration in overhead fall.
Each case presents observations and data, theoretical predictions,
model and training, evaluation, and results and physical calibration,
in that order. Additional analyses follow the main results where applicable.

\begin{figure}[H]
\centering
\includegraphics[width=\linewidth]{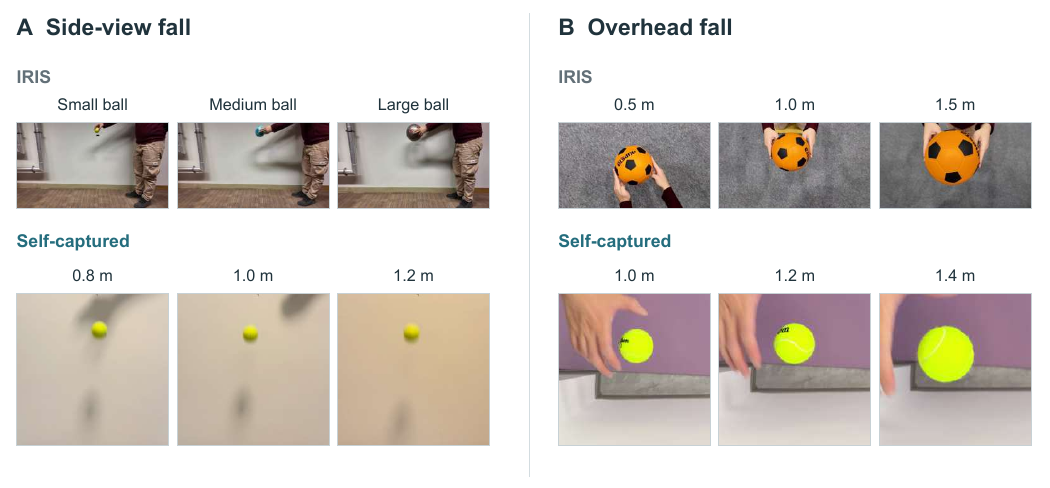}
\caption{\textbf{Falling-object recording settings in IRIS and our collection.}
IRIS examples are from \citet{khanbayov2026iris}.
Left: side view; right: overhead view. Each half contains three IRIS
settings above three self-captured release heights, with one recorded
frame per setting. IRIS side-view examples vary the ball and framing;
the overhead examples show different viewpoints and framing. Our examples
use the same ball and camera configuration within each view.
Images preserve aspect ratio; self-captured images use a common fixed crop
within each view. These are acquisition examples, not matched physical
states or performance comparisons.}
\label{fig:exp-real-acquisition-comparison}
\end{figure}

\subsection{Pendulum: invariant coefficients and length calibration}
\label{app:exp-real-pendulum}

\subsubsection{Observations and data}
\label{app:exp-real-pendulum-data}
We use 30 IRIS recordings: ten repeated trials at each release angle
$20^\circ$, $45^\circ$, and $90^\circ$, with reference pendulum length
$L=0.50$ m~\citep{khanbayov2026iris}.
Matching recording indices 01--05 across angles define five training
groups, each containing one video at each angle. A model trains on one
three-video group and is evaluated on the other 27 recordings.
Native indices $60,62,\ldots,538$ provide 240 frames at 30 Hz per video;
encoder inputs are grayscale images resized to $96\times54$.
Figure~\ref{fig:exp-pendulum-observations} illustrates the three
observation settings.

\begin{figure}[H]
\centering
\includegraphics[width=\linewidth]{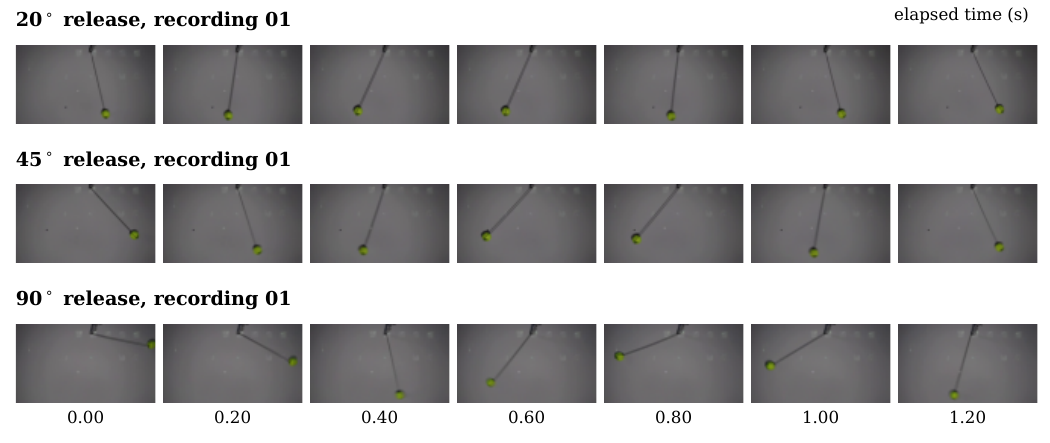}
\caption{\textbf{Pendulum observations on IRIS}~\citep{khanbayov2026iris}.
Each row shows recording 01 from one release-angle setting. The displayed
RGB frames illustrate the observed motion; elapsed time is measured from
the first selected frame.}
\label{fig:exp-pendulum-observations}
\end{figure}

\subsubsection{Theoretical predictions}
\label{app:exp-real-pendulum-theory}
For physical angle $q$, the damped pendulum obeys
\begin{equation}
 q''+\delta q'+\kappa\sin q=0,\qquad \kappa=g/L.
 \label{eq:exp-real-pendulum-law}
\end{equation}
Under the shared-state and coverage assumptions of
Theorem~\ref{thm:semilinear-affine-collapse}, the compatibility identities
in \eqref{eq:compat-a-main} restrict the angular coordinate to
$z=sq+2\pi n$, with $s\in\{-1,+1\}$ and $n\in\mathbb Z$.
Both $\delta$ and $\kappa$ are invariant under this coordinate orbit
(Table~\ref{tab:regimes}). Thus the theoretical prediction is that
compatible fits preserve these coefficients despite sign/period ambiguity,
and that supplying gravity determines length through $L=g/\kappa$.
For small angles, $\sin q\simeq q$ gives an LTI approximation; the
approximation becomes inaccurate at the larger angles in this collection.
We therefore also train a matched LTI model to examine the effect of
ODE-family misspecification on parameter estimates from real videos.

\subsubsection{Model and training}
\label{app:exp-real-pendulum-training}
A shared per-frame CNN with group normalization and an unbounded scalar
head produces $z$. We jointly fit the encoder and coefficients of
$z''+\delta z'+\kappa\sin z=0$; the matched LTI model replaces
$\kappa\sin z$ by $\alpha z$. Each training group has five paired
nonlinear/LTI fits, giving 25 fits per model. Each pair starts from the
same network initialization and uses the same training triplet.
Damping is nonnegative and restoring coefficients are positive.
Both models use 8,000 Adam steps, encoder/ODE learning rates
$10^{-3}/10^{-2}$, and gradient clipping at 1. The objective uses local
standard-deviation normalization and a non-collapse penalty with
standard-deviation floor $s_{\mathrm{floor}}=0.10$ and weight
$\lambda_{\mathrm{var}}=100$.

\subsubsection{Evaluation}
\label{app:exp-real-pendulum-evaluation}
We evaluate coordinate NRMSE $e_{\rm map}$ from
\eqref{eq:common-coordinate-error} under two alignment rules, each fitted
on the three training recordings and frozen on the other 27 videos.
The reference $q$ is the optically measured angle. For the \emph{affine}
comparison, fit $z\simeq\lambda q+\tau$ and use
$\widetilde q=(z-\tau)/\lambda$. This gives both models the same scale
and offset freedom and measures coordinate transfer after those
corrections. For the \emph{sign/period} comparison, choose
$s\in\{-1,+1\}$ and $n\in\mathbb Z$ using
$z\simeq sq+2\pi n$, and use $\widetilde q=(z-2\pi n)/s$.
This stricter rule tests angular recovery within the nonlinear family's
theoretical coordinate orbit, without freely correcting scale or offset.
For LTI, it is a fixed-angle diagnostic because LTI permits arbitrary
coordinate scale. Reporting both rules separates agreement with the
physical angular scale from transfer under a common affine readout.

\subsubsection{Results and physical calibration}
\label{app:exp-real-pendulum-results}
The nonlinear model has lower median test coordinate NRMSE under both
rules (Table~\ref{tab:exp-pendulum-alignments}): $5.04\%$ versus
$7.13\%$ after affine alignment, and $13.22\%$ versus $39.04\%$
after sign/period alignment. The affine comparison favors the nonlinear
model in 20 of the 25 paired fits.
Figure~\ref{fig:exp-pendulum-results}A--B shows the individual fits and
median/IQR summaries within each training group. The remaining
sign/period error shows that the nonlinear encoder only approximately
recovers the theoretical angular gauge; its smaller affine error alone
would not establish recovery of the physical angular scale.

\begin{table}[H]
\centering\small
\caption{\textbf{Coordinate accuracy under two common alignment rules.}
Coordinate NRMSE $e_{\rm map}$ from \eqref{eq:common-coordinate-error}
(\%), median [Q25, Q75] over 25 paired fits per model. Both alignments
are fitted on training recordings and evaluated on the same 27 held-out
recordings per pair.}
\label{tab:exp-pendulum-alignments}
\begin{tabular}{@{}lcc@{}}
\toprule
Model & Affine: $z\simeq\lambda q+\tau$ & Sign/period: $z\simeq sq+2\pi n$\\
\midrule
Nonlinear pendulum & 5.04 [4.38, 5.83] & 13.22 [12.20, 13.97]\\
LTI & 7.13 [5.46, 7.45] & 39.04 [28.65, 53.30]\\
\bottomrule
\end{tabular}
\end{table}

The learned nonlinear restoring coefficient is
$\hat\kappa=19.050\pm0.044\,\mathrm{s^{-2}}$, while LTI gives
$\hat\alpha=17.376\pm0.015\,\mathrm{s^{-2}}$; the corresponding
damping estimates are $0.07805\pm0.00114$ and
$0.08270\pm0.00240\,\mathrm{s^{-1}}$.
Here $\kappa=g/L$ in \eqref{eq:exp-real-pendulum-law}; it does not
determine $g$ and $L$ separately. 

\textbf{Anchor.}
Using the gravity anchor
$g=9.81\,\mathrm{m/s^2}$ after fitting gives
\begin{equation}
 \hat L=\frac{9.81}{\hat\kappa},\qquad
 \hat L_{\mathrm{eff}}=\frac{9.81}{\hat\alpha}.
 \label{eq:exp-real-pendulum-anchor}
\end{equation}
The nonlinear estimate is $0.5150\pm0.0012$ m, compared with
$0.5646\pm0.0005$ m for the LTI effective length. Against the IRIS
reference $L=0.50$ m, median parameter relative errors $e_\theta$
from \eqref{eq:common-parameter-error} are $2.92\%$ [2.82, 3.17]
and $12.92\%$ [12.83, 12.97], respectively.
Figure~\ref{fig:exp-pendulum-results}C displays the calibrated lengths
and their mean/SD by training group, with the reference length dashed.
The more accurate nonlinear length estimates show why retaining the
nonlinear restoring law matters when calibrating large-angle motion:
the same gravity anchor cannot compensate for the LTI approximation.

\begin{figure}[H]
\centering
\includegraphics[width=\linewidth]{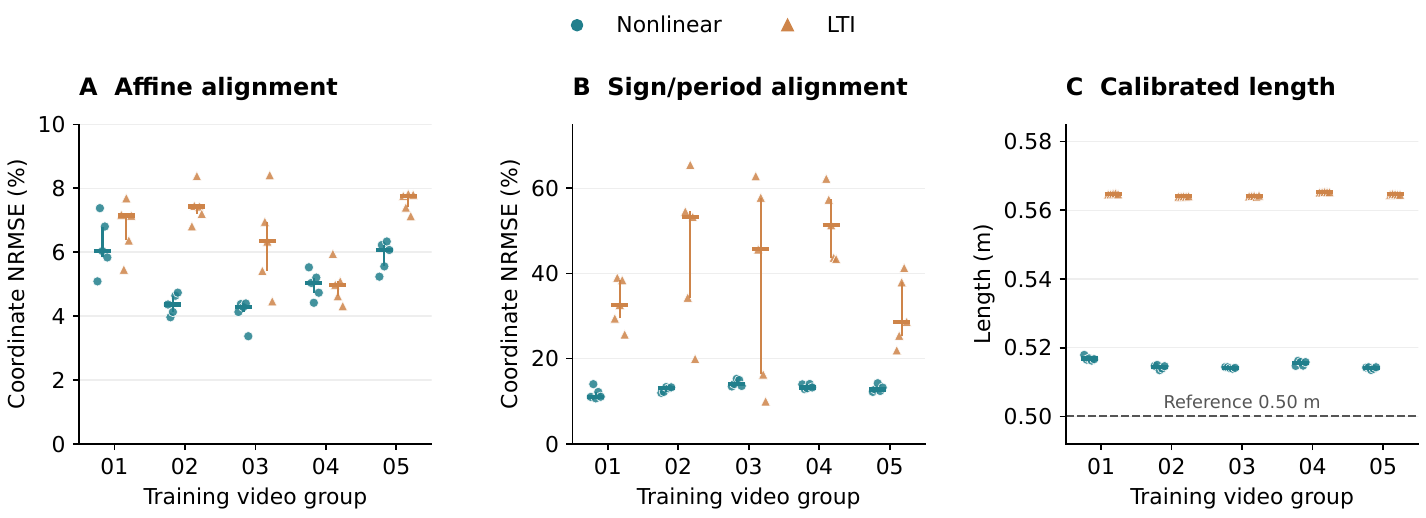}
\caption{\textbf{Coordinate recovery and calibrated length by training group.}
Groups 01--05 each contain five nonlinear fits (teal circles) and five
matched LTI fits (orange triangles). A--B: test coordinate NRMSE
$e_{\rm map}$ from \eqref{eq:common-coordinate-error} after affine and
sign/period alignment, respectively; bars show median and IQR.
C: length after gravity calibration in
\eqref{eq:exp-real-pendulum-anchor}; bars show mean $\pm$ sample SD,
and the dashed line marks the IRIS reference $0.50$ m.
All vertical axes are linear.}
\label{fig:exp-pendulum-results}
\end{figure}

\FloatBarrier
\subsection{Side-view free fall: affine ambiguity and acceleration calibration}
\label{app:exp-real-side}

\subsubsection{Observations and data}
\label{app:exp-real-side-data}
We record a falling ball with the same camera position, camera settings,
and ball across all release heights. A fixed native crop
$(x,y,w,h)=(390,760,300,800)$ is resized to $72\times192$; each clip
contains 5--11 frames at approximately 60 Hz, with recorded timestamps
retained. The camera captures a fixed vertical interval, so the observed
windows need not include the release instant. Figure~\ref{fig:exp-side-observations}
shows three example windows and the motion trails present in these recordings.

\begin{figure}[H]
\centering
\includegraphics[width=.95\linewidth]{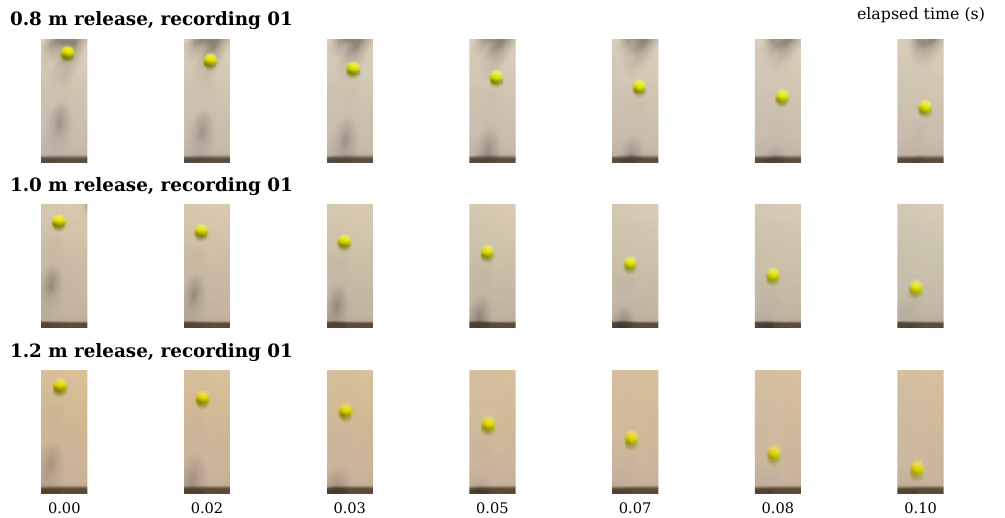}
\caption{\textbf{Side-view observations.} Three release heights (0.8, 1.0,
and 1.2 m), one row each. The fixed camera captures an intermediate
vertical interval; the release instant is unavailable in these windows.
The first seven frames of recording 01 are shown from each selected
window; shared elapsed times are relative to the start of each window.
Fixed-crop RGB images are $72\times192$.}
\label{fig:exp-side-observations}
\end{figure}

We form five training IC collections, IC1--IC5, using different
combinations of repeated recordings. Each contains two videos at each
of 0.8, 1.0, and 1.2 m, giving six training videos per collection.
The same five encoder initializations are used for each collection,
giving 25 fits. All collections are evaluated on the same 14 test
videos, none of which appears in any training collection.

\subsubsection{Theoretical predictions}
\label{app:exp-real-side-theory}
For downward physical displacement $q$, ideal free fall obeys $q''=g$.
Under the shared-state and three-slope coverage assumptions,
Theorem~\ref{thm:semilinear-affine-collapse} restricts a compatible learned
coordinate to $z=\lambda q+\tau$, with $z''=\lambda g$. Different
release heights provide different velocities through the same observed
height interval. With a locally constant image scale, the theory therefore
predicts the coordinate and parameter relations
\begin{equation}
 z\simeq\lambda_{\rm px}y+\tau,\qquad
 A\simeq\lambda_{\rm px}a_{\rm px},
 \label{eq:exp-real-side-orbit}
\end{equation}
where $y$ is vertical pixel position and $a_{\rm px}$ its acceleration.
Thus fitted $A$ can vary with $\lambda_{\rm px}$ while the ratio
$A/\lambda_{\rm px}$ remains approximately invariant. We test this
affine orbit across fits and use an external length anchor to convert
the scale-corrected acceleration to physical gravity.

\subsubsection{Model and training}
\label{app:exp-real-side-training}
Four convolutions with 16/24/32/48 channels, GroupNorm and SiLU, followed
by fixed spatial pooling and a 128-unit head, produce scalar $z$ from
saturation-based images. At every training step, least squares profiles
clip-specific intercepts and velocities and a common acceleration $A$
in $z''=A$. Gradients through this profiled fit update the encoder;
$A$ has no separate optimizer learning rate.
The objective uses dynamics weight 10 and standard-deviation floor
$s_{\mathrm{floor}}=0.10$. Its non-collapse term is
$\lambda_{\mathrm{var}}M^{-1}\sum_m(s_{\mathrm{floor}}-s_m)_+^2$
with $\lambda_{\mathrm{var}}=100$, where $s_m$ is the within-clip
latent standard deviation and $(u)_+=\max(u,0)$.
Horizontal-shift and brightness consistency have weight 0.1;
weak upper-scale and centering penalties have weights 0.01 and 0.001.
AdamW updates the encoder for 8,000 steps at learning rate $10^{-3}$,
with gradient clipping at 1 and zero weight decay. Random convolutional
and linear weights and biases are scaled by 0.1 at initialization;
normalization parameters are unchanged.

\subsubsection{Evaluation}
\label{app:exp-real-side-evaluation}
The post-fit reference pool contains eight videos, four each at 0.8
and 1.2 m; it can overlap training and is disjoint from the test set.
On this pool, fit one global
$z\simeq\lambda_{\rm px}y+\tau$ and common pixel acceleration $a_{\rm px}$
with clip-specific initial conditions, then freeze these quantities.
Coordinate NRMSE $e_{\rm map}$ in \eqref{eq:common-coordinate-error}
uses $x=y$ and $\widetilde x=(z-\tau)/\lambda_{\rm px}$ on the 14
held-out videos, including all seven releases at the unseen training
height 1.4 m. We compute each fit's median error over these videos, then
summarize across fits. Parameter error $e_\theta$ in
\eqref{eq:common-parameter-error} compares $A/\lambda_{\rm px}$ with
$a_{\rm px}$ to assess the affine acceleration orbit; after diameter
calibration, it compares $\hat g$ with $9.81\,\mathrm{m/s^2}$.
The reference position $y$, pixel acceleration $a_{\rm px}$, and
calibration width $d_{\rm px}$ are expressed in native-image pixel units.

\subsubsection{Results and physical calibration}
\label{app:exp-real-side-results}
The fitted $(\lambda_{\rm px},A)$ pairs follow the predicted affine
orbit in \eqref{eq:exp-real-side-orbit}
(Figure~\ref{fig:exp-side-calibrated-results}A), with median
scale-corrected acceleration error $8.69\%$ [3.68, 11.63].
The same frozen affine readout gives median test coordinate NRMSE
$3.75\%$ [3.22, 5.14]; panel B shows its variation across the five
IC collections. These results support approximate agreement with the
orbit while retaining measurable acceleration error.

\textbf{Anchor.} To set the physical length scale, the stationary-ball image in
Figure~\ref{fig:exp-side-calibration} provides
horizontal ball width $d_{\rm px}=88$ native pixels. With diameter reference
$D=0.067$ m, the pixel-to-metre conversion gives
\begin{equation}
 \hat g=\frac{A}{\lambda_{\rm px}}\frac{D}{d_{\rm px}}.
 \label{eq:exp-real-side-anchor}
\end{equation}

\begin{figure}[H]
\centering
\includegraphics[width=.50\linewidth]{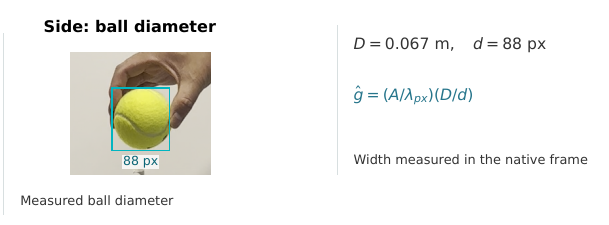}
\caption{\textbf{Side-view physical calibration.} The measured ball width and diameter reference define the spatial conversion used to report gravity.}
\label{fig:exp-side-calibration}
\end{figure}

The calibrated gravity estimate is
$9.370\pm0.771\,\mathrm{m/s^2}$, with parameter relative error
$5.10\%$ [3.21, 7.42] against $9.81\,\mathrm{m/s^2}$ (median [Q25, Q75]).
Figure~\ref{fig:exp-side-calibrated-results}C reports the five gravity
estimates and their mean/SD within each IC collection. The diameter anchor
therefore turns the scale-corrected pixel acceleration into an estimate
in physical units despite variation in the learned coordinate scale.

\begin{figure}[H]
\centering
\includegraphics[width=\linewidth]{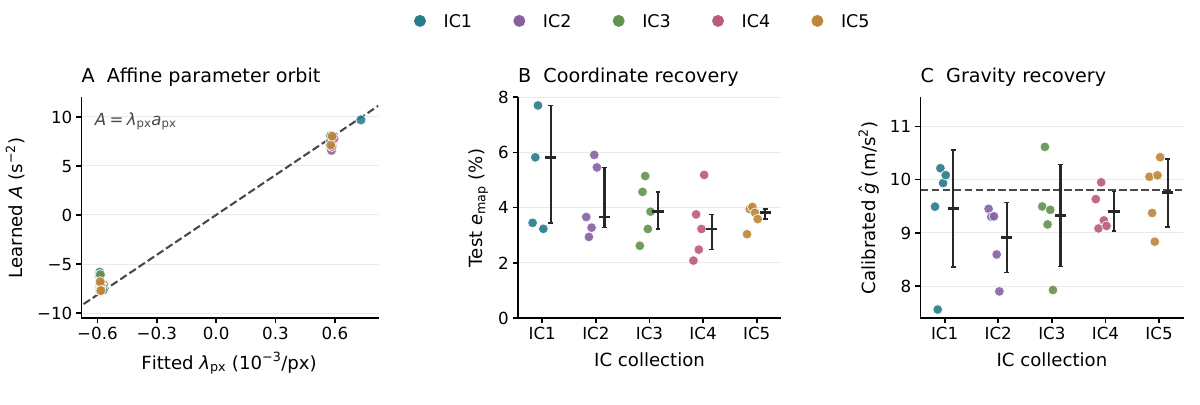}
\caption{\textbf{Affine parameter orbit, coordinate recovery, and calibrated gravity.}
Colors identify the five training IC collections, IC1--IC5;
each point is one of five fits per collection. A: learned acceleration $A$ versus coordinate scale
$\lambda_{\rm px}$; the dashed line is the theoretical relation
$A=\lambda_{\rm px}a_{\rm px}$ from \eqref{eq:exp-real-side-orbit},
using the independently fitted pixel-acceleration reference.
B: held-out coordinate NRMSE $e_{\rm map}$ from
\eqref{eq:common-coordinate-error}; black bars show median and IQR.
C: gravity after diameter calibration in \eqref{eq:exp-real-side-anchor};
black bars show mean $\pm$ sample SD, and the dashed line marks
$9.81\,\mathrm{m/s^2}$.}
\label{fig:exp-side-calibrated-results}
\end{figure}

The small coordinate error does not guarantee an equally accurate
acceleration estimate. Converting pixel acceleration to gravity using
the ball diameter assumes that one pixel represents the same physical
distance in the calibration image and throughout the falling trajectory.
If the ball's distance from the camera changes, or its pixel diameter
is measured inaccurately, the resulting gravity estimate can be biased.
The IC collections and initialization scale were selected during development
using this dataset. The comparisons therefore show how estimates
vary with the training videos and encoder initialization within this dataset.

\FloatBarrier
\subsection{Overhead free fall: connection normalization and physical calibration}
\label{app:exp-real-overhead}
\begingroup
\setlength{\intextsep}{6pt}

\subsubsection{Observations and data}
\label{app:exp-real-overhead-data}
We fix the camera approximately 2 m above the ground, looking downward,
and keep its position and camera settings unchanged throughout recording.
The videos have resolution $1080\times1920$ in portrait orientation and
are recorded at approximately 60 fps. A fixed $960\times960$ crop
with native coordinates $(x,y,w,h)=(60,480,960,960)$ is
downsampled to $64\times64$ for the encoder.
We release the same ball from heights 0.8, 1.0, 1.2, and 1.4 m above
the ground and use five videos at each height, giving 20 dynamic videos in total.
Figure~\ref{fig:exp-overhead-observations} shows example observations
from three release heights.

\begin{figure}[H]
\centering
\includegraphics[width=.9\linewidth]{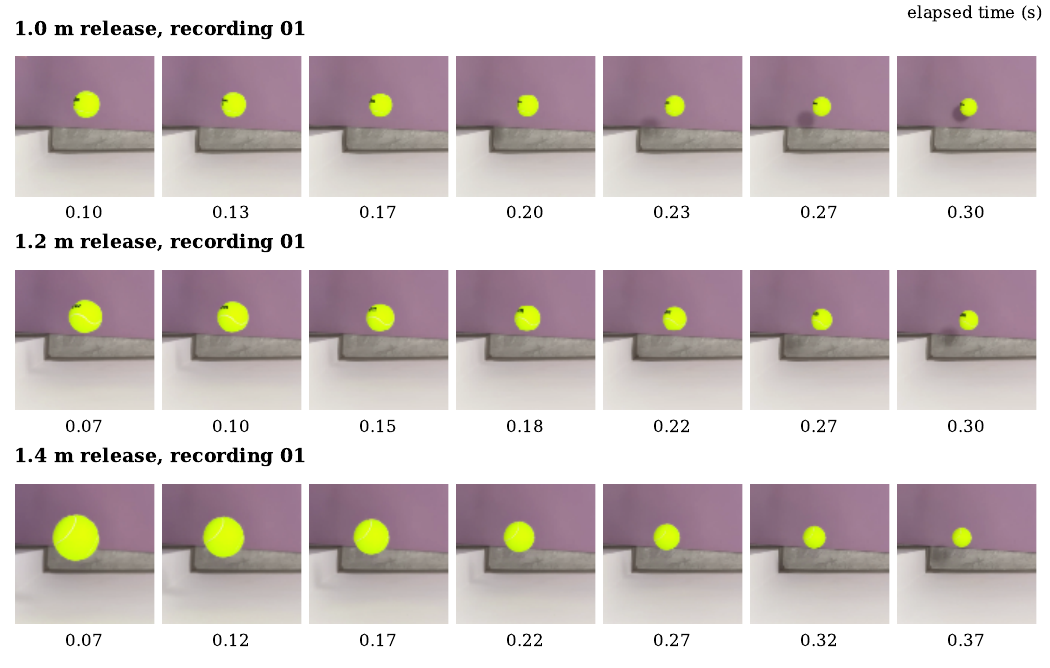}
\caption{\textbf{Overhead observations.} Fixed-crop RGB frames at release
heights 1.0, 1.2, and 1.4 m, one row each. Times are relative to the
start of the corresponding original clip.}
\label{fig:exp-overhead-observations}
\end{figure}

We form five training initial-condition (IC) collections, IC1--IC5,
from different combinations of repeated releases at the four nominal
release heights (Table~\ref{tab:exp-overhead-groups}). These nested
collections contain 5, 7, 9, 11, and 14 videos, respectively, drawn from
the same 14-video training pool; all share the other six videos for
testing. Each collection includes at least one recording from every
release height. Under ideal free fall,
different release heights give distinct velocities at the same lower
height; on their common height interval, this provides the three distinct
velocities required by Theorem~\ref{thm:canonical-report-orbit}.
The retained training segments contain only 13--23 frames each, and
the test windows contain 7--18 frames. Because each clip provides only
a short trajectory, we pool repeated recordings at the same release
heights for joint training and evaluate on a separate pool of held-out
recordings from those heights. This reduces reliance on any single
short clip while keeping the training and test videos distinct.
Recorded timestamps are retained throughout.

\begin{table}[H]
\centering\small
\caption{\textbf{Overhead training and test recordings.}
IC1--IC5 are nested training collections of repeated releases; each
covers all four heights and shares the same six test videos.
Recording IDs are local to each release height. Frames are counted within the
training segments or test windows actually used.}
\label{tab:exp-overhead-groups}
\setlength{\tabcolsep}{6pt}
\begin{tabular}{@{}lccccrr@{}}
\toprule
Collection & 0.8 m & 1.0 m & 1.2 m & 1.4 m & Videos & Frames\\
\midrule
IC1 & 01 & 01 & 01,02 & 01 & 5 & 95\\
IC2 & 01 & 01,02 & 01,02,03 & 01 & 7 & 132\\
IC3 & 01,02 & 01,02 & 01,02,03 & 01,02 & 9 & 166\\
IC4 & 01,02,03 & 01,02,03 & 01,02,03 & 01,02 & 11 & 199\\
IC5 & 01,02,03,04 & 01,02,03,04 & 01,02,03,04 & 01,02 & 14 & 247\\
Test & 05 & 05 & 05 & 03,04,05 & 6 & 78\\
\bottomrule
\end{tabular}
\end{table}

\subsubsection{Theoretical predictions}
\label{app:exp-real-overhead-theory}
The overhead view encodes height through the ball's apparent size.
Let $h$ be upward physical height, $H$ the fixed camera height, and
$R$ the image radius. For a fixed ball under ideal pinhole projection,
$R=C/(H-h)$, where $C>0$ combines focal length and ball radius.
Although physical free fall has constant acceleration $h''=-g$,
differentiating this nonlinear observation map gives
\begin{equation}
 R''=2\frac{(R')^2}{R}-\frac{g}{C}R^2.
 \label{eq:exp-real-overhead-projection}
\end{equation}
The squared-velocity term therefore arises from perspective projection.
This motivates the learned family $z''=\rho(z')^2/z-kz^2$, with $z>0$;
the raw latent coordinate need not be affine in physical height.

To remove this term, we use the law-dependent normalization of
Section~\ref{subsec:canonical-reports}. Substituting $c_2(z)=\rho/z$ into
\eqref{eq:normalizer-ode}, with $\psi_\rho(1)=0$ and
$\psi_\rho'(1)=1$, gives
\begin{equation}
 r=\psi_\rho(z)=
 \begin{cases}(z^{1-\rho}-1)/(1-\rho),&\rho\ne1,\\
 \log z,&\rho=1,\end{cases}
 \qquad r''=-kz^{2-\rho}.
 \label{eq:exp-real-overhead-normalizer}
\end{equation}
Indeed, $\psi_\rho'(z)=z^{-\rho}$, so the chain-rule terms involving
$(z')^2$ cancel in $r''=\psi_\rho'(z)z''+\psi_\rho''(z)(z')^2$.
The map is increasing and regular on $z>0$. At $\rho=2$, it becomes
$\psi_2(z)=1-1/z$. Applied to the ideal image radius, this gives
$\psi_2(R)=h/C+1-H/C$, an affine function of height, with constant
acceleration $-g/C$.

More generally, under the shared, non-collapsed state-map and exact
compatibility assumptions, three distinct velocities at each covered
height allow Theorem~\ref{thm:canonical-report-orbit} to apply.
Since height already has no squared-velocity channel,
\eqref{eq:canonical-encoder-affine} requires $r=\lambda h+\tau$ with
one affine map shared by all clips. Hence $r''=-\lambda g$ is constant.
For $k\ne0$ and nonconstant $z$, the canonical law in
\eqref{eq:exp-real-overhead-normalizer} can satisfy this only when
$\rho=2$. The theory thus predicts recovery of an affine height coordinate
after normalization and an invariant coefficient $\rho=2$ in the ideal
compatible setting.

The coefficient $k$ still depends on coordinate scale. For example,
rescaling the raw output to $\tilde z=cz$, $c>0$, preserves $\rho$ but
changes $k$ to $k/c$. In canonical coordinates, an affine change
$\tilde r=\alpha r+\beta$ rescales acceleration by $\alpha$.
Normalization therefore removes the nonlinear distortion while leaving
the conversion to physical units undetermined. Height anchors determine
$h=ar+b_h$; at $\rho=2$, the calibrated acceleration is $h''=-ak$,
so $g=ak$. For fitted $\rho$ near 2, the corresponding acceleration is
$-akz^{2-\rho}$, whose variation also tests departure from constant
free-fall acceleration. The calibration and error checks below assess
these predictions on the real recordings.

\subsubsection{Model and training}
\label{app:exp-real-overhead-training}
We resize RGB frames to $64\times64$ and apply color-based preprocessing
to enhance the ball's contrast against the background and reduce
background interference.
The encoder uses two $3\times3$ stride-one convolutions, each with two
output channels, GroupNorm and SiLU, followed by global average pooling
and a linear scalar head. Its output $o$ is mapped to the positive latent
$z=0.05+\operatorname{softplus}(o)$.
Network initializations are generated by Gaussian perturbations of one
untrained reference network: each convolutional and linear weight or
bias tensor receives a random perturbation,
while GroupNorm parameters retain their initial values.

The short overhead clips contain only 13--23 training frames, so
numerical second derivatives can amplify frame-level noise.
We therefore replace the finite-difference dynamics term in
\eqref{eq:training_objective} with an integral residual of the same
canonical law in \eqref{eq:exp-real-overhead-normalizer}.
This uses the full clip at its recorded timestamps without numerically
differentiating the encoder outputs. Writing $r_j=\psi_\rho(z_j)$ and
integrating $r_j''=-kz_j^{2-\rho}$ twice from the start of clip $j$ gives
\begin{equation}
 r_j(t)=r_{j,0}+v_{j,0}t
 -k\int_0^t(t-s)z_j(s)^{2-\rho}\,ds.
 \label{eq:exp-real-overhead-integral}
\end{equation}
The clip-specific $(r_{j,0},v_{j,0})$ accommodate different initial
conditions, while the encoder, $\rho$, and $k$ remain shared across clips.

For smooth trajectories with $z_j>0$, satisfying
\eqref{eq:exp-real-overhead-integral} for all $t$ is equivalent to
satisfying the canonical ODE, with $r_{j,0}=r_j(0)$ and
$v_{j,0}=r_j'(0)$. Since $\psi_\rho$ is invertible on this domain,
this is also equivalent to the original learned ODE.
The integral formulation thus preserves the continuous-time dynamics
and parameter-compatibility target of
Theorem~\ref{thm:canonical-report-orbit}.
On sampled, noisy frames, the integral and finite-difference losses
are different numerical objectives; a small integral loss alone does
not establish the theorem's continuous-time compatibility or coverage
assumptions.

Let $\varepsilon_{j,n}$ be the left-hand side minus the right-hand side
of \eqref{eq:exp-real-overhead-integral} at timestamp $t_{j,n}$, and
let $\bar r_j$ be the mean encoded canonical state within clip $j$.
The dynamics loss is
\begin{equation}
 \mathcal L_{\mathrm{int}}(\phi,\rho)
 =\min_{k,\{r_{j,0},v_{j,0}\}}
 \frac{\sum_{j,n}\varepsilon_{j,n}^{2}}
 {\sum_{j,n}\bigl(r_j(t_{j,n})-\bar r_j\bigr)^2}.
 \label{eq:exp-real-overhead-integral-loss}
\end{equation}
The denominator expresses residual error relative to within-clip signal
variation and uses a small numerical floor to prevent division by zero.
We evaluate the integral using piecewise-linear interpolation of
$z^{2-\rho}$ at the recorded timestamps. At every update, linear least
squares computes the shared $k$ and each clip's $(r_{j,0},v_{j,0})$;
these quantities have no separate optimizer learning rate.
Adam jointly updates the encoder and $\rho$ for 4,000 steps, with initial
learning rates $\ell_{{\rm enc},0}=10^{-3}$ and
$\ell_{\rho,0}=6\times10^{-4}$. Both rates follow
\begin{equation}
 \ell_j(t)=\ell_{j,0}\left[0.05+0.475\left(1+\cos\frac{\pi t}{8000}\right)\right],
 \qquad 0\le t\le4000,
 \label{eq:exp-overhead-lr}
\end{equation}
so the final rates are $5.25\times10^{-4}$ and $3.15\times10^{-4}$.
Gradient norms are clipped at 20. Let
$s_0=\max(\operatorname{std}(z_0),0.08)$, fixed from the
initial training outputs. In addition to $\mathcal L_{\mathrm{int}}$,
the objective includes the mean over clips of
$[\max(0,0.6s_0-\operatorname{std}(z_j))/s_0]^2$ and a positivity penalty
below 0.08, each with weight 10. The first discourages near-constant
encodings, serving the same anti-collapse role as
\eqref{eq:AB_variance_floor}; the second adds a margin from the singular
boundary $z=0$ of the law and normalizer. These penalties support the
non-collapse and domain conditions used by the theory without supplying
physical coordinate or parameter labels.

\subsubsection{Evaluation}
\label{app:exp-real-overhead-evaluation}
Let $u=1/R$ denote inverse tracked ball radius. For raw $z$ and canonical
$r=\psi_\rho(z)$, fit separate global affine readouts
$h(z)\simeq\lambda_hu+\tau_h$ on pooled training frames, with $h(z)=z$ or
$\psi_\rho(z)$. Use the frozen inverse $\widetilde u=[h(z)-\tau_h]/\lambda_h$
on the six test clips.
The coordinate reference is $x=u$. Residual plots average test-frame
residuals in twelve state bins to show how the distortion varies with
ball size; error metrics use the original frames.

\subsubsection{Results and physical calibration}
\label{app:exp-real-overhead-results}
The main results in Table~\ref{tab:exp-real-parameters} use IC5, the
largest collection with 14 training videos, and $\rho_0=2.0$.
Table~\ref{tab:exp-overhead-normalization} summarizes the effect of
canonical normalization across its five fits. The coordinate comparison
uses the inverse-radius reference and training-fitted affine readouts;
it requires no physical calibration. Normalization reduces median test
NRMSE from $15.85\%$ to $5.57\%$, while the learned $\rho$ is close to
the ideal value 2.

\begin{table}[H]
\centering\small
\caption{\textbf{Law-derived normalization in overhead fall.}
IC5 with $\rho_0=2.0$, five fits. Coordinate errors $e_{\rm map}$ use
\eqref{eq:common-coordinate-error} and are median [Q25, Q75];
$\rho$ is mean $\pm$ sample SD.}
\label{tab:exp-overhead-normalization}
\setlength{\tabcolsep}{15pt}
\begin{tabular}{@{}ccc@{}}
\toprule
Raw $e_{\rm map}$ (\%) & Canonical $e_{\rm map}$ (\%) & Learned $\rho$\\
\midrule
15.85 [15.74, 16.01] & 5.57 [5.57, 5.60] & $2.072\pm0.024$\\
\bottomrule
\end{tabular}
\end{table}

The predicted canonical acceleration $-kz^{2-\rho}$ has a median
coefficient of variation of $0.73\%$ over test frames: its standard
deviation divided by the absolute value of its mean, in percent.
Multiplication by a constant calibration scale leaves this quantity
unchanged. Together with the coordinate errors, this supports an
approximately affine height coordinate and nearly constant canonical
acceleration. Figure~\ref{fig:exp-overhead-residuals} shows how
normalization reduces the remaining coordinate distortion across ball sizes.

\begin{figure}[H]
\centering
\includegraphics[width=.74\linewidth]{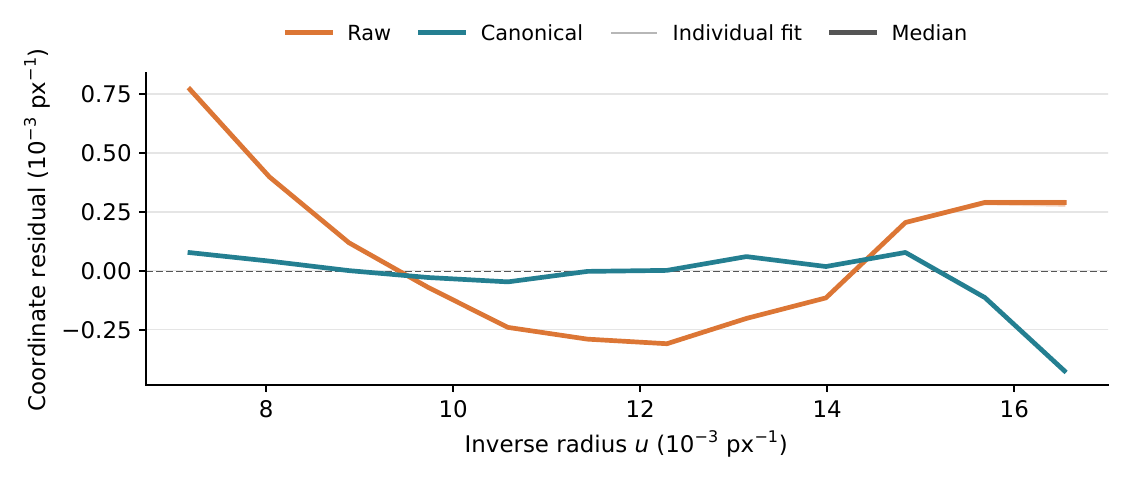}
\caption{\textbf{Coordinate residuals for IC5.}
All five fits use $\rho_0=2.0$. Test-coordinate residuals after the
training-fitted raw or canonical affine readout are averaged in twelve
inverse-radius bins. Thin curves show individual fits; thick curves show
pointwise medians. Zero indicates agreement with the inverse-radius
reference. Physical anchors are not used in these readouts.}
\label{fig:exp-overhead-residuals}
\end{figure}

\textbf{Anchor.} We use the nominal release heights recorded during data acquisition as external physical calibration information to map the learned
canonical coordinate to height in metres.
These height labels are not used to train the encoder or the
dynamical model. After training, we associate each of the seven
original first frames in Figure~\ref{fig:exp-overhead-calibration}
with its nominal release height $h_i\in\{1.0,1.2,1.4\}\,\mathrm{m}$.
The frozen encoder and learned normalizer produce the corresponding
canonical coordinate $r_i=\psi_{\hat\rho}(z_i)$.
Each pair $(r_i,h_i)$ therefore provides a physical anchor linking
the learned coordinate to a height in metric units.
We use these anchors to determine the remaining affine scale and
offset, $h=ar+b_h$:
\begin{equation}
(a,b_h) = \operatorname*{argmin}_{\alpha,\beta}
\sum_i
\frac{[h_i-\alpha r_i-\beta]^2}{n_{h_i}},
\qquad
\hat g = a\hat k
\left\langle z^{2-\hat\rho}\right\rangle_{\rm train}.
\label{eq:exp-real-overhead-anchor}
\end{equation}
Here $n_{h_i}$ is the number of calibration frames assigned to
height $h_i$, so each distinct height receives equal total weight.
The fitted scale $a$ converts canonical acceleration to physical
acceleration; the offset $b_h$ does not affect acceleration.
We report $\hat g$ as the mean over all training frames.

This protocol treats each nominal release height as an approximation
to the ball-center height in the corresponding first frame, rather
than as an independently measured frame-level height.
Any discrepancy therefore contributes to calibration error.
The 1.4 m recording 04 contributes its first frame to calibration
and a later window to coordinate testing. Thus, the two uses involve
different frames, but this recording is not held out from physical
calibration.

\begin{figure}[H]
\centering
\includegraphics[width=.9\linewidth]{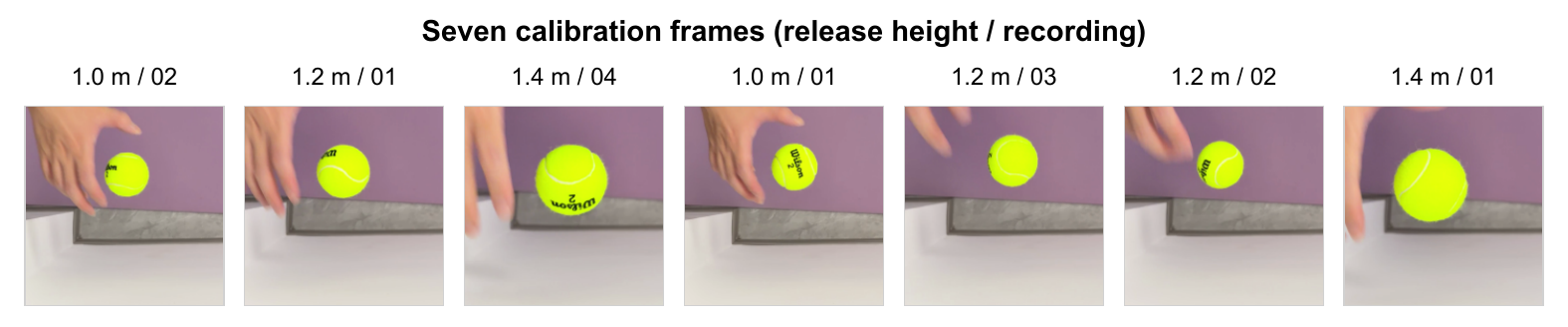}
\caption{\textbf{Release-height anchors.} Original RGB frame zero from
seven recordings, shown before masked-saturation preprocessing. Labels
give nominal release height and recording ID in
Table~\ref{tab:exp-overhead-groups}. Calibration applies the same
preprocessing and frozen encoder used for the dynamic observations.}
\label{fig:exp-overhead-calibration}
\end{figure}

This calibration gives $\hat g=10.146\pm0.029\,\mathrm{m/s^2}$,
with median parameter relative error $e_\theta=3.35\%$ [3.26, 3.39] against
$9.81\,\mathrm{m/s^2}$, as reported in the main
Table~\ref{tab:exp-real-parameters}.
The calibrated acceleration on held-out test frames has dynamics
relative RMSE $e_{\rm dyn}=3.43\%$ [3.33, 3.65].
Gravity is reported as mean $\pm$ sample SD across the five fits;
errors are median [Q25, Q75].
For \eqref{eq:common-dynamics-error}, we use
$\widetilde F_t=-akz_t^{2-\rho}$, $F_t=-9.81$, and equal test-frame
weights. Thus $\hat g$ summarizes training-frame acceleration, whereas
$e_{\rm dyn}$ checks the calibrated law on held-out trajectories.

\subsubsection{Additional analysis: IC collections and coefficient initialization}
\label{app:exp-real-overhead-additional}
We examine sensitivity to the training IC collection and the initial
coefficient $\rho_0$ using IC1--IC5 from
Table~\ref{tab:exp-overhead-groups}. Each collection is fitted with five
initial networks at each $\rho_0\in\{1.0,2.0,3.0\}$, giving
$5\times5\times3=75$ fits. The initial networks are paired across
collections and coefficient initializations; all settings use the same
test videos and release-height anchors.
Figure~\ref{fig:exp-overhead-calibrated-results} compares the IC
collections at $\rho_0=2.0$, and
Table~\ref{tab:exp-overhead-init-results} gives all initialization results.

\begin{figure}[H]
\centering
\includegraphics[width=\linewidth]{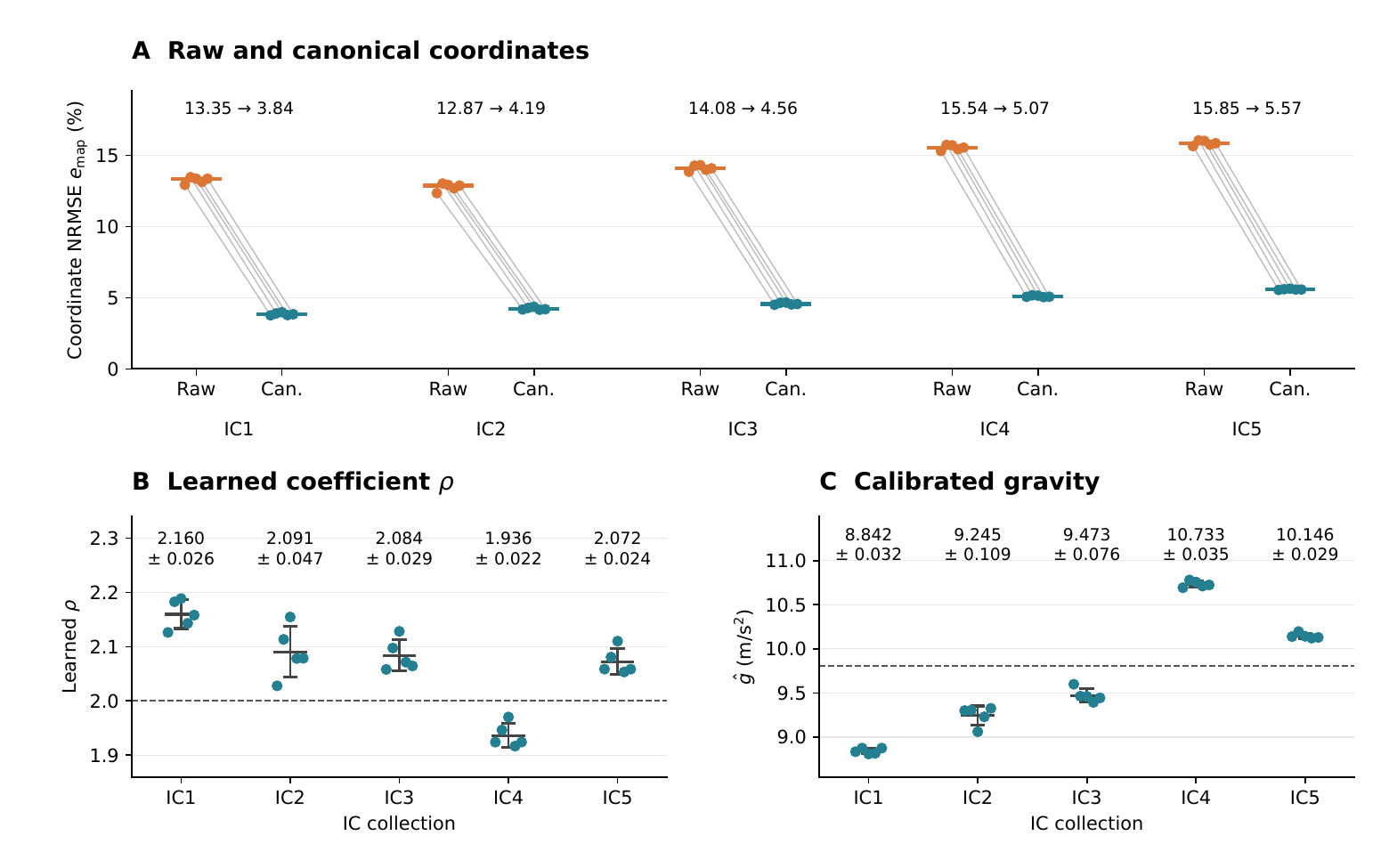}
\caption{\textbf{Sensitivity to the training IC collection.}
IC1--IC5 use the trajectory collections in
Table~\ref{tab:exp-overhead-groups}, with $\rho_0=2.0$ and five paired
local network initializations.
A: paired raw/canonical test NRMSE on the same six test videos;
horizontal bars and labels give medians; vertical bars show the IQR. B: learned $\rho$, with reference 2.
C: gravity from release-height calibration, with reference $9.81\,\mathrm{m/s^2}$.
Bars and labels in B/C show mean $\pm$ sample SD.}
\label{fig:exp-overhead-calibrated-results}
\end{figure}

Canonical normalization reduces coordinate error for every collection
and coefficient initialization. Across collections, however, coordinate
accuracy and calibrated gravity accuracy do not vary together: IC5 has
larger coordinate error than the smaller collections but a gravity
estimate closer to the reference than IC1 and IC2. The estimated
coefficient also varies with the training collection despite the shared
physical setup. Because the nested IC collections differ in both
trajectory composition and clip count, these differences reflect the
two factors jointly.

The initialization comparison examines the estimates reached after
4,000 updates. For IC5, $\rho_0=1.0$ and 2.0 give
$\rho=2.070\pm0.019$ and $2.072\pm0.024$, respectively, whereas
$\rho_0=3.0$ gives $2.259\pm0.013$. The first two starts reach similar
estimates near 2, but the third retains a larger offset; the learned
coefficient therefore remains sensitive to initialization.
For the main paper, we use IC5, which includes all 14 training videos,
with $\rho_0=2.0$. The other collections and initializations provide the
sensitivity analysis reported here.

\begin{table}[H]
\centering\small
\caption{\textbf{IC collections and coefficient initialization in overhead fall.}
Five fits per row, 75 fits in total. Coefficient and gravity estimates are
mean $\pm$ sample SD; coordinate errors are median [Q25, Q75] NRMSE (\%).
IC1--IC5 are defined in Table~\ref{tab:exp-overhead-groups}.
All gravity estimates use release-height calibration. Settings share
six test clips and seven calibration frames.
The main experiment uses IC5 (14 videos) and $\rho_0=2.0$.}
\label{tab:exp-overhead-init-results}
\setlength{\tabcolsep}{5pt}
\resizebox{\linewidth}{!}{\begin{tabular}{@{}ccrrrr@{}}
\toprule
IC collection & $\rho_0$ & Learned $\rho$ & $\hat g$ ($\mathrm{m/s^2}$) & Raw & Canonical\\
\midrule
IC1 & 1.0 & $2.037\pm0.012$ & $9.037\pm0.040$ & 12.81 [12.78, 12.83] & 3.86 [3.70, 3.88]\\
IC1 & 2.0 & $2.160\pm0.026$ & $8.842\pm0.032$ & 13.35 [13.13, 13.35] & 3.84 [3.78, 3.90]\\
IC1 & 3.0 & $2.233\pm0.079$ & $8.674\pm0.054$ & 13.42 [13.29, 13.70] & 3.85 [3.82, 3.94]\\
\addlinespace[3pt]
IC2 & 1.0 & $2.040\pm0.025$ & $9.380\pm0.128$ & 12.58 [12.57, 12.80] & 4.17 [4.15, 4.24]\\
IC2 & 2.0 & $2.091\pm0.047$ & $9.245\pm0.109$ & 12.87 [12.68, 12.91] & 4.19 [4.17, 4.28]\\
IC2 & 3.0 & $2.196\pm0.095$ & $8.960\pm0.067$ & 12.95 [12.84, 13.29] & 4.26 [4.25, 4.39]\\
\addlinespace[3pt]
IC3 & 1.0 & $2.096\pm0.019$ & $9.423\pm0.039$ & 14.09 [14.01, 14.24] & 4.57 [4.54, 4.64]\\
IC3 & 2.0 & $2.084\pm0.029$ & $9.473\pm0.076$ & 14.08 [13.97, 14.27] & 4.56 [4.52, 4.64]\\
IC3 & 3.0 & $2.293\pm0.027$ & $9.177\pm0.075$ & 15.05 [14.92, 15.19] & 4.69 [4.67, 4.77]\\
\addlinespace[3pt]
IC4 & 1.0 & $1.935\pm0.019$ & $10.718\pm0.040$ & 15.52 [15.39, 15.65] & 5.06 [5.06, 5.13]\\
IC4 & 2.0 & $1.936\pm0.022$ & $10.733\pm0.035$ & 15.54 [15.42, 15.69] & 5.07 [5.06, 5.14]\\
IC4 & 3.0 & $2.171\pm0.011$ & $10.292\pm0.058$ & 16.59 [16.48, 16.63] & 5.17 [5.16, 5.24]\\
\addlinespace[3pt]
IC5 & 1.0 & $2.070\pm0.019$ & $10.142\pm0.035$ & 15.83 [15.72, 15.95] & 5.58 [5.57, 5.60]\\
IC5 & 2.0 & $2.072\pm0.024$ & $10.146\pm0.029$ & 15.85 [15.74, 16.01] & 5.57 [5.57, 5.60]\\
IC5 & 3.0 & $2.259\pm0.013$ & $9.834\pm0.048$ & 16.69 [16.60, 16.75] & 5.69 [5.68, 5.69]\\
\bottomrule
\end{tabular}}
\end{table}

\endgroup
\FloatBarrier

\section{Supplementary identifiability tool}
\label{app:web-tool}

The supplementary offline browser tool takes a scalar ODE and parameter
restrictions as input. For a supported theorem route, it reports parameter
invariants under the admissible coordinate maps, calibration requirements,
and the velocity feature rank required at each physical state. For
squared-velocity models, it indicates when canonical normalization must
precede parameter comparison. Figures~\ref{fig:web-cubic}--\ref{fig:web-overhead}
show three paper examples.
The supplementary material includes the offline tool and a README with
setup and usage instructions.

\begin{figure}[H]
\centering
\includegraphics[width=\linewidth]{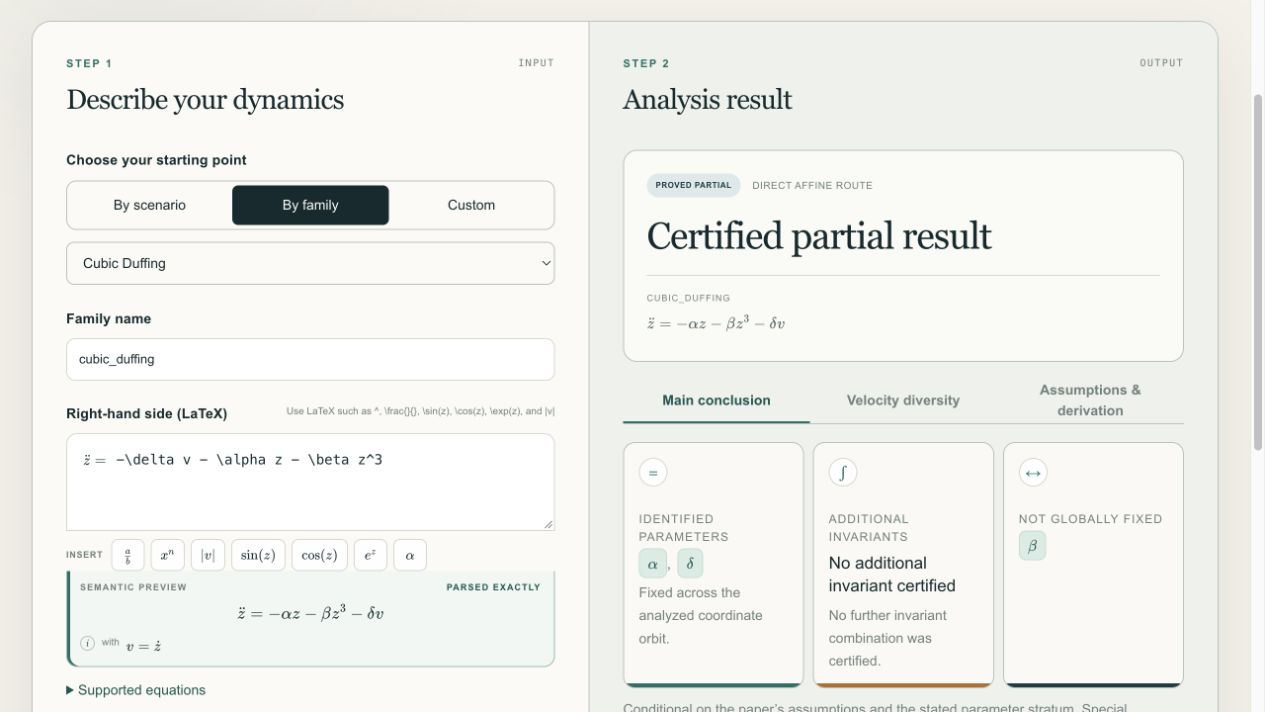}
\caption{\textbf{Cubic Duffing example in the supplementary tool.}
For the displayed nonzero-$\beta$ family, the tool reports $\alpha$ and
$\delta$ as fixed across the analyzed coordinate orbit, while $\beta$
varies with coordinate scale.}
\label{fig:web-cubic}
\end{figure}

\begin{figure}[H]
\centering
\includegraphics[width=\linewidth]{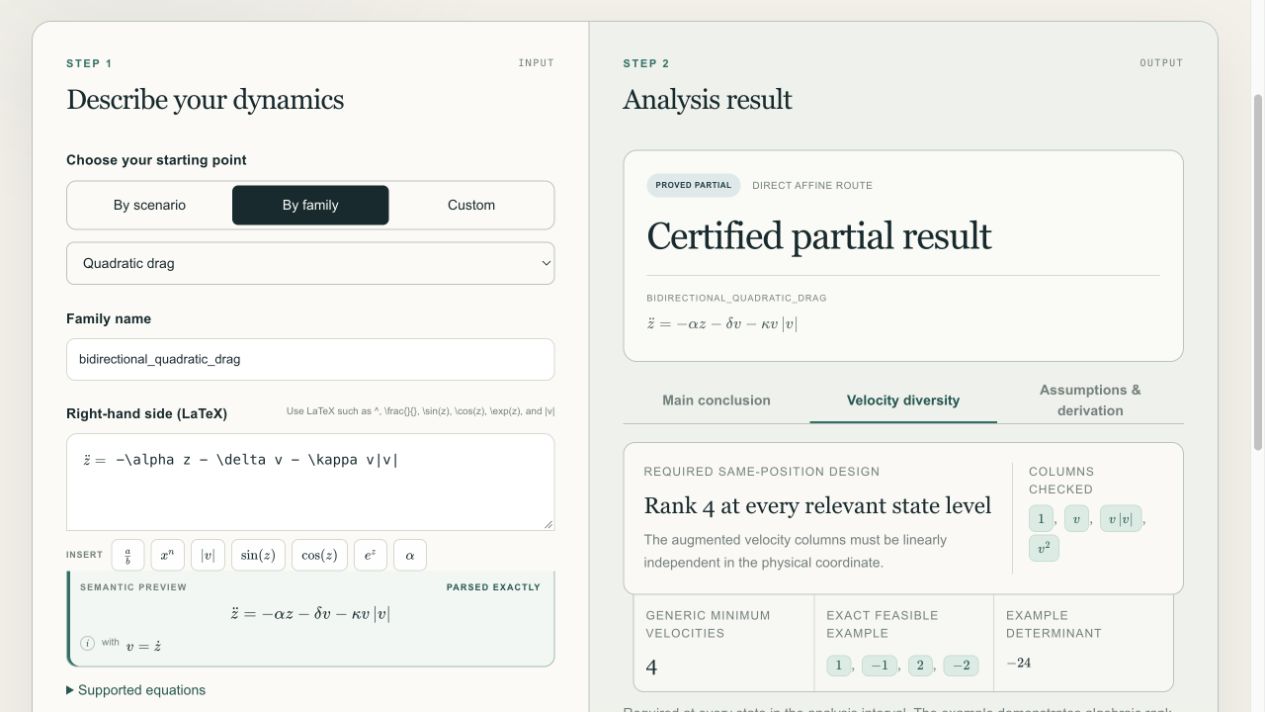}
\caption{\textbf{Velocity coverage for quadratic drag.}
The tool lists the required rank-four feature set
$[1,v,v|v|,v^2]$ at each physical state and an algebraically feasible
velocity example $\{1,-1,2,-2\}$.}
\label{fig:web-drag}
\end{figure}

\begin{figure}[H]
\centering
\includegraphics[width=\linewidth]{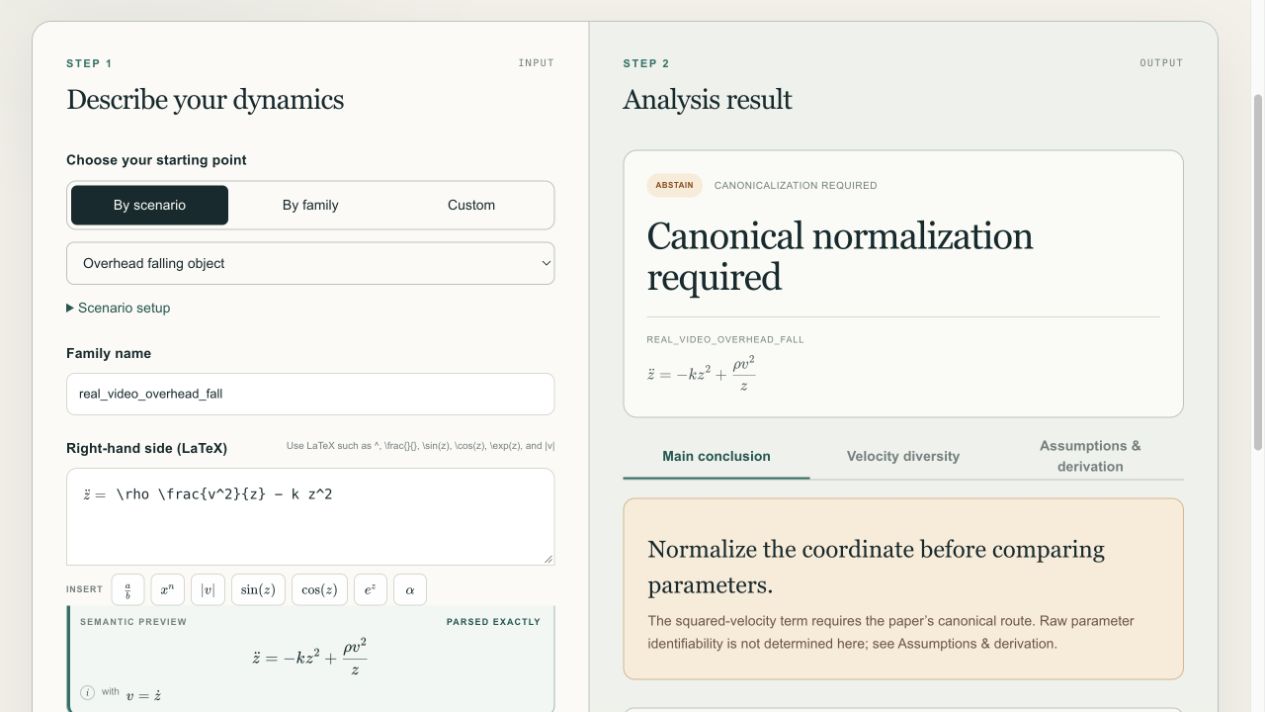}
\caption{\textbf{Canonical route for overhead fall.}
For $z''=\rho(z')^2/z-kz^2$, the tool directs the comparison to the
canonical coordinate before assessing parameter compatibility.}
\label{fig:web-overhead}
\end{figure}

\FloatBarrier

\end{document}